\documentclass[letterpaper,twocolumn,10pt]{article}

\usepackage{usenix}

\usepackage{tikz}
\usepackage{amsmath}

\usepackage{filecontents}

\usepackage[utf8]{inputenc} % allow utf-8 input
\usepackage[T1]{fontenc}    % use 8-bit T1 fonts
\usepackage{hyperref}       % hyperlinks
\usepackage{url}            % simple URL typesetting
\usepackage{booktabs}       % professional-quality tables
\usepackage{amsfonts}       % blackboard math symbols
\usepackage{nicefrac}       % compact symbols for 1/2, etc.
\usepackage{microtype}      % microtypography
\usepackage{xcolor}         % colors
\usepackage{subcaption}
\usepackage{graphicx}
\usepackage{multirow}
\usepackage{amsmath,amssymb,amsfonts}
\usepackage{amsthm}
\usepackage{mathrsfs}
\usepackage{xcolor}
\usepackage{textcomp}
\usepackage{manyfoot}
\usepackage{booktabs}
\usepackage{algorithm}
\usepackage{algorithmicx}
\usepackage{algpseudocode}
\usepackage{listings}
\usepackage[capitalize,noabbrev,nameinlink]{cleveref} % needs to be loaded after hyperref
\usepackage{graphicx}  % 用于插入图片
\usepackage{subcaption} % 用于子图管理
\crefname{lem}{Lemma}{Lemmas}
\crefname{defin}{Definition}{Definitions}
\crefname{thm}{Theorem}{Theorems}
\newtheorem{theorem}{Theorem} % continuous numbers
\newtheorem{lemma}{Lemma}% 

\newcommand{\T}{^\top}
\newcommand{\point}{z}
\newcommand{\pert}{\tilde{z}}
\newcommand{\points}{Z}
\newcommand{\perts}{\tilde{Z}}
\newcommand{\loss}{\ell}
\newcommand{\gradlossPoint}{\nabla_{\model}\loss(\point, \optmodel)}
\newcommand{\gradlossPert}{\nabla_{\model}\loss(\pert, \optmodel)}
\newcommand{\bigloss}{L}
\newcommand{\optmodel}{\theta^*}
\newcommand{\model}{\theta}
\newcommand{\Model}{\Theta}

\renewcommand{\paragraph}[1]{\vspace{2pt}\noindent\textbf{#1.}}

\begin{document}
%-------------------------------------------------------------------------------

%don't want date printed
\date{}

\title{\Large \bf Forgetting Any Data at Any Time: A Theoretically Certified Unlearning Framework for Vertical Federated Learning}

% for single author (just remove % characters)
\author{
{\rm Linian Wang}\\
Key Lab of High Confidence Software Technologies (Peking University),\\ Ministry of Education \& School of Computer Science, Peking University
\and
{\rm Leye Wang}\\
Key Lab of High Confidence Software Technologies (Peking University),\\ Ministry of Education \& School of Computer Science, Peking University
% copy the following lines to add more authors
% \and
% {\rm Name}\\
%Name Institution
} % end author
% \author{Anonymous Authors}

\maketitle

%-------------------------------------------------------------------------------
\begin{abstract}
%-------------------------------------------------------------------------------

Privacy concerns in machine learning are heightened by regulations such as the GDPR, which enforces the “right to be forgotten” (RTBF), driving the emergence of \textit{machine unlearning} as a critical research field. Vertical Federated Learning (VFL) enables collaborative model training by aggregating a sample’s features across distributed parties while preserving data privacy at each source. This paradigm has seen widespread adoption in healthcare, finance, and other privacy-sensitive domains. However, existing VFL systems lack robust mechanisms to comply with RTBF requirements, as unlearning methodologies for VFL remain underexplored. In this work, we introduce the first VFL framework with \textit{theoretically guaranteed unlearning capabilities}, enabling the removal of \textit{any data at any time}. Unlike prior approaches—which impose restrictive assumptions on model architectures or data types for removal—our solution is \textit{model- and data-agnostic}, offering universal compatibility. Moreover, our framework supports \textit{asynchronous unlearning}, eliminating the need for all parties to be simultaneously online during the forgetting process. These advancements address critical gaps in current VFL systems, ensuring compliance with RTBF while maintaining operational flexibility. We make all our implementations publicly available at https://github.com/wangln19/vertical-federated-unlearning.

\end{abstract}

%-------------------------------------------------------------------------------
\section{Introduction}
%-------------------------------------------------------------------------------

When personal data is used to train machine learning models, privacy protection can be considered from two key perspectives: first, ensuring that data does not leave the local domain during the training process to minimize the risk of data leakage~\cite{mcmahan2023communicationefficientlearningdeepnetworks}; second, providing users with the ability to revoke their consent and delete their personal data after the model has been trained~\cite{nguyen2024surveymachineunlearning}. To address the issue of data isolation and privacy risks during the training process, FL has emerged as a promising solution. In addition, regulations like the GDPR have proposed the ``right to be forgotten'' (RTBF), which necessitates the removal of personal data from trained models. This need has given rise to the concept of \textit{machine unlearning}~\cite{bourtoule2020machineunlearning}. 

Federated Unlearing (FU) combines these two orthogonal challenges, offering comprehensive privacy protection. FL can be categorized into two main types: Horizontal Federated Learning (HFL), where participants share the same features but have different samples, and Vertical Federated Learning (VFL), where participants share the same samples but have different features. VFL has found widespread application in cross-silo scenarios, such as those in banking and healthcare. However, Vertical Federated Unlearning (VFU) is still under-investigated.
In real-world scenarios, unlearning requests are frequently generated, with diverse objectives, such as client, feature, or sensitive information removal~\cite{nguyen2024surveymachineunlearning,warnecke2023machineunlearningfeatureslabels}. Current VFU methods only address specific data removal types (e.g., client~\cite{VFULR,VFUFR} or feature~\cite{VFUGBDT} unlearning) and are often tailored to particular VFL models (e.g., logistic regression~\cite{VFULR} or gradient boosting trees~\cite{VFUGBDT}). This underscores the critical need for a unified VFU framework capable of handling multiple data removal scenarios and models.

Another challenge in VFU lies in its computational and communication overheads. In VFL systems, each client maintains a sub-model trained on its local feature data. However, unlearning data from one party triggers cascading impacts across all parties’ sub-models.
Existing VFU methods address this by enforcing \textit{synchronous unlearning}~\cite{gu2025ferrarifederatedfeatureunlearning}: upon an unlearning request, all clients must coordinate in real-time to update their sub-models. Such reliance on synchronization is operationally impractical—network instability, client downtime, or resource constraints frequently disrupt client participation, thus breaking the unlearning process. 
%in existing VFU methods, upon receiving an unlearning request, all clients must be immediately assembled to synchronize and perform unlearning, which we call `synchronous unlearning'~\cite{gu2025ferrarifederatedfeatureunlearning}. In practical scenarios, unlearning requests can emerge continuously and unpredictably, requiring all VFL clients to remain online and ready to engage in unlearning at all times. However, clients may occasionally drop out due to various reasons like network issues, and thus synchronous unlearning is fragile in practice.
%This imposes substantial communication and computational costs. %Moreover, passive parties (i.e., the clients without training data labels), do not directly benefit from the VFL system. They can only ask for reward distribution from the active party (i.e., the client with task training labels) during the inference process. 
%The heavy burden of unlearning may discourage passive parties from participating in the unlearning process, as they lack sufficient incentives to go online and contribute.

In this paper, we propose a novel VFL framework with robust VFU capabilities of \textbf{forgetting any data at any time}. This framework is capable of handling diverse unlearning objectives in a unified manner and supports asynchronous unlearning method to alleviate the burden on clients. Built on the widely adopted aggregate VFL (AggVFL) systems \cite{yu2024surveyprivacythreatsdefenseaggvfl,liu2024vertical,fu2022label}, our framework unifies different unlearning targets from the perspective of confidence scores: the local models of clients in VFL (regardless of simple linear regression or complex deep neural networks) are treated as feature extractors, each generating class-specific confidence scores for the same sample based on different local features. The global model of the active party (the client who owns training labels \cite{liu2024vertical}) aggregates these outputs to maintain a matrix representing the training set, which then produces final predictions.

First, we convert VFU tasks for different objectives (client, feature, etc.) into the closed-form updates for the confidence matrix. The confidence matrix provides fine-grained information for unlearning without introducing extra privacy leakage risk compared to traditional VFL training processes~\cite{hardy2017privatefederatedlearningvertically,Yang2019ParallelDL}. Second, we propose to leverage the confidence matrix for backpropagation, thus enabling both model training and unlearning in a unified manner within our framework (Sec.~\ref{sec:method_diverse_unlearn}). For learning models with strongly convex \cite{boyd2004convex} loss functions (e.g., \textit{mean squared error} and \textit{logistic loss with L2 regularization}), we prove that our method offers certified unlearning with theoretical guarantees (Sec. ~\ref{sec:certified}).

To address the issue of excessive burdens placed on clients during unlearning, we also introduce an asynchronous unlearning method (Sec.~\ref{sec:method_asyn}). %In VFL, we need to balance the conflicting privacy and efficiency demands of different clients. 
In VFU, the client initiating the unlearning request demands prompt removal of data influences from all clients' sub-models, while the remaining clients require minimizing their unlearning burden and cannot be expected to keep staying online. Therefore, the forgetting times for different clients may vary, something traditional synchronous unlearning methods do not account for \cite{VFULR,VFUFR}. Our asynchronous unlearning approach allows only a subset of clients to participate in each unlearning epoch, significantly improving the unlearning practicality. %This enables clients with more urgent forgetting timelines to forget immediately, while others can remain offline. 
As shown in Figure \ref{stream}, when a client initiates an unlearning request, only the requesting client and the active party need to be online to process the request. Other clients may choose to participate in immediate unlearning or defer the process until later (e.g., until they submit their own unlearning requests).
% a subset of other clients participate in the unlearning process.  %until the next time constraint is met. 
Our experimental results demonstrate that asynchronous unlearning reduces computational and communication overhead while preserving model performance close to that of synchronous unlearning.

\begin{figure}[t]  
    \centering
     \includegraphics[width=\columnwidth]{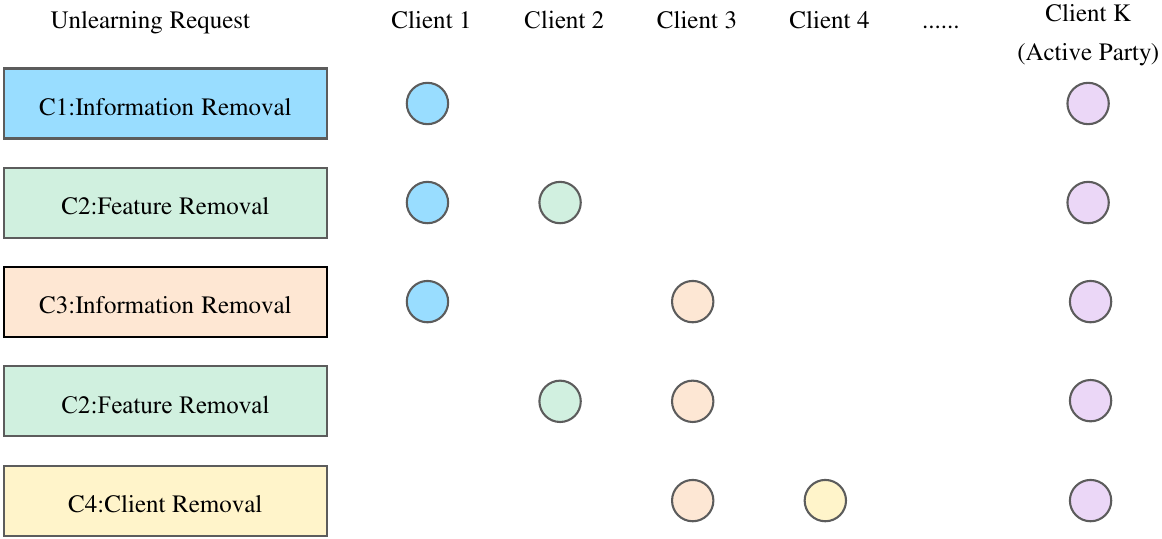} 
    \caption{What Asynchronous Unlearning Is. When Client 2 propose a feature removal request, only Client 1, 2 and K are online to unlearn.}
    \label{stream}
\end{figure}

% Furthermore, when all clients are online, our method is compatible with conventional batch processing~\cite{batchunlearning} and resolves the privacy leakage issue of storing unlearning information in batch processing. We quantify and store the impact of training set updates resulting from unlearning on the overall system using the confidence matrix, and we introduce the stability of the clients' update contribution factor. This allows us to implement asynchronous unlearning for clients who are not online. Our experiments show that asynchronous unlearning reduces computational and communication costs while ensuring that the performance of the unlearned model remains close to that of a model retrained from scratch.

This paper makes the following key contributions:

1. \textbf{A Unified Framework for Multiple VFU Targets with Theoretical Guarantees}:We present the first framework that supports diverse unlearning requests, including client, feature, sensitive information removal, etc., offering a flexible and scalable solution to address the RTBF requests in data regulation laws. For models with strongly convex losses, our method offers certified unlearning with theoretical guarantees.

2. \textbf{An Asynchronous VFU Method}: We propose the first asynchronous VFU system, designed to efficiently balance frequent unlearning requests from diverse clients while significantly lowering computational and communication overhead. 

3. \textbf{Extensive Experimental Validation}: Our extensive experimental results validate the effectiveness and applicability of the proposed VFU framework, demonstrating its practical viability in real-world scenarios.

%-------------------------------------------------------------------------------
\section{Related Work}
%-------------------------------------------------------------------------------

% \color{red}
% \subsubsection*{Vertical Federated Learning}

% Vertical Federated Learning (VFL) is ..... In general, there are two types of VFL, i.e., Aggregate VFL (AggVFL) and Split VFL (SplitVFL) \cite{liu2024vertical}. AggVFL is..... (representative work on AggVFL). SplitVFL is ..... (SplitVFL).

% In this work, we focus on AggVFL .....

\color{black}
Vertical Federated Learning (VFL) is a distributed machine learning paradigm that enables collaboration between multiple parties without sharing sensitive data. In general, there are two types of VFL, i.e., Aggregate VFL (AggVFL) and Split VFL (SplitVFL) \cite{liu2024vertical}. AggVFL employs a non-trainable global module (e.g., Sigmoid activation) to aggregate intermediate results through secure computation. Representative work on AggVFL includes SecureBoost \cite{cheng2021secureboostlosslessfederatedlearning} for federated gradient boosting trees and SFTL \cite{Liu2020SFTL} for  cryptographic defense. On the other hand, SplitVFL utilizes trainable global modules aligned with vertical split neural networks \cite{gupta2018distributedlearningdeepneural}. This type of VFL has been studied in various works, including privacy-preserving split learning \cite{vepakomma2018split} and communication-efficient CELU-VFL \cite{Fu2022CELU}. In this work, we focus on AggVFL due to its superior compatibility with applications requiring strict data isolation. 

\subsubsection*{Theoretically-Guaranteed Machine Unlearning}

Machine unlearning has attracted considerable attention in recent years. However, most existing research on machine unlearning lacks theoretical guarantees about its efficacy. There are currently two main approaches that provide formal guarantees of complete forgetting. The first approach, exemplified by SISA~\cite{bourtoule2020machineunlearning}, focuses on exact unlearning. These methods~\cite{ijcai2022ARCANE,Recommendation_Unlearning} partition both the model and the training set, performing retraining only on the sub-model relevant to the unlearning request, thereby ensuring complete unlearning. The second approach draws from differential privacy concepts~\cite{dp2011} and proves that the model post-unlearning is statistically indistinguishable from a retrained model, demonstrating the effect of approximate unlearning, also called \textit{certified unlearning}~\cite{Certified2020,neel2020descenttodeletegradientbasedmethodsmachine,warnecke2023machineunlearningfeatureslabels}.

While the first approach offers exact unlearning, it is inefficient in scenarios where unlearning requests are scatted in many sub-models~\cite{warnecke2023machineunlearningfeatureslabels}. %As such, it is unsuitable for our use case. 
Hence, our proposed framework aligns with the second approach, achieving certified approximate unlearning while ensuring efficiency.

\subsubsection*{Federated Unlearning}

Federated unlearning (FU) can be categorized based on the type of unlearning target~\cite{nguyen2024surveymachineunlearning}, including client removal, feature removal, and sample removal, among others. Most existing unlearning methods address specific unlearning requests and are not easily interoperable~\cite{zhao2025exploringfederatedunlearninganalysis}. However, in real-world applications, the targets for forgetting vary widely. This calls for a unified approach that can handle multiple forgetting requests simultaneously.

Existing research on FU primarily focuses on Horizontal FU (HFU), with few papers addressing Vertical FU (VFU). Deng \textit{et al}~\cite{VFULR} first introduce a VFU method for Logistic Regression (LR) models and client removal. Zhang \textit{et al}~\cite{VFUGBDT} propose VFU approach for Gradient Boosted Decision Trees, which can unlearn samples by recalculating split points, or unlearn features by changing the splitting feature. Wang \textit{et al}~\cite{VFUFR} introduce a model-agnostic VFU method that accelerates retraining, but it only addresses client removal. %Unlike models that are retrained from scratch, this does not guarantee complete unlearning. 
Gu \textit{et al}~\cite{gu2024fewshotlabelunlearningvertical} propose a VFU method to forget labels through gradient ascent techniques.

%\subsection{Asynchronous Federated Unlearning}

Recent research on FU has acknowledged the significance of asynchronous unlearning, but all on HFU tasks. Su \textit{et al}~\cite{knot} divide the HFL system into multiple small modules based on slicing, enabling asynchronous forgetting across different modules. This method assumes unlearning requests arrive during the system training phase. Gu \textit{et al}~\cite{gu2025ferrarifederatedfeatureunlearning} assume that all clients possess identical models, so that after a client performs a unlearning operation, the updated model can be broadcast to all clients, thus achieving global unlearning. %Both of these studies are based on the concept of HFU, where partial client forgetting does not require coordination with other clients.

VFU presents unique challenges in asynchronous unlearning. In VFU, each client holds only a partial model, and forgetting operations require global coordination among clients. This introduces distinct difficulties for asynchronous forgetting, as the model segments are distributed across clients.

Table \ref{table:relatedwork} provides a comparative summary of our approach and existing VFU methods, emphasizing the distinct contributions of our research.

\begin{table}[t]
\centering
\caption{Vertical Federated Unlearning Research Summary}
\resizebox{\columnwidth}{!}{%
\begin{tabular}{lcccc}
\toprule
\textbf{Paper} & \textbf{Unlearning Target} & \textbf{Asynchronous} & \textbf{Certified} & \textbf{Model-Agnostic} \\
% \midrule
% \textbf{VFU} & & & & \\
\midrule
\cite{VFULR} & Client & & & \\
\cite{VFUGBDT} & Feature\&Sample & & & \\
\cite{VFUFR} & Client & & & \checkmark\\
\cite{gu2024fewshotlabelunlearningvertical} & Label & & & \checkmark\\
 \textbf{Our Work}& Multiple& \checkmark & \checkmark & \checkmark\\
% \midrule
% \textbf{HFU} & & & & \\
% \midrule
% \cite{knot} & Sample  & \checkmark & & \checkmark\\
% \cite{gu2025ferrarifederatedfeatureunlearning} & Feature & \checkmark & & \checkmark\\
\bottomrule
\end{tabular}
}
\label{table:relatedwork}
\end{table}

%-------------------------------------------------------------------------------
\section{Method}
%-------------------------------------------------------------------------------
\subsection{Basic Idea}

Our approach is based on the following idea: in linear models $y=ax_1 + bx_2 + c$, forgetting $x_1$ corresponds to subtracting the corresponding term $ax_1$ from $y$ and then fine-tuning the model to remove $x_1$'s influence on the parameters $b$ and $c$. 

For more complex models in VFL, we can apply this idea at the output layer. In a classification problem, we maintain a confidence matrix in the active party of the VFL, where each row represents a sample and each column represents the confidence for a particular class. Upon receiving a forgetting request, we subtract the corresponding confidence value from the matrix, eliminating the forgotten target’s influence on the output. We then fine-tune the model using the updated output to remove the forgotten target's effect on the model parameters. 

For instance, in binary classification, the confidence vector associated with a sample is two-dimensional, where each element corresponds to one of the two classes. Consider two clients, \(i\) and \(j\), where Client \(i\) has a classification confidence vector of \([-1, 1]\) and Client \(j\) has a confidence vector of \([0, 3]\) for a specific sample. During aggregation, the global model combines these vectors, producing a summed confidence vector \([-1, 4]\), which is stored in a row of the confidence matrix. To remove Client \(i\)'s contribution on this sample, its confidence vector \([-1, 1]\) is subtracted from the corresponding row of the matrix. The updated confidence matrix is then used to compute the gradient with respect to the label. By updating Client \(j\)’s parameters using this revised matrix, the influence of Client \(i\) on Client \(j\)’s parameters is effectively removed.

% We suppose that there are two clients $i$ and $j$.
% Client $i$ has a classification confidence vector of \([-1, 1]\), while Client $j$ has a confidence vector of \([0, 3]\). When the global model aggregates these two confidence vectors, it results in \([-1, 4]\), which is stored in one row of the confidence matrix. To unlearn Client $i$'s contribution, the confidence vector \([-1, 1]\) is subtracted from the first row of the matrix. After this update, the modified confidence matrix is used to compute a new gradient with respect to the label, and upon updating Client $j$'s parameters, the influence of Client $i$ on Client $j$’s parameters is effectively eliminated.

\subsection{Compatibility with Diverse Requests}
\label{sec:method_diverse_unlearn}

% To ensure compatibility with different unlearning targets in VFL, we should design the system so that data does not leave the domain, while still allowing the forgotten information to be communicated to each client. The communicated information needs to be sufficiently granular to support various unlearning requests, such as feature removal or client removal. We find that the class confidence perfectly meets these requirements: unlearning tasks for different targets can be unified into a closed-form update of the confidence matrix. When the model parameters remain unchanged, updates to the confidence matrix can reflect changes in any data from a client's training set, without disclosing the original data. 

We want to design a VFL system that supports diverse unlearning requests (e.g., client removal, feature removal) while ensuring data remains within the local domain. This requires a mechanism to communicate forgotten information to each client in a granular yet privacy-preserving manner. 

% We start by using a widely-adopted aggregate VFL (AggVFL) structure, where the complex model is deployed across the clients~\cite{yu2024surveyprivacythreatsdefenseaggvfl}. Each client directly outputs a confidence vector, while the global model only aggregates these vectors. We have the client proposing the unlearning request calculate the difference in the confidence vectors of the training set before and after unlearning, and send this difference to the active party. Based on this, the active party updates the confidence matrix, effectively removing the influence of the forgotten data. Because any format of data changes in the training set can be represented by the confidence difference vector, unlearning requests for different targets can be uniformly transformed into updates to the confidence matrix, thereby standardizing the update process across various unlearning requests. Furthermore, this unlearning process transmits the same information as the training process in VFL, ensuring no additional privacy leakage. 

We leverage the confidence matrix as a unified representation for unlearning. When an unlearning request is made, the requesting client calculates the difference in confidence vectors before and after unlearning and sends this difference to the active party. The active party updates the confidence matrix to reflect the removal of the forgotten data. For Granularity, any data change (e.g., client removal, feature removal) can be represented as an update to the confidence matrix. For Privacy, this unlearning process transmits the same information as the training process in VFL, ensuring no additional privacy leakage.

Additionally, our framework makes a slight adjustment to the backpropagation steps, so the system starts the update process from the stored confidence matrix. This ensures that after the confidence matrix is updated due to unlearning, neural network training can proceed normally, avoiding reliance on raw data in the VFL system. The unlearning process is as follows, as illustrated in Figure \ref{overview}:

\begin{enumerate}
    \item The requesting client computes the confidence difference vector and sends it to the active party.
    \item The active party updates the confidence matrix and computes gradients based on the updated matrix and labels.
    \item Gradients are sent to the clients, who update their model parameters and compute new confidence differences.
    \item The active party aggregates all confidence differences and updates the confidence matrix for the next round.
\end{enumerate}

% Additionally, our framework makes a slight adjustment to the backpropagation steps, so the system starts the update process from the stored confidence matrix. This ensures that after the confidence matrix is updated due to unlearning, neural network training can proceed normally, avoiding reliance on raw data in the VFL system. The model training process is illustrated in Figure \ref{overview}:

% \begin{enumerate}
%     \item After a forgetting request is made, the requesting client computes the difference in confidence vectors (before and after unlearning) and sends this difference to the active party.
%     \item The active party uses the updated confidence matrix and labels to compute gradients, which are then sent to the clients.
%     \item The clients compute the gradient of the model parameters based on the received gradients, update the parameters, and compute the resulting confidence difference, which is sent back to the active party.
%     \item The active party aggregates all the confidence differences, sums them, and updates the confidence matrix to start the next round of updates.
% \end{enumerate}

\begin{figure*}[h]  
    
    \begin{minipage}[b]{0.43\textwidth}
        \centering
        \includegraphics[width=\textwidth]{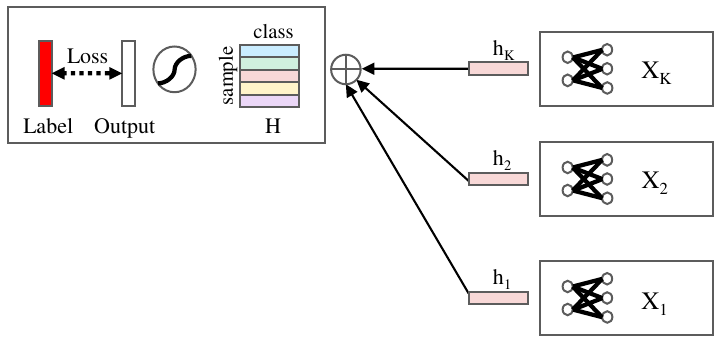} % 替换为您的图片路径
        % \caption{a1}
    \end{minipage}
    \hfill
    \begin{minipage}[b]{0.43\textwidth}
        \centering
        \includegraphics[width=\textwidth]{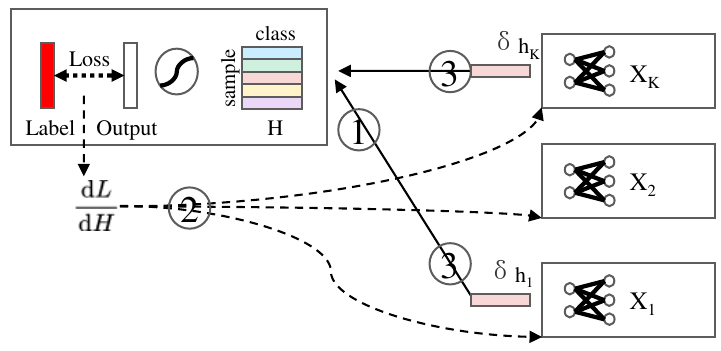} % 替换为您的图片路径
        % \caption{1}
    \end{minipage}

    \caption{Structure and the Unlearning Process of  Our Method.}
    \label{overview}
\end{figure*}

% The advantages of AggVFL structure are:
% \begin{enumerate}
%     \item \textbf{Privacy protection:} The local model outputs from clients represent high-level features (such as the class confidence in classification tasks), and the difference during unlearning is significantly distinct from the raw feature values, which better meets the privacy requirement of data staying within the domain. Additionally, the outputs from individual clients are aggregated, which can be enhanced using existing privacy-preserving techniques like obfuscation to prevent the server from inferring which client contributed to which update, further improving privacy protection.
%     \item \textbf{Mathematical rigor:} By aggregating the representations from each client in the global model using summation, the derivatives of the loss with respect to the overall confidence matrix and the derivatives of the loss with respect to each client's confidence output are mathematically equivalent and unaffected by the global model parameters. This property is useful when we later prove the certified forgetting.
% \end{enumerate}

We use the widely-adopted AggVFL structure, the advantages of the AggVFL structure include:

\begin{enumerate}
    \item \textbf{Privacy protection}: The local model outputs (e.g., class confidence) are high-level features distinct from raw data, ensuring data remains within the domain. Aggregation process can be further enhanced using existing privacy-preserving techniques like obfuscating individual contributions.
    \item \textbf{Mathematical rigor}: The summation-based aggregation of confidence outputs ensures that the derivatives of the loss with respect to the global confidence matrix and each client’s output are mathematically equivalent. This property is crucial for proving certified unlearning guarantees.
\end{enumerate}

\subsection{Achieving Asynchronous Unlearning}
\label{sec:method_asyn}

\subsubsection{Asynchronous Unlearning Setting}

% There are two types of time constraints to consider for VFL unlearning:
% \begin{enumerate}
%     \item \textbf{VFL system inference phase:} During inference, the VFL system requires data and computation from all clients. Therefore, all clients must be online before each inference, and all unlearning requests must be completed by these clients.
%     \item \textbf{Unlearning time requirement:} The unlearning request are proposed with a time requirement, i.e., all corresponding unlearning tasks for the clients must be completed before the deadline.\footnote{According to the GDPR and CCPA, while these regulations grant users the right to delete personal data, they do not specify a concrete time frame within which machine learning companies must complete the unlearning process or eliminate the impact of such data from their models.}
% \end{enumerate}

In VFU, two key time constraints must be considered:
\begin{enumerate}
    \item \textbf{VFL Inference Phase:} During inference, the VFL system relies on data and computations from all clients. Thus, all clients must be online, and any unlearning requests must be completed prior to inference.
    \item \textbf{Unlearning Deadline:} Unlearning requests are associated with a deadline, requiring all related tasks to be completed within the specified timeframe.\footnote{While GDPR and CCPA grant users the right to delete personal data, they do not mandate a specific timeframe for completing unlearning or mitigating its impact on models.}
\end{enumerate}

% In our asynchronous unlearning setup, within a specific time window during which the VFL system operates (between two consecutive time points when the time limit is reached), not all clients are required to remain online. When unlearning occurs, only the client proposing the unlearning request, the active party and a subset of other clients need to be online. In this asynchronous unlearning setting, data to be forgotten is directly stored in the client proposing the unlearning request, making this client more sensitive to the unlearning speed. The asynchronous approach allows this client to perform unlearning immediately, satisfying privacy requirements without requiring the other clients to remain online. On other clients, it allows the unlearning process to occur after a certain period, thus reducing the burden.

Our asynchronous unlearning framework operates within a defined time window between two consecutive inference phases. Unlike synchronous approaches, not all clients need to remain online. When an unlearning request is initiated, only the requesting client and the active party must be online. The requesting client stores the data to be forgotten, making it particularly sensitive to unlearning speed. This client can perform unlearning immediately, ensuring compliance with privacy regulations without requiring other clients to stay online. Other clients can defer their unlearning processes, reducing their operational burden.

% In the asynchronous unlearning setup, some clients may be offline and unable to participate in the update process. Thus, the difficulty lies in two main aspects:

The asynchronous setup introduces two main challenges: 

\begin{enumerate}
    \item \textbf{Quantifying Client Impact:} Accurately measuring how each client's model update contributes to the global model update.
    \item \textbf{Compensating for Offline Clients:} Addressing the absence of offline clients during unlearning by leveraging updates from online clients and stored information, ensuring no additional privacy leakage.
\end{enumerate}

% \begin{enumerate}
%     \item \textbf{Quantifying the impact of each client’s model on the global model update:} We must accurately determine how the update of each client's model contributes to the overall model update.
%     \item \textbf{Compensating the impact of offline clients:} When some clients are offline during the unlearning process, we need to use the updates from online clients and stored information to make up for the effect that the offline clients’ updates would have had. Additionally, during this process, the stored information must bring no extra privacy leakage.
% \end{enumerate}

The first challenge is effectively addressed by our confidence-based framework. During each update, the impact of each client's update is directly reflected in the change in its output confidence, which then influences the global model's output, and subsequently the gradient for the next update. The second challenge is how to calculate the change in confidence output resulting from model parameter updates for offline clients. We solve this by introducing the stability of the \textit{update contribution factor} in the operation of a VFL system.

% \begin{enumerate}
%     \item \textbf{Quantifying Client Impact:} Accurately measuring how each client's model update contributes to the global model update.
%     \item \textbf{Compensating for Offline Clients:} Addressing the absence of offline clients during unlearning by leveraging updates from online clients and stored information, ensuring no additional privacy leakage.
% \end{enumerate}

% Our confidence-based framework addresses the first challenge by directly linking each client's update impact to changes in its output confidence, which influences the global model's output and subsequent gradients. For the second challenge, we introduce the concept of the \textit{update contribution factor}, which quantifies the confidence output changes resulting from parameter updates for offline clients, ensuring system stability and privacy preservation.

\subsubsection{The Stability of The \textit{Update Contribution Factor}}

Each passive party computes the confidence score \( h_i \) of its local model. The active party aggregates these scores, computing \( H = \sum h_i \), and feeds \( H \) into the global model \( G \) to produce the final output \( p = G(H) \).

During backpropagation, as client parameters \( \theta_i \) are updated, their confidence scores \( h_i = f_i(\theta_i, X_i) \) also change. Let \( \delta(h_i) \) denote the change in \( h_i \). The global model's input change \( \delta(H) \) is a linear combination of these individual changes: \( \delta(H) = \sum \delta(h_i) \), which in turn affects the global output \( p \).

The linear relationship \( \delta(H) = \sum \delta(h_i) \) provides insight into each client's contribution to the global model update. The ratio \(R_i = \frac{\delta(h_i)}{\delta(H)} \), noted as termed the \textit{update contribution factor}, quantifies client \( i \)'s influence on the model's change during updates. Notably, in VFL systems using logistic regression, these factors remain stable across training epochs.

To illustrate, consider two clients, \( i \) and \( j \), with features \( i_1 \) to \( i_4 \) and \( j_1 \) to \( j_5 \), respectively. Their confidence scores are computed as:
\begin{align*}
 h_i &= f_i(\theta_i, x_{i1}, x_{i2}, x_{i3}, x_{i4}) \\
 h_j &= f_j(\theta_j, x_{j1}, x_{j2}, x_{j3}, x_{j4}, x_{j5}).
\end{align*}

Since the training set remains fixed, the representations output by each client depend solely on their parameters. We approximate changes in these representations using a first-order Taylor expansion. This approximation is justified in unlearning scenarios, where training set changes are minimal, and parameter updates before and after unlearning are not significantly different.

\begin{align*}
 \delta(h_i) &= \delta(\theta_{i1})\frac{\partial{h_i}}{\partial{\theta_{i1}}} + \delta(\theta_{i2})\frac{\partial{h_i}}{\partial{\theta_{i2}}} + \delta(\theta_{i3})\frac{\partial{h_i}}{\partial{\theta_{i3}}} + \delta(\theta_{i4})\frac{\partial{h_i}}{\partial{\theta_{i4}}} \\
 \delta(h_j) &= \delta(\theta_{j1})\frac{\partial{h_j}}{\partial{\theta_{j1}}} + \cdots + \delta(\theta_{j5})\frac{\partial{h_j}}{\partial{\theta_{j5}}} 
\end{align*}

Subsequently, the change in the parameters can be directly derived from the backpropagation update rule, denoted that $\eta$ is the learning rate. By substituting these changes into the update rule, we observe that the ratio of the changes in the representations of client \( i \) and client \( j \) depends solely on the training data in the regression models. Since the training data remains unchanged between epochs, this ratio remains stable over time.

\begin{align*}
 \delta(\theta_{i1}) &= -\eta\frac{\partial{L}}{\partial{\theta_{i1}}} = -\eta\frac{\partial{L}}{\partial{H}}\frac{\partial{H}}{\partial{h_i}}\frac{\partial{h_i}}{\partial{\theta_{i1}}}\\
 \delta(h_i) &= -[(\frac{\partial{h_i}}{\partial{\theta_{i1}}})^2 + \cdots + (\frac{\partial{h_i}}{\partial{\theta_{i4}}})^2]\eta\frac{\partial{L}}{\partial{H}}\frac{\partial{H}}{\partial{h_i}} \\
  &= -[x_{i1}^2 + \cdots + x_{i4}^2]\eta\frac{\partial{L}}{\partial{H}}\\
\delta(h_j) &= -[(\frac{\partial{h_j}}{\partial{\theta_{j1}}})^2 + \cdots + (\frac{\partial{h_j}}{\partial{\theta_{j5}}})^2]\eta\frac{\partial{L}}{\partial{H}}\frac{\partial{H}}{\partial{h_j}} \\
&= -[x_{j1}^2 + \cdots + x_{j5}^2]\eta\frac{\partial{L}}{\partial{H}}\\
\frac{\delta(h_i)}{\delta(h_j)} &= \frac{[x_{i1}^2 + \cdots + x_{i4}^2]}{[x_{j1}^2 + \cdots + x_{j5}^2]}\\
R_i &= \frac{\delta(h_i)}{\delta(H)} = \frac{[x_{i1}^2 + \cdots + x_{i4}^2]}{[x_{i1}^2 + \cdots + x_{i4}^2] + [x_{j1}^2 + \cdots + x_{j5}^2]}
\end{align*}

From the above derivation, it is clear that for a Multilayer Perceptron (MLP), the contribution coefficients of different clients to the update of \(z\) may vary across epochs. However, under the unlearning setting, if the MLP has a limited number of layers, this variation is negligible. We provide a proof of this observation in Appendix A.

Since the \textit{update contribution factor} remains stable, we can leverage the confidence updates from online clients to estimate those of the offline client K. For the global model update, we have: 

\begin{align*}
 \delta(h_{K}) &= \frac{R_K}{\sum_{i\in\{online\}}R_i}\sum_{i\in\{online\}}\delta(h_{i})\\
  \delta(H) &= \frac{1}{\sum_{i\in\{online\}}R_i}\sum_{i\in\{online\}}\delta(h_{i})
\end{align*}

The \textit{update contribution factor} must be computed and stored by the active party prior to receiving unlearning requests. This can be readily accomplished at the conclusion of the VFL system's training phase. 

\subsection{Achieving Certified Unlearning}
\label{sec:certified}
%-------------------------------------------------------------------------------

\subsubsection{Definition}

%The widely accepted definition of certified unlearning is analogous to differential privacy. 
Consider a learning algorithm $\mathcal{A}$ that, when trained on a dataset $D$, generates a model $\model \in \Model$. An unlearning method $\mathcal{U}$ that transforms a model $\model$ into a corrected version $\model_\mathcal{U} = \mathcal{U}(\model, D, D^\prime)$. $D^\prime$ contains the perturbations $\perts$ needed for unlearning, while the corresponding original data is $\points$. The concept of \emph{$\epsilon$-certified unlearning} means that it is hard to distinguish models after unlearned $\mathcal{U}$ from the set of possible retrained models $\mathcal{A}(D^\prime)$~\cite{Certified2020,warnecke2023machineunlearningfeatureslabels}.

\newtheorem{defin}{Definition}
\begin{defin}
	\label{def-ecr}
	Given some $\epsilon >0$ and a learning algorithm
	$\mathcal{A}$, an unlearning method~$\mathcal{U}$ is
	$\epsilon$-certified if
	\begin{equation*}
		e^{-\epsilon} \leq
		\frac{P\Big(\mathcal{U}\big(\mathcal{A}(D),D,D^\prime\big)
			\in\mathcal{T}\Big)}
		{P\big(\mathcal{A}(D^\prime)\in\mathcal{T}\big)}\leq e^{\epsilon}
		\label{eqn:ecr}
	\end{equation*}
	holds for all
	$\mathcal{T}\subset \Model, D, \text{and } D^\prime$.
\end{defin}
The \emph{$(\epsilon, \delta)$-certified unlearning} is similarly defined.

\begin{defin}
	\label{def-edcr}
	Under the assumptions of \cref{def-ecr}, an
	unlearning method $\mathcal{U}$ is
	$(\epsilon,\delta)$-certified if
	\begin{gather*}
		P\Big(\mathcal{U}\big(\mathcal{A}(D),D,D^\prime\big)
		\in\mathcal{T}\Big)
		\leq e^{\epsilon}P\big(\mathcal{A}(D^\prime)\in\mathcal{T}\big)+
		\delta \\
		\text{and}\\
		P\big(\mathcal{A}(D^\prime)\in\mathcal{T}\big) \leq e^{\epsilon}
		P\Big(\mathcal{U}\big(\mathcal{A}(D),D,D^\prime\big)\in
		\mathcal{T}\Big)+\delta
	\end{gather*}
	hold for all
	$\mathcal{T}\subset \Model, D, \text{and } D^\prime$.
\end{defin}

\subsubsection{Design for Certified Unlearning}

To achieve rigorously defined certified unlearning, our method incorporates two critical components:
\begin{enumerate}
    \item \textbf{Noise Injection during Training}: We add Gaussian noise $b \sim \mathcal{N}(0, \sigma^2 I)$ to the gradients during model training and unlearning, ensuring bounded parameter sensitivity.
    \item \textbf{First-Round Gradient Ascent}: During the initial unlearning update, we simultaneously perform:
    \begin{align*}
        \theta_{\text{unlearn}} = \theta^* -\tau \Big(\underbrace{\nabla\bigloss(\theta^* ; D^\prime)}_{\text{descent on new data}} - \underbrace{\nabla\bigloss(\theta^* ; D)}_{\text{ascent on old data}}\Big)
        %- \tau\bigg(\underbrace{\sum_{z \in \perts}\nabla\loss(z,\theta^*)}_{\text{descent on new data}} - \underbrace{\sum_{z \in \points}\nabla\loss(z,\theta^*)}_{\text{ascent on old data}}\bigg)
    \end{align*}
\end{enumerate}

\begin{theorem}[Certified Unlearning Guarantee]
\label{thm:main}
Assume the loss $\loss(\theta; z)$ is convex, $\gamma$-smooth with $L_2$ regularization $\frac{\lambda}{2}\|\theta\|_2^2$. For any data modification $(\points, \perts)$, our method achieves $(\epsilon, \delta)$-certified unlearning with
	$\delta=1.5e^{-c^2/2}$ when:
\begin{itemize}
\item Training noise $p\sim \mathcal{N}(0, c(1+\tau\gamma_z n)\gamma_zM\vert\points\vert/\epsilon)^d$ for some $c > 0$
    % $\sigma = \frac{c\sqrt{|\points|}}{\epsilon}$ for some $c > 0$
%     \item Unlearning rate $\tau \leq \frac{1}{\gamma n}$
\item Assume that $\Vert x_i\Vert_2 \leq 1$ for all data points and the
%\DIFdelbegin \DIFdel{loss }\DIFdelend
%\DIFaddbegin \DIFadd{gradient }\DIFaddend 
gradient
$\nabla\loss(z,\model)$ is $\gamma_z$-Lipschitz. Further let $\perts$ change the features $j,\dots,j+F$ by
magnitudes at most $m_j,\dots,m_{j+F}$, and $M=\sum_{j=1}^{F} m_j$ 
\end{itemize}

% The unlearned model satisfies:
% \begin{align*}
%     \mathbb{P}\left[\|\theta_{\text{unlearn}} - \mathcal{A}(D')\|_2 \leq \epsilon\right] \geq 1 - \delta
% \end{align*}
% where $\delta = 1.5e^{-c^2/2}$.
\end{theorem}

\begin{proof}[Proof Sketch]
The certification follows three key arguments:
\begin{enumerate}
    \item Gradient residual bound using Lipschitz continuity and our update rule
    \item Relating the gradient residual to the sensitivity of perturbation vector $b$
b using the $L_2$-regularized strong convexity.
    
    \item Applying Gaussian noise over models to yield $(\epsilon, \delta)$-guarantees.
\end{enumerate}
Full proof appears in Appendix~\ref{app:proof}.
\end{proof}

%-------------------------------------------------------------------------------
\section{Evaluation}
%-------------------------------------------------------------------------------

%In the experimental section, we evaluate the performance of our method in typical unlearning scenarios in VFL. 
While our method addresses asynchronous unlearning—a capability absent in existing approaches—our evaluation is structured in two stages.
Specifically, we aim to verify:% two key points: 
%\begin{enumerate}

\textbf{Synchronous Unlearning}. Our method performs comparably, or even better, than existing unlearning methods \cite{VFULR,VFUFR} when all clients are online.

\textbf{Asynchronous Unlearning}. During asynchronous unlearning, our method maintains comparable performance to synchronous unlearning.
    %When performing asynchronous unlearning, our method does not significantly underperform compared to the case where all clients are online.
%\end{enumerate}

\subsection{Experimental Setup}

\subsubsection*{Unlearning Scenarios}

%In the context of VFL, where different clients share samples but have different features, we identify three primary forgetting scenarios commonly encountered:
There are three common unlearning scenarios in VFL:

    1. \textbf{Client Removal:} A client exits the VFL system, and its data's influence must be completely removed. %Lemma~\ref{lemma1}, 
    This is equivalent to retraining the model by setting all data from the client to zero \cite{warnecke2023machineunlearningfeatureslabels}.
    
    2. \textbf{Feature Removal:} A certain feature (or feature set) from a client, which involves sensitive user information, is no longer available due to policy changes or other reasons. The influence of this feature must be removed from the system. %According to Lemma~\ref{lemma1}, 
    This is equivalent to retraining the model with the feature set to zero \cite{warnecke2023machineunlearningfeatureslabels}.
    
    3. \textbf{Sensitive Information Removal:} A feature from a client may be sensitive to a subset of users, prompting requests for removal (e.g., some users modify the visibility of the age field from `public' to `private'). %The influence of this sensitive information needs to be erased. 
    This can be achieved by retraining the model where the sensitive information is replaced by the mean value of that feature across all samples \cite{warnecke2023machineunlearningfeatureslabels}.

\begin{figure}[t]  
    \centering
     \includegraphics[width=0.4\columnwidth]{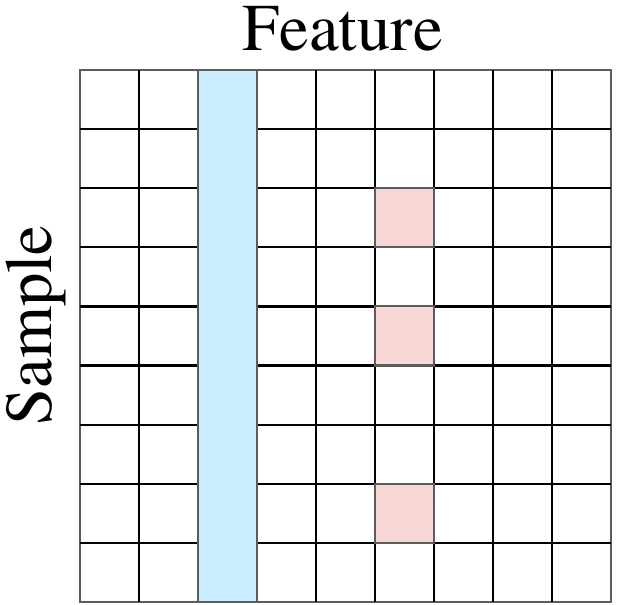} 
    \caption{Suppose the whole matrix is the data from one client. In \textit{client removal}, the whole matrix is removed. In \textit{feature removal}, the feature to be removed is marked with blue. In \textit{sensitive information removal}, the sensitive information to be removed is marked with red.}
    \label{scenary}
    \vspace{-1em}
\end{figure}

%These three scenarios are progressively more specific: a feature represents a subset of data from a client, and sensitive information is a subset of a feature's data. Their visual representations are shown in Figure ~\ref{scenary}.
The relationships between these three scenarios are illustrated in Figure~\ref{scenary}. As sensitive information may pertain to a single feature of a single sample (the smallest data unit), our method can easily extend to tasks such as \textit{sample removal} (forgetting all data from a single sample) and \textit{class removal} (forgetting all samples of a particular class). %While we focus on the three most common tasks in VFL, our method is also applicable to other forgetting scenarios.

% \begin{lemma}
% \cite{warnecke2023machineunlearningfeatureslabels}
% 	\label{lemma1}
% 	Consider a model $\optmodel$ trained on a dataset $D \subset \mathbb{R}^d$. If we remove a set of features $\feats$ from $D$ and retrain the model, the resulting optimal model is $\optmodel_{-\feats}$ with reduced input dimensions. In contrast, if we set the values of $\feats$ to zero and retrain, we obtain a model $\optmodel_{\feats=0}$ with the same input dimensions as $\optmodel$. For learning models processing inputs $x$ using linear
% 	transformations of the form $\model^{T} x$, we have
% 	$\optmodel_{-\feats} \equiv \optmodel_{\feats=0}$.
% \end{lemma}

\subsubsection*{Performance Measures}

Following \cite{VFUFR}, we conduct evaluation from two aspects,
%In machine unlearning, the success of unlearning is typically evaluated from three aspects, \textit{efficacy}, 
\textit{fedility} and \textit{efficiency}:

%\begin{enumerate}
    % \item \textbf{Efficacy} measures the effectiveness of removing the influence of the forgotten target on the model. The general criterion is that the model after unlearning should closely resemble the retrained model. Specific metrics include:
    % \begin{itemize}
    %     \item The \textit{gradient residual} of the model after unlearning. A value closer to zero indicates that the unlearned model is more similar to the retrained model. \footnote{The concept of gradient residual is originally proposed for convex functions. For models like MLP with non-convex loss functions, the region around the optimal point can be approximated as convex~\cite{choromanska2015losssurfacesmultilayernetworks}. During the unlearning process, the model parameter changes are usually limited, allowing the loss function to be approximated as convex.}

    %     \item The \textit{Kullback-Leibler divergence} between the outputs of the unlearned model and the retrained model on the test set. A value closer to zero indicates greater similarity between the two models.
    % \end{itemize}
    
    \textbf{Fidelity} assesses whether unlearning affects the utility of the model. The performance of the unlearned model should remain close to that of the original model, and the unlearning process should not significantly impact the model's practical utility. The metrics include \textit{Accuracy} and \textit{AUC} 
    %and \textit{loss} 
    of the unlearned model on the test set.
    
    \textbf{Efficiency} concerns the time taken for unlearning compared to retraining. The unlearning process should be as fast as possible. % relative to retraining. 
    %Additionally, asynchronous unlearning should have lower computational and communication costs compared to synchronous unlearning since fewer clients are required to participate in the unlearning process. 
    We compute efficiency using the total number of epochs and communication cost. %communication cost per epoch.
%\end{enumerate}

%\subsubsection*{Comparison with Baselines and Ablation Study}

\subsubsection*{Baselines} %and Our Method Variants}

We compare our method with three existing unlearning approaches.
%across the dimensions of %efficacy, 
%fidelity and efficiency, we selected three baselines.

    \textbf{VFULR~\cite{VFULR}:} VFULR addresses the client removal problem in the VFL logistic regression system. The method directly subtracts the contributions of the client to be forgotten from the sum of linear terms, followed by a single update. %A regularization term is added to ensure that the influence of any individual client does not dominate.
    
    \textbf{VFUFR~\cite{VFUFR}:} VFUFR addresses the client removal issue in general VFL systems, not limited to logistic regression. The method stores the original bottom model in each client to accelerate the retraining process, and it uses an enhanced optimizer for further speedup.
    
    \textbf{Retrain:} This method retrains the model from scratch on the new training dataset after data removal. It generally yields the best performance and serves as an upper bound for unlearning methods, though it is computationally expensive~\cite{nguyen2024surveymachineunlearning}.

Notably, VFULR and VFUFR are designed for VFL client-removal tasks on synchronous unlearning; we thus limit our comparison with these two baselines on this scenario.

%\textbf{Ablation study.} 
% \noindent \textbf{Our Method Variants.}
% To achieve certified unlearning, our method requires two components: (1) adding noise during the training of the original model, and (2) performing gradient ascent on the original data during the first update. We conduct an ablation study to evaluate the impact of each component, leading to the following method variants:
% %The experiments are divided into four groups based on whether these two conditions are met:
% \begin{itemize}
%     \item \textbf{Ours}: both (1) and (2) are satisfied.
%     \item \textbf{Ours without noise (Ours w.o. noise)}: (1) is not satisfied, but (2) is.
%     \item \textbf{Ours without gradient ascending (Ours w.o. ascend)}: (1) is satisfied, but (2) is not.
%     \item \textbf{Ours without gradient ascending without noise (Ours w.o. ascend\&noise)}: neither (1) nor (2) is satisfied.
% \end{itemize}

\subsection{Synchronous Client Removal}

In existing VFL literature, client removal is a commonly studied unlearning scenario. In this context, we address the first key question: whether our method performs comparably to or even better than existing unlearning approaches when all clients are online (i.e., \textit{synchronous unlearning}). 

To validate this, we select five datasets: \textit{Adult Income}~\cite{adult_2}, \textit{Credit}~\cite{default_of_credit_card_clients_350}, \textit{Diabetes}~\cite{diabetes_34}, \textit{Nursery}~\cite{nursery_76}, and \textit{Malware}~\cite{drebin}. %and News20-S5~\cite{news20}. 
Among these, \textit{Adult Income}, \textit{Credit}, \textit{Diabetes} and \textit{Nursery} are tabular datasets, while \textit{Malware} is a text-based dataset. For the text dataset, we extract bag-of-words features~\cite{zou2024vflairresearchlibrarybenchmark}. %In a VFL system, samples are shared across clients, but features are not. It is uncommon to partition an image dataset into separate features distributed across different clients, making VFL unsuitable for image data. Consequently, we did not include image datasets in our evaluation. 
We split 80\% of the data for training and reserve 20\% for testing. Table~\ref{tab:datasets} overviews the datasets and their vertical federated setting. Without loss of generality, we set one client for removal; the number of features in the removed client is in Table~\ref{tab:datasets}.

\begin{table}[t]
	\begin{center}    
		\caption{Datasets for unlearning.}
		%\vspace{-5pt}
		\label{tab:datasets}
		%\tablesize
        \resizebox{\columnwidth}{!}{%
		\begin{tabular}{lcccccc}

			\toprule
			& \textbf{Adult Income}  & \textbf{Credit}&\textbf{Diabetes} 
& \textbf{Nursery} & \textbf{Malware} %& \textbf{News20-S5}
\\
			\midrule
			\#Samples & 48842  & 30000&768 
& 10950& 49226%& 5000
\\
			\#Features  & 108& 23&8 & 19& 2081 %& 87748
            \\
			\midrule
 \#Clients& 16& 4& 4& 4& 16&%2
 \\
 %Active Party& 15& 3& 3& 3& 15&1\\
 %Client Removed& 0& 0& 0& 0& 0&0\\
 \#Feat. in Removed Client & 27& 6& 2& 7& 520%&21937
 \\
  \bottomrule
		\end{tabular}
        }
	\end{center}
\end{table}

For each dataset, we conduct experiments using both the LR and MLP models, resulting in a total of 6×2=12 experimental setups. Since VFULR is designed only for the LR model, there is no VFULR results in the MLP model experiments. %For detailed experimental results, please refer to Tables \ref{tab:b}. %In this scenario, our approach assumes that all clients are online. 
%Some results are listed in Appendix D.

%\textbf{Efficacy}: The smaller the values of the gradient residual and KL divergence metrics, the closer the unlearned model is to the retrained model. In the "Retrain" column of the tables, the KL divergence is always 0 because the comparison baseline for KL divergence is itself. However, the gradient residual is not always zero, as in practice, early stopping is used during training, which may prevent the model from reaching the theoretical optimal point of the loss function. The gradient residual values for the retrained models are small and can be considered negligible errors.  When aggregating results across all 6 datasets, our method exhibits performance that is very similar to VFUFR, and both outperform VFULR by a significant margin. Notably, for the News20-S5 dataset, the gradient residual of our method is even smaller than that of the retrained model. This can be attributed to the relatively large feature-to-sample ratio in that datasets, which may result in smaller gradients near the optimal point of the loss function, making it difficult to converge exactly to the theoretical minimum. The three variants of our method in the ablation study show minimal deviation from the original approach.

\textbf{Fidelity:} Table \ref{tab:f-lr} and \ref{tab:f-mlp} show the fidelity results for both LR and MLP model unlearning. Higher accuracy and AUC indicate better performance. As expected, Retrain performs the best in most cases, while it generally needs much more epochs to converge (we will show in the efficiency results). Our method performs similarly to VFUFR and outperforms VFULR. In datasets with a higher number of features, the performance difference between retraining and various unlearning methods is not significant. In contrast, in datasets like Diabetes, which have fewer features, client removal results in a noticeable drop in model performance. This may be due to the fact that, in datasets with many features, the clients removed during unlearning likely do not contain high-information features. %The three variants in our ablation study exhibit performance that is close to that of the original method.

\textbf{Efficiency:} Table~\ref{tab:e-lr} and \ref{tab:e-mlp} show the efficiency results. We measure training time in terms of epochs. VFULR performs only one update, thus being the most efficient; but its fidelity is not satisfactory as aforementioned.
Other methods can perform multiple updates. We employed an early stopping mechanism, limiting retraining to a maximum of 400 epochs and other methods to a maximum of 50 epochs, since  the retraining method typically requires more epochs to converge (the hyper-parameter choice will be discussed in Section ~\ref{sec:choice}). The results show that the retraining method typically requires hundreds of epochs, significantly higher than other methods. Our method has a similar training efficiency as VFUFR.
We also evaluate communication cost per epoch. %, which reflects the burden of data transfer in a VFL system. 
All the methods transfer similar data amount in each epoch.
%Without the initial gradient ascent step, our methods (Ours w.o. ascend, Ours w.o. ascend\&noise) and the retraining and VFUFR methods have similar communication costs per epoch. When the first round of gradient ascent is performed, our methods (Ours, Ours w.o. noise) incur a slightly higher communication cost per epoch, though still outperforming VFULR.

In summary, for synchronous VFL unlearning in client-removal scenarios, our method achieves performance comparable to the state-of-the-art VFUFR framework. Meanwhile, our approach surpasses VFUFR by offering unique advantages: greater flexibility in handling diverse unlearning scenarios (e.g., feature or sensitive data removal), the ability to perform asynchronous unlearning when not all clients are online, and theoretically grounded certified unlearning guarantees. These additional capabilities will be evaluated subsequently.

\begin{table*}[t]
\footnotesize
	\begin{center}    
		\caption{Fidelity of VFL Client Removal for LR models on Synchronous Unlearning.}
		%\vspace{-5pt}
		\label{tab:f-lr}
		%\tablesize
		\begin{tabular}{lcccccccccccc}

			\toprule
            &\multicolumn{2}{c}{\textbf{Adult Income}} &\multicolumn{2}{c}{\textbf{Credit}} &\multicolumn{2}{c}{\textbf{Diabetes}} &\multicolumn{2}{c}{\textbf{Nursery}} &\multicolumn{2}{c}{\textbf{Malware}} &\multicolumn{2}{c}{\textbf{Average}}\\
           \cmidrule(l){2-3}\cmidrule(l){4-5}\cmidrule(l){6-7}\cmidrule(l){8-9}\cmidrule(l){10-11}\cmidrule(l){12-13}
            \textbf{Metric} & \textbf{Accuracy} & \textbf{AUC} & \textbf{Accuracy} & \textbf{AUC} & \textbf{Accuracy} & \textbf{AUC} & \textbf{Accuracy} & \textbf{AUC} & \textbf{Accuracy} & \textbf{AUC} & \textbf{Accuracy} & \textbf{AUC} \\
			\midrule
			\textbf{VFULR} & $0.832$ & $0.886$ & $0.779$ & $0.689$ & $0.630$ & $0.652$ & $0.737$ & $0.908$ & $0.975$ & $0.933$ & $0.791$ & $0.814$ 
\\
			\textbf{VFUFR} & $0.847$ & $0.896$ & $0.781$ & $0.688$ & $0.695$ & $0.747$ & $0.728$ & $0.906$ & $0.979$ & $0.947$ & $\underline{0.806}$ & $\underline{0.837}$ 
\\
			\textbf{Ours} & $0.846$ & $0.896$ & $0.781$ & $0.688$ & $0.695$ & $0.740$ & $0.730$ & $0.906$ & $0.979$ & $0.948$ & $\underline{0.806}$ & $0.836$ 

\\
            \midrule
            \textbf{Retrain} & $0.853$ & $0.905$ & $0.797$ & $0.686$ & $0.695$ & $0.782$ & $0.731$ & $0.906$ & $0.979$ & $0.954$ & $\textbf{0.811}$ & $\textbf{0.847}$

\\
			\bottomrule
		\end{tabular}
	\end{center}
\end{table*}

\begin{table*}[t]
	\footnotesize
	\begin{center}    
		\caption{Fidelity of VFL Client Removal for MLP models on Synchronous Unlearning.}
		%\vspace{-5pt}
		\label{tab:f-mlp}
		%\tablesize
		\begin{tabular}{lcccccccccccc}

			\toprule
            &\multicolumn{2}{c}{\textbf{Adult Income}} &\multicolumn{2}{c}{\textbf{Credit}} &\multicolumn{2}{c}{\textbf{Diabetes}} &\multicolumn{2}{c}{\textbf{Nursery}} &\multicolumn{2}{c}{\textbf{Malware}} &\multicolumn{2}{c}{\textbf{Average}}\\
            \cmidrule(l){2-3}\cmidrule(l){4-5}\cmidrule(l){6-7}\cmidrule(l){8-9}\cmidrule(l){10-11}\cmidrule(l){12-13}
            \textbf{Metric} & \textbf{Accuracy} & \textbf{AUC} & \textbf{Accuracy} & \textbf{AUC} & \textbf{Accuracy} & \textbf{AUC} & \textbf{Accuracy} & \textbf{AUC} & \textbf{Accuracy} & \textbf{AUC} & \textbf{Accuracy} & \textbf{AUC} \\
			\midrule
			\textbf{VFUFR} & $0.859$ & $0.912$ & $0.776$ & $0.743$ & $0.623$ & $0.709$ & $0.735$ & $0.910$ & $0.983$ & $0.957$ & $0.795$ & $\underline{0.846}$ 
 
\\
			\textbf{Ours} & $0.859$ & $0.911$ & $0.778$ & $0.742$ & $0.636$ & $0.704$ & $0.737$ & $0.910$ & $0.982$ & $0.956$ & $\underline{0.798}$ & $0.845$
\\ \midrule
            \textbf{Retrain} & $0.859$ & $0.912$ & $0.805$ & $0.751$ & $0.740$ & $0.818$ & $0.749$ & $0.911$ & $0.984$ & $0.958$ & $\mathbf{0.827}$ & $\mathbf{0.870}$ 

\\
			\bottomrule
		\end{tabular}
	\end{center}
\end{table*}

\begin{table*}[t]
	\small
	\begin{center}    
		\caption{Efficiency of VFL Client Removal for LR models on Synchronous Unlearning.}
		%\vspace{-5pt}
		\label{tab:e-lr}
		%\tablesize
        \resizebox{\textwidth}{!}{%
		\begin{tabular}{lcccccccccccc}

			\toprule
            &\multicolumn{2}{c}{\textbf{Adult Income}} &\multicolumn{2}{c}{\textbf{Credit}} &\multicolumn{2}{c}{\textbf{Diabetes}} &\multicolumn{2}{c}{\textbf{Nursery}} &\multicolumn{2}{c}{\textbf{Malware}}\\
            \cmidrule(l){2-3}\cmidrule(l){4-5}\cmidrule(l){6-7}\cmidrule(l){8-9}\cmidrule(l){10-11}
            \textbf{Metric} & \textbf{\#Epoch} & \textbf{Comm./Epoch} & \textbf{\#Epoch} & \textbf{Comm./Epoch} & \textbf{\#Epoch} & \textbf{Comm./Epoch} & \textbf{\#Epoch} & \textbf{Comm./Epoch} & \textbf{\#Epoch} & \textbf{Comm./Epoch}\\
			\midrule
			\textbf{VFULR} & $1$ & $4.17$ & $1$ & $0.732$ & $1$ & $0.0187$ & $1$ & $0.791$ & $1$ & $4.81$\\
			\textbf{VFUFR} & $50$ & $3.91$ & $50$ & $0.553$ & $50$ & $0.0141$ & $50$ & $0.597$ & $50$ & $4.51$\\
			\textbf{Ours} & $50$ & $3.99$ & $50$ & $0.576$ & $50$ & $0.0144$ & $50$ & $0.609$ & $50$ & $4.60$\\ \midrule
            \textbf{Retrain} & $180$ & $3.91$ & $400$ & $0.553$ & $129$ & $0.0141$ & $400$ & $0.597$ & $209$ & $4.51$\\
			\bottomrule
		\end{tabular}
        }
	\end{center}
\end{table*}

\begin{table*}[t]
	%\small
	\begin{center}    
		\caption{Efficiency of VFL Client Removal for MLP models on Synchronous Unlearning.}
		%\vspace{-5pt}
		\label{tab:e-mlp}
		%\tablesize
        \resizebox{\textwidth}{!}{%
		\begin{tabular}{lcccccccccccc}

			\toprule
            &\multicolumn{2}{c}{\textbf{Adult Income}} &\multicolumn{2}{c}{\textbf{Credit}} &\multicolumn{2}{c}{\textbf{Diabetes}} &\multicolumn{2}{c}{\textbf{Nursery}} &\multicolumn{2}{c}{\textbf{Malware}}\\
            \cmidrule(l){2-3}\cmidrule(l){4-5}\cmidrule(l){6-7}\cmidrule(l){8-9}\cmidrule(l){10-11}
            \textbf{Metric}     & \textbf{\#Epoch} & \textbf{Comm./Epoch} & \textbf{\#Epoch} & \textbf{Comm./Epoch} & \textbf{\#Epoch} & \textbf{Comm./Epoch} & \textbf{\#Epoch} & \textbf{Comm./Epoch} & \textbf{\#Epoch} & \textbf{Comm./Epoch}\\
			\midrule
			\textbf{VFUFR} & $50$ & $3.91$ & $50$ & $0.553$ & $50$ & $0.0141$ & $50$ & $0.597$ & $50$ & $4.51$\\
			\textbf{Ours} & $50$ & $3.99$ & $50$ & $0.564$ & $50$ & $0.0144$ & $50$ & $0.609$ & $50$ & $4.60$\\ \midrule
            \textbf{Retrain} & $400$ & $3.91$ & $236$ & $0.553$ & $400$ & $0.0141$ & $168$ & $0.597$ & $225$ & $4.51$\\
			\bottomrule
		\end{tabular}
        }
	\end{center}
\end{table*}

\subsection{Sync. \& Async. Feature Removal}

Since VFULR and VFUFR do not support the feature removal scenario, our analysis is exclusively benchmarked against Retrain. We further evaluate our method on both synchronous and asynchronous settings, demonstrating comparable performance between the two settings. For the asynchronous unlearning setting, we suppose that a random 25\% of clients will drop out in each epoch.
%Feature removal is a further refinement of the client removal scenario. In this case, there is no baseline method, so we address the second research question: whether our method performs comparably in asynchronous unlearning as it does when all clients are online. 
Due to the page limitation, we show experiment results on two datasets: \textit{Adult Income} and \textit{Malware}. The experimental setup is similar to that of the client removal scenario, except that in this case, we unlearn a portion of the features from the clients rather than unlearning entire clients. The results are presented in Figure ~\ref{Feature Removal Results}.

\textbf{Fidelity:}  The first row of Figure \ref{Feature Removal Results} shows the test accuracy metric. The performance of our method is very similar to retraining, with the test accuracy decreasing only slightly as the proportion of features removed increases, which we have already analyzed in the client removal scenario. %The three variants of our method in the ablation study show little deviation from the original approach, and 
Specifically, there is no significant difference between the accuracy of our method on synchronous or asynchronous unlearning, verifying the effectiveness on the asynchronous setting. The AUC results are similar and not shown due to the page limitation.

\textbf{Efficiency:}  The second and third rows of Figure \ref{Feature Removal Results} display the epoch number and communication cost per epoch. From the figure, it is evident that the change in the proportion of feature removal has minimal effect on the efficiency metrics. In terms of the number of training epochs, the retraining method requires several times more epochs than our methods. %The communication cost per epoch result is similar to what we analyzed in the client removal scenario. %When gradient ascent is applied in the first round, the communication cost per epoch for our method is slightly higher. 
It is noteworthy that our method on asynchronous unlearning incurs approximately 25\% less communication cost than synchronous unlearning per epoch, as fewer clients are online, reducing the communication burden. The proportion of reduction in communication cost is directly related to the proportion of clients that are offline.

\begin{figure*}[h]  % 使用 figure* 环境，使图片跨越两栏
    \centering

    % 第三列
    \begin{minipage}[b]{0.23\textwidth}
        \centering
        \includegraphics[width=\textwidth]{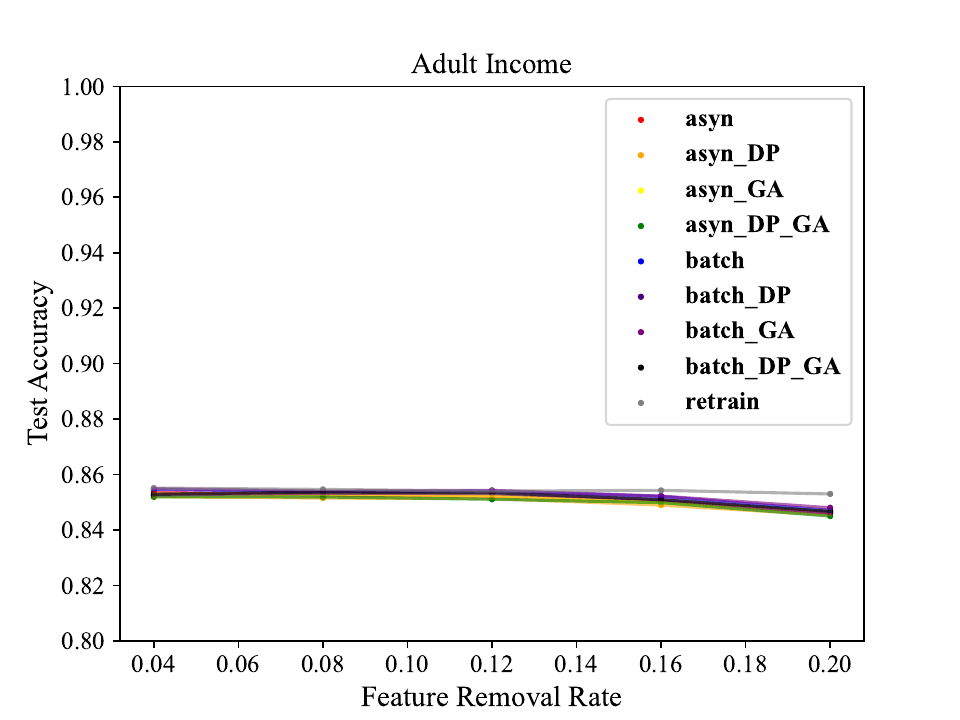}
        % \caption{a3}
    \end{minipage}
    \hfill
    \begin{minipage}[b]{0.23\textwidth}
        \centering
        \includegraphics[width=\textwidth]{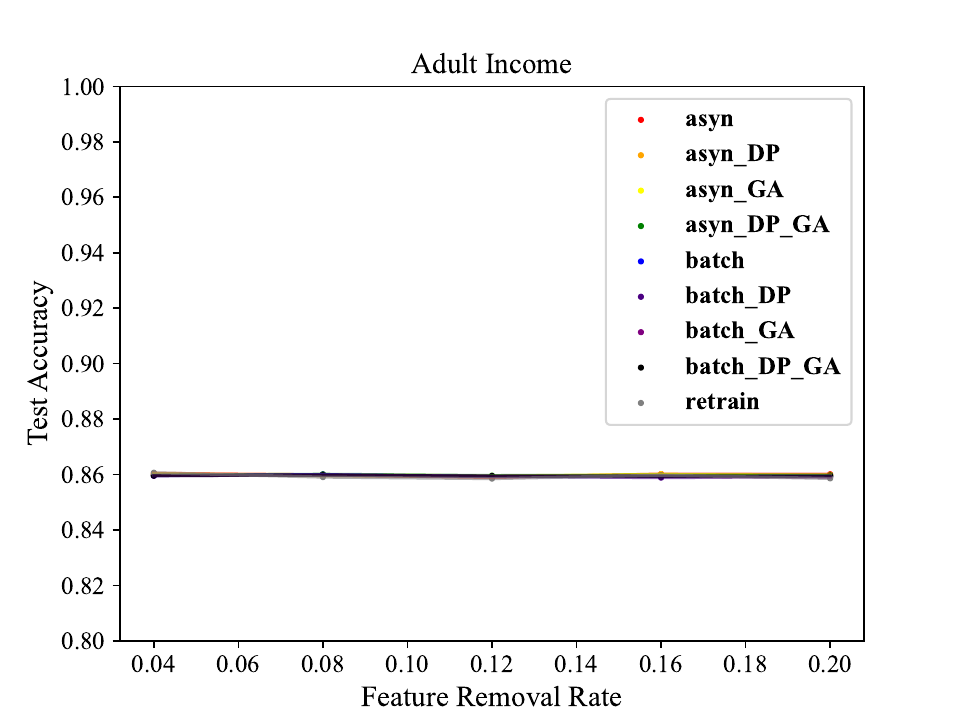}
        % \caption{子图3}
    \end{minipage}
    \hfill
    \begin{minipage}[b]{0.23\textwidth}
        \centering
        \includegraphics[width=\textwidth]{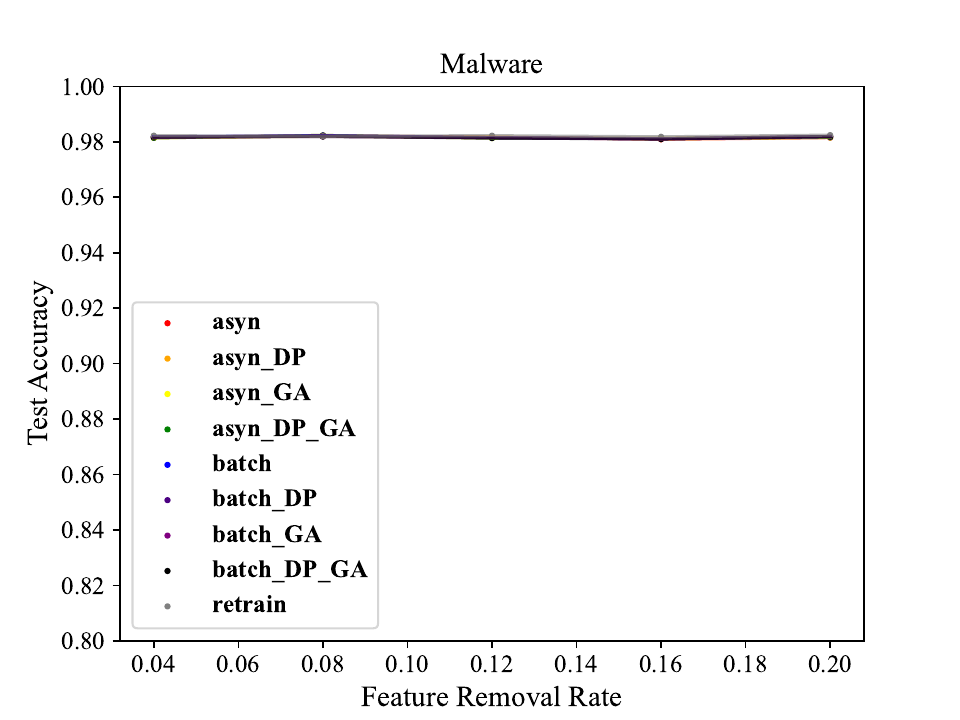}
        % \caption{子图3}
    \end{minipage}
    \hfill
    \begin{minipage}[b]{0.23\textwidth}
        \centering
        \includegraphics[width=\textwidth]{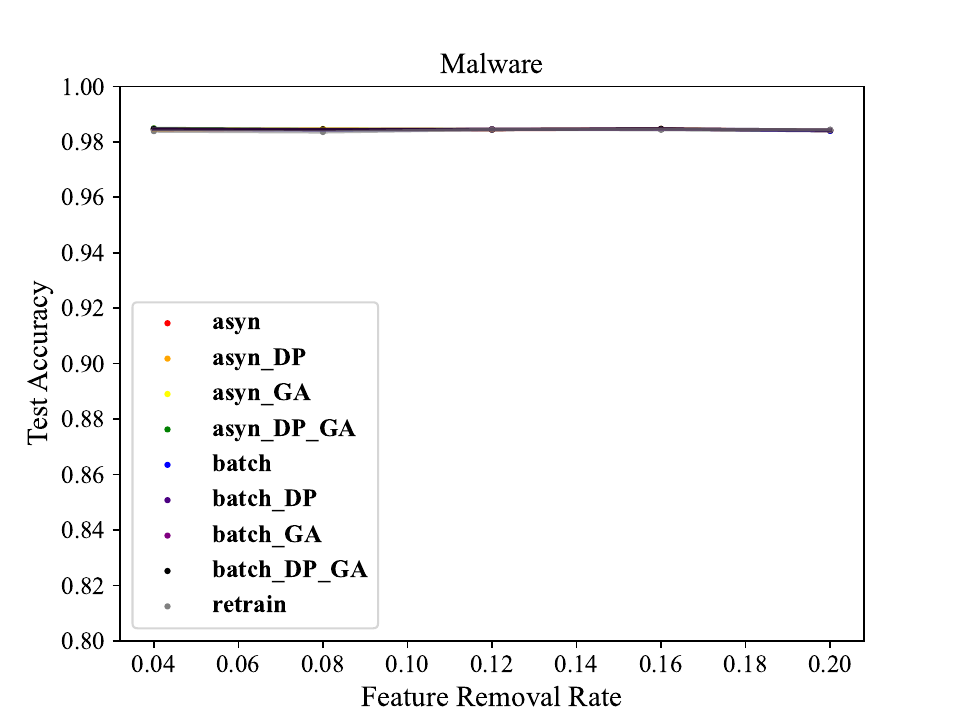}
        % \caption{子图3}
    \end{minipage}

    \vspace{0.5cm}  % 行间距

    % 第六列
    \begin{minipage}[b]{0.23\textwidth}
        \centering
        \includegraphics[width=\textwidth]{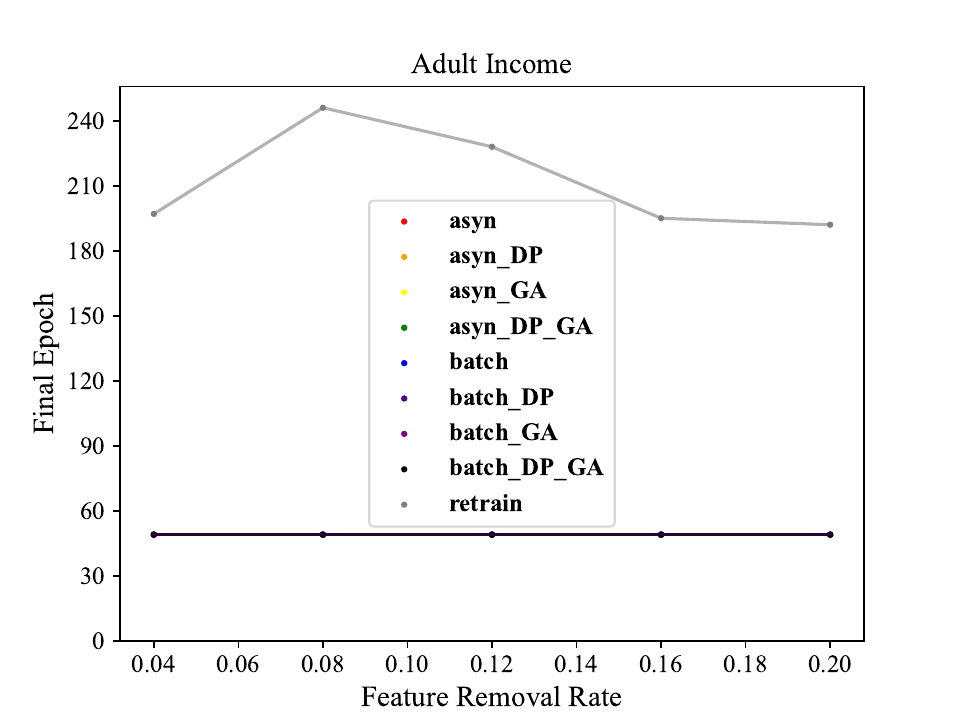}
        % \caption{a6}
    \end{minipage}
    \hfill
    \begin{minipage}[b]{0.23\textwidth}
        \centering
        \includegraphics[width=\textwidth]{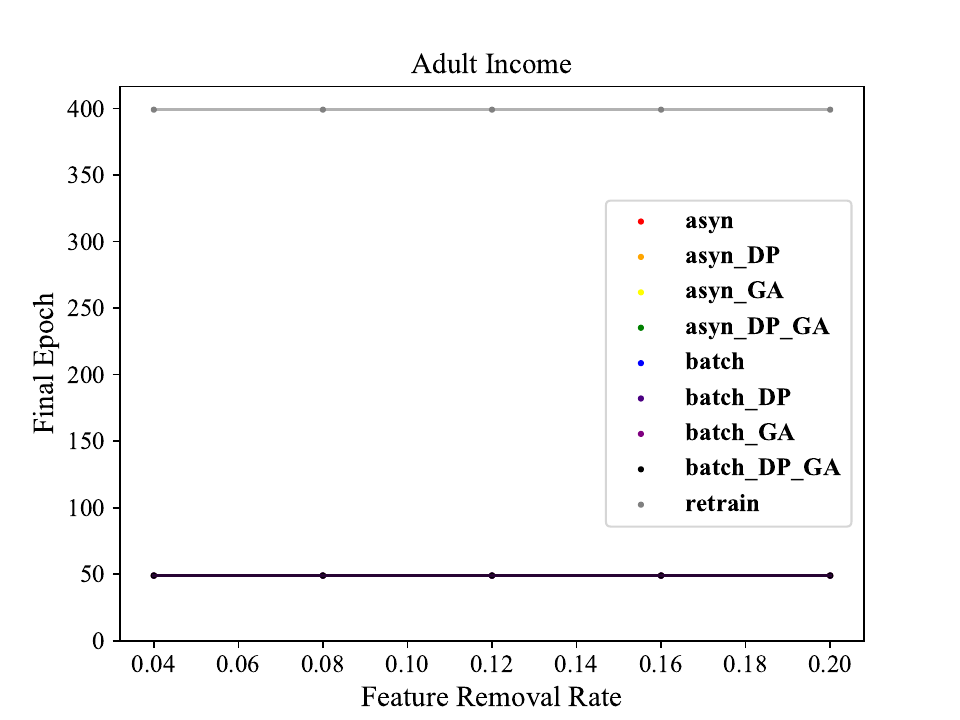}
        % \caption{子图6}
    \end{minipage}
    \hfill
    \begin{minipage}[b]{0.23\textwidth}
        \centering
        \includegraphics[width=\textwidth]{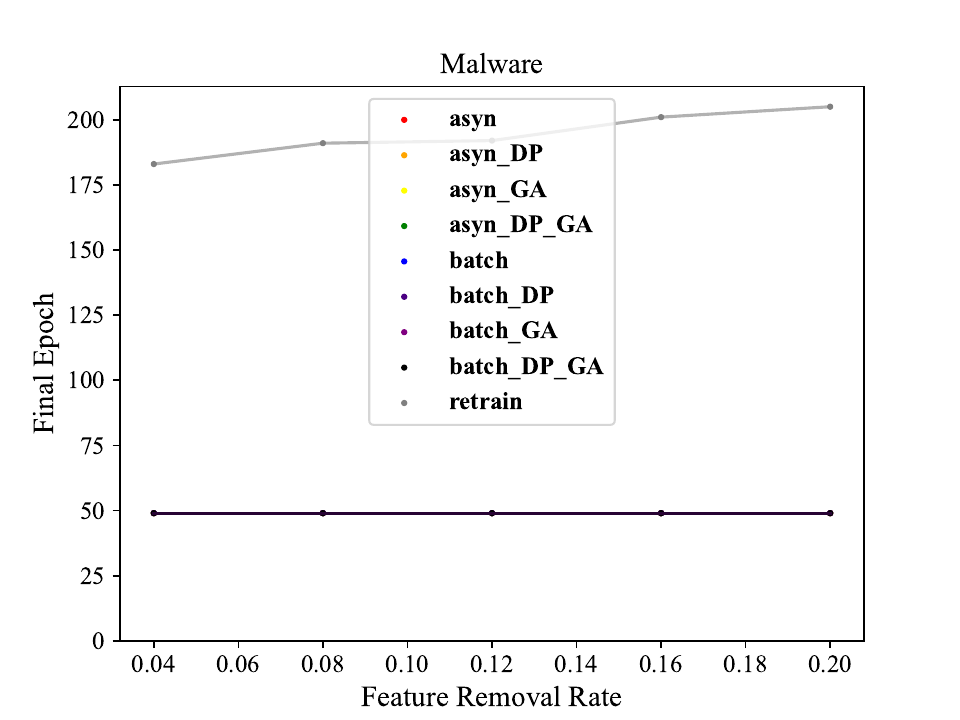}
        % \caption{子图6}
    \end{minipage}
    \hfill
    \begin{minipage}[b]{0.23\textwidth}
        \centering
        \includegraphics[width=\textwidth]{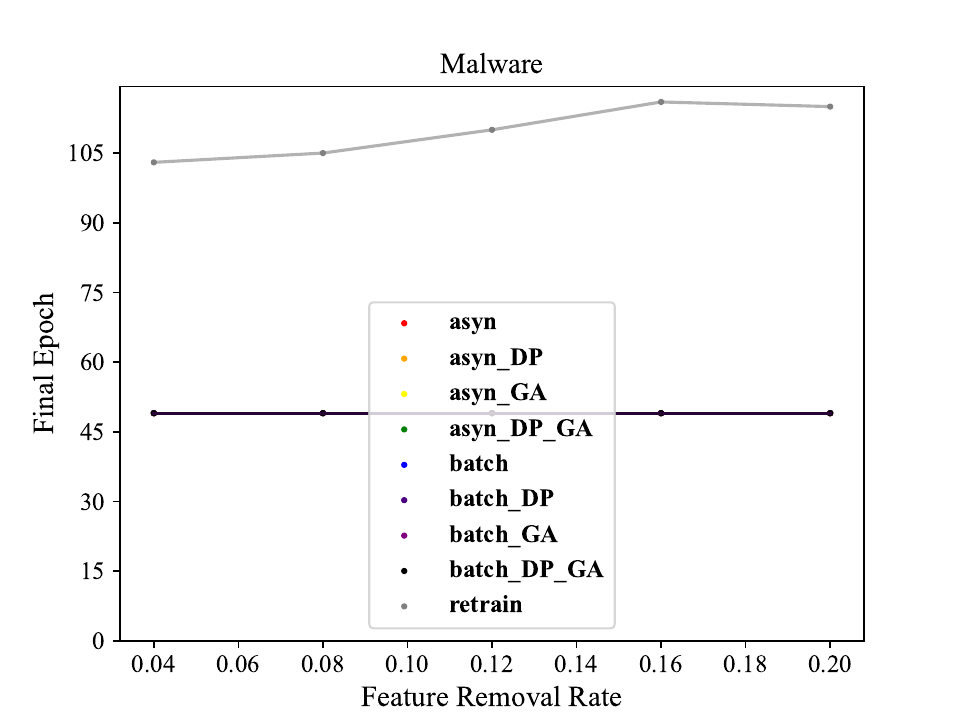}
        % \caption{子图6}
    \end{minipage}

    \vspace{0.5cm}  % 行间距

    % 第七列
    \begin{minipage}[b]{0.23\textwidth}
        \centering
        \includegraphics[width=\textwidth]{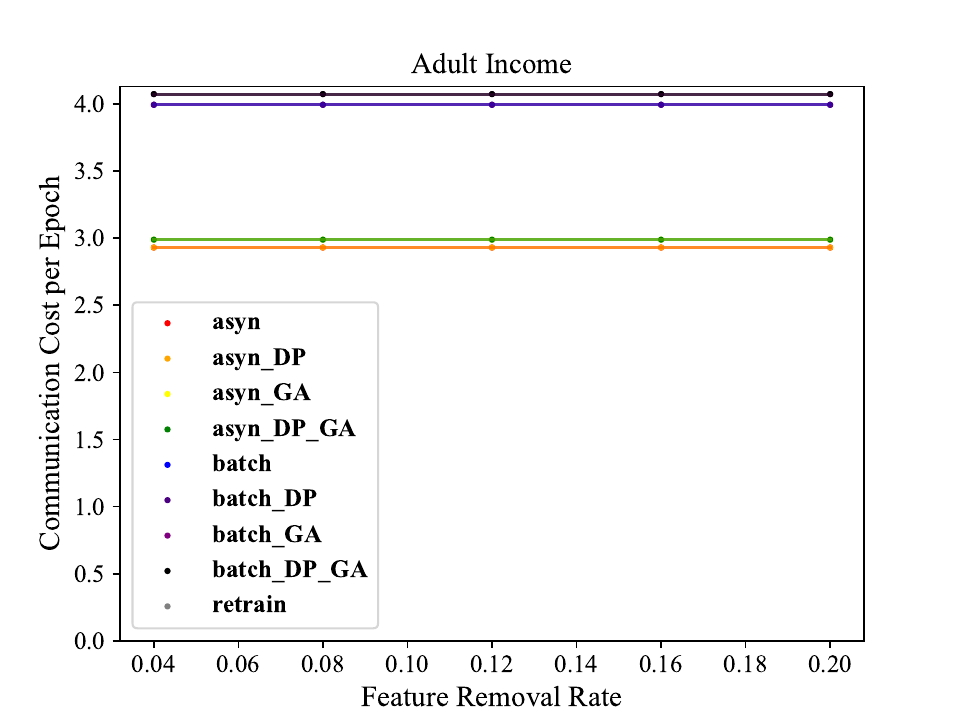}
        % \caption{a7}
    \end{minipage}
    \hfill
    \begin{minipage}[b]{0.23\textwidth}
        \centering
        \includegraphics[width=\textwidth]{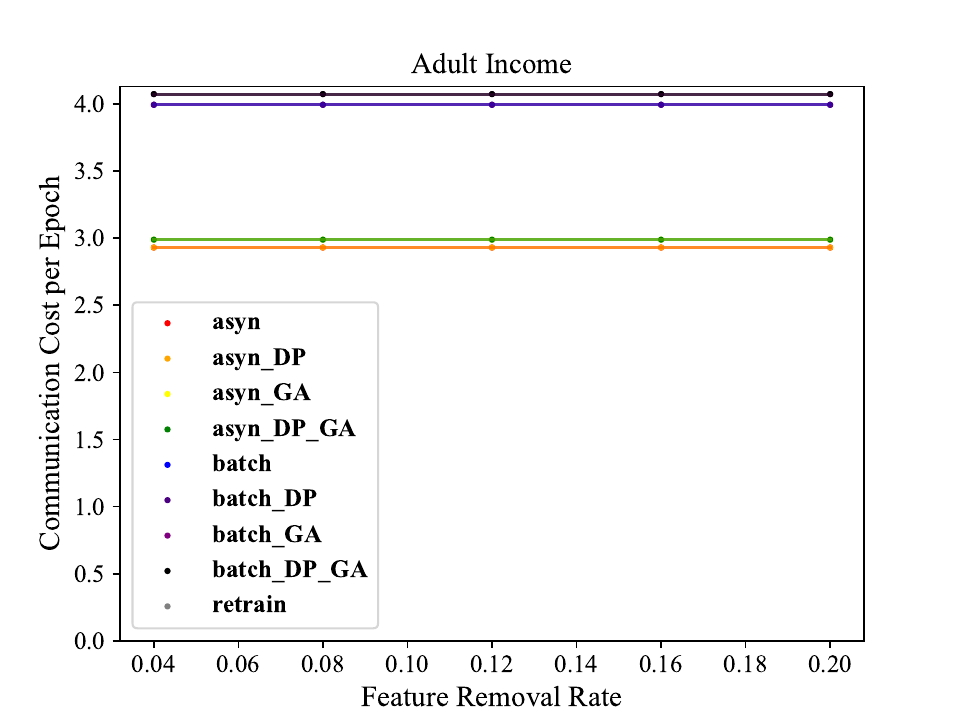}
        % \caption{子图7}
    \end{minipage}
    \hfill
    \begin{minipage}[b]{0.23\textwidth}
        \centering
        \includegraphics[width=\textwidth]{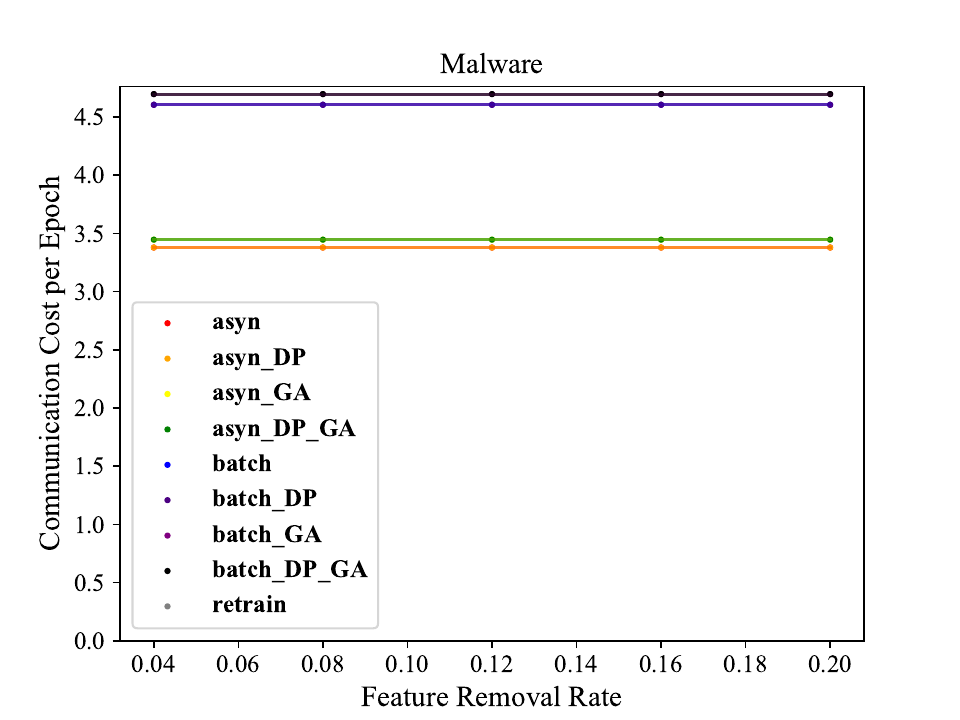}
        % \caption{子图7}
    \end{minipage}
    \hfill
    \begin{minipage}[b]{0.23\textwidth}
        \centering
        \includegraphics[width=\textwidth]{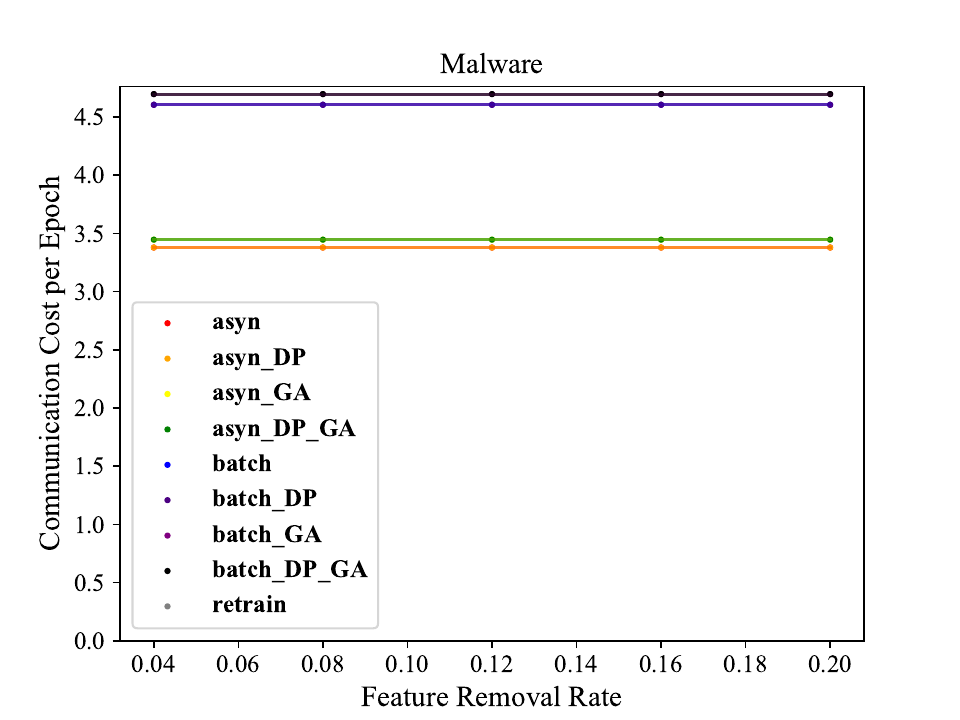}
        % \caption{子图7}
    \end{minipage}
    
    \caption{Feature Removal Results.}
    \label{Feature Removal Results}
\end{figure*}

\subsection{Sync. \& Async. Sensitive Info. Removal}

Sensitive information removal represents a more fine-grained refinement of the feature removal scenario, where only specific features from certain samples within a client are removed. %We continue to investigate the second question: whether the performance of our method in asynchronous unlearning is not significantly worse than when all clients are online. 
%We carefully select two datasets, Credit and Diabetes, for experimentation as they include sensitive attributes that users may want to hide from public. Specifically, the Credit dataset includes the age feature, while the Diabetes dataset includes the number of pregnancies feature. These features may be considered sensitive information for some users, so we focus on how to unlearn these data points. The results are shown in Figure ~\ref{Sensitive Information Removal Results.}. %The figures are organized into four columns based on the experimental datasets and models: from left to right, Credit-LR, Credit-MLP, Diabetes-LR, and Diabetes-MLP. Due to page limitations, it is not possible to present all the performance metrics in the main text. For additional results, please refer to Appendix C. In each subplot, the x-axis represents the proportion of the number of values forgotten in this feature relative to the total number, and the y-axis represents the corresponding performance metric.
We select two datasets, \textit{Credit} and \textit{Diabetes}, for experimentation due to their inclusion of sensitive attributes that users might prefer to keep private. Specifically, \textit{Credit} contains the `age' feature, while \textit{Diabetes} includes the `number of pregnancies' feature. These attributes may be regarded as sensitive by some users, prompting our focus on unlearning such data points. The results are presented in Figure ~\ref{Sensitive Information Removal Results.}.

% \textbf{Efficacy:} The first row in the figure shows the gradient residual metric. As the proportion of sensitive information removal increases, the gradient residual also increases overall. In the Credit dataset, the gradient residual curves for the batch method and the asynchronous method are very close to each other, and they are almost identical to the retrained model curve, indicating that the forgetting process is very effective and nearly identical to retraining the model. However, in the Diabetes dataset, both the asynchronous and batch methods have significantly higher gradient residuals compared to the retrained curve, with the asynchronous method being slightly higher than the batch method. This phenomenon can be explained by the analysis in the Client Removal section. The limited number of features in the Diabetes dataset likely contribute to the importance of the forgotten information. Moreover, the three ablation study variants of our method show almost no difference compared to the original method.

\textbf{Fidelity:} The first row in the figure presents the test accuracy metric. As the proportion of sensitive information removal increases, the test accuracy generally decreases. This decrease is more pronounced in \textit{Diabetes} which includes only 8 features in total, and the retrained model curve is slightly higher than that of our method, indicating some performance gap. This observation is also consistent with previous results in Sec.  that datasets with s %The three ablation study variants of our method show minimal differences from the original method, and there is almost no gap between the asynchronous and batch methods.

\textbf{Efficiency:} The third and fourth rows in the figure show the number of epochs and the communication cost per epoch. For the training epochs, although there are some fluctuations in the results, the retrained method requires several times more epochs than the unlearning methods. In the communication cost per epoch for the Diabetes dataset, when the sensitive information removal rate is close to 0.2, the communication cost per epoch of the asynchronous method is actually higher than that of the synchronous method. This occurs because the asynchronous method converges after approximately 20 epochs, and the additional communication overhead for evaluation is divided by a smaller denominator, resulting in a relatively large impact.

\begin{figure*}[h]  % 使用 figure* 环境，使图片跨越两栏
    \centering

    % 第三列
    \begin{minipage}[b]{0.23\textwidth}
        \centering
        \includegraphics[width=\textwidth]{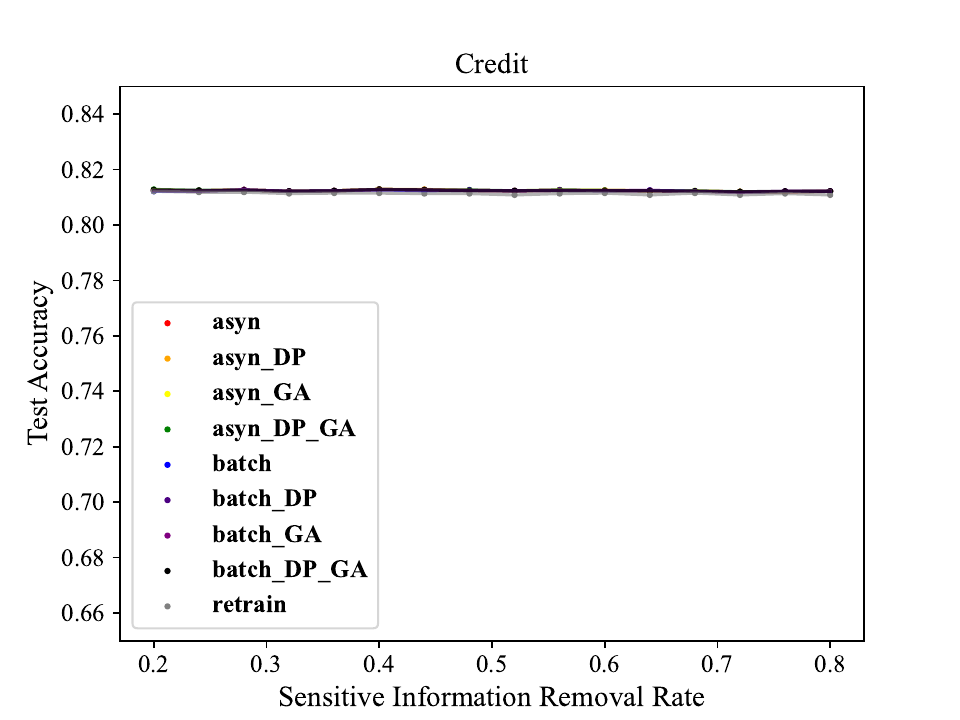}
        % \caption{a3}
    \end{minipage}
    \hfill
    \begin{minipage}[b]{0.23\textwidth}
        \centering
        \includegraphics[width=\textwidth]{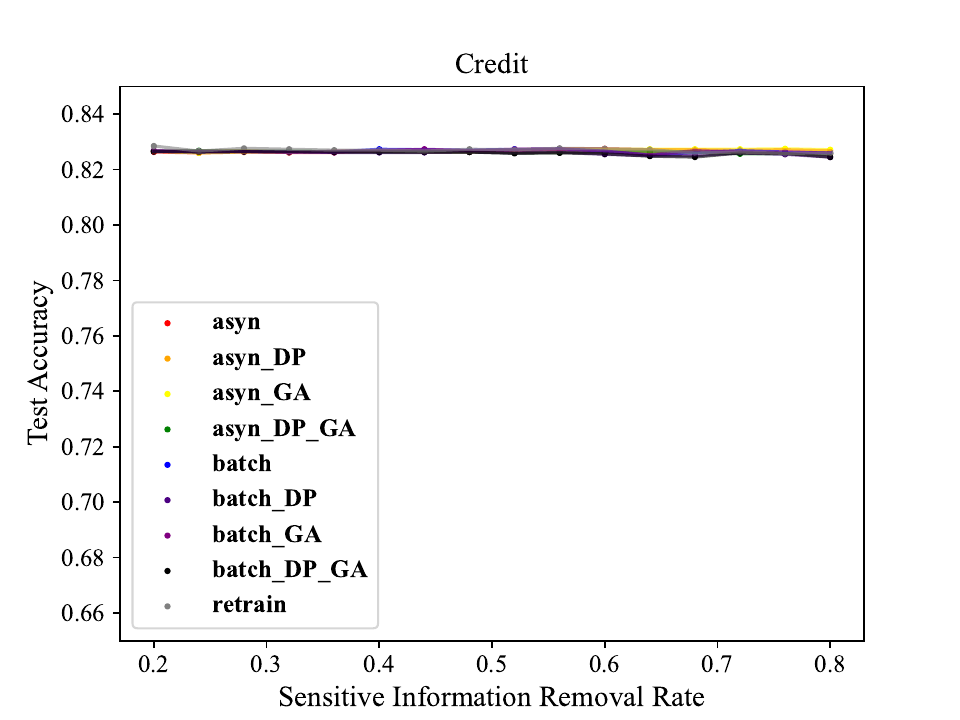}
        % \caption{子图3}
    \end{minipage}
    \hfill
    \begin{minipage}[b]{0.23\textwidth}
        \centering
        \includegraphics[width=\textwidth]{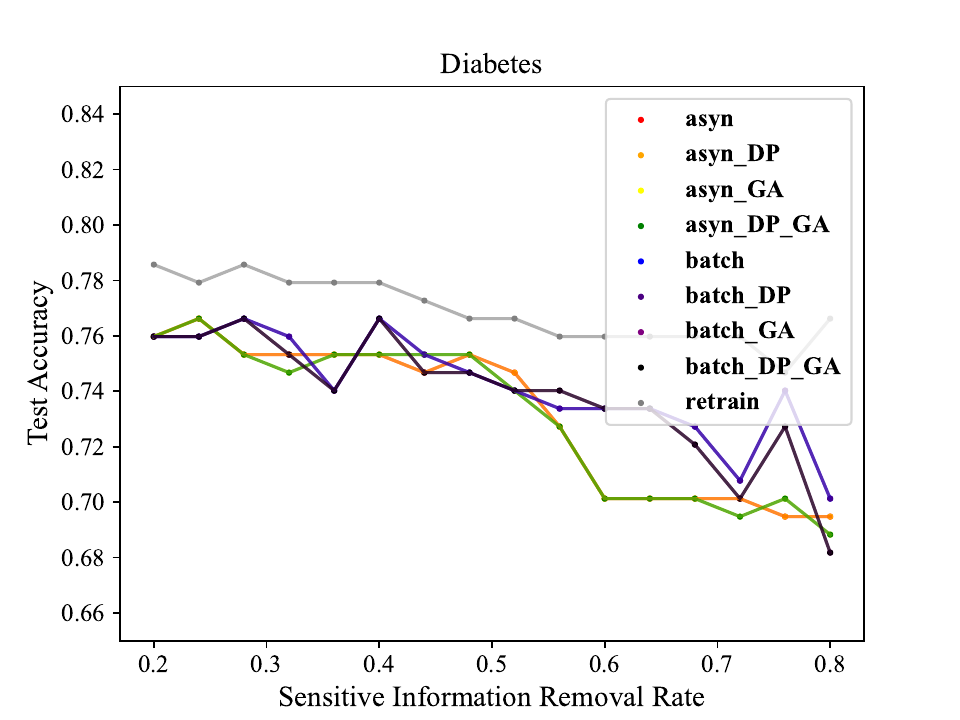}
        % \caption{子图3}
    \end{minipage}
    \hfill
    \begin{minipage}[b]{0.23\textwidth}
        \centering
        \includegraphics[width=\textwidth]{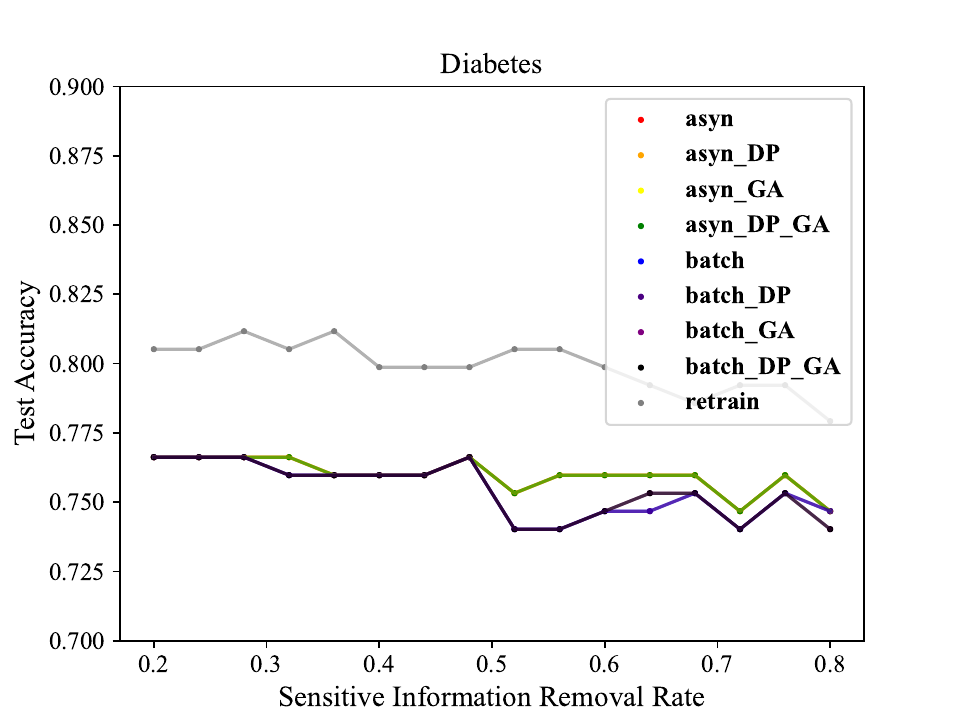}
        % \caption{子图3}
    \end{minipage}

    \vspace{0.5cm}  % 行间距

    % 第六列
    \begin{minipage}[b]{0.23\textwidth}
        \centering
        \includegraphics[width=\textwidth]{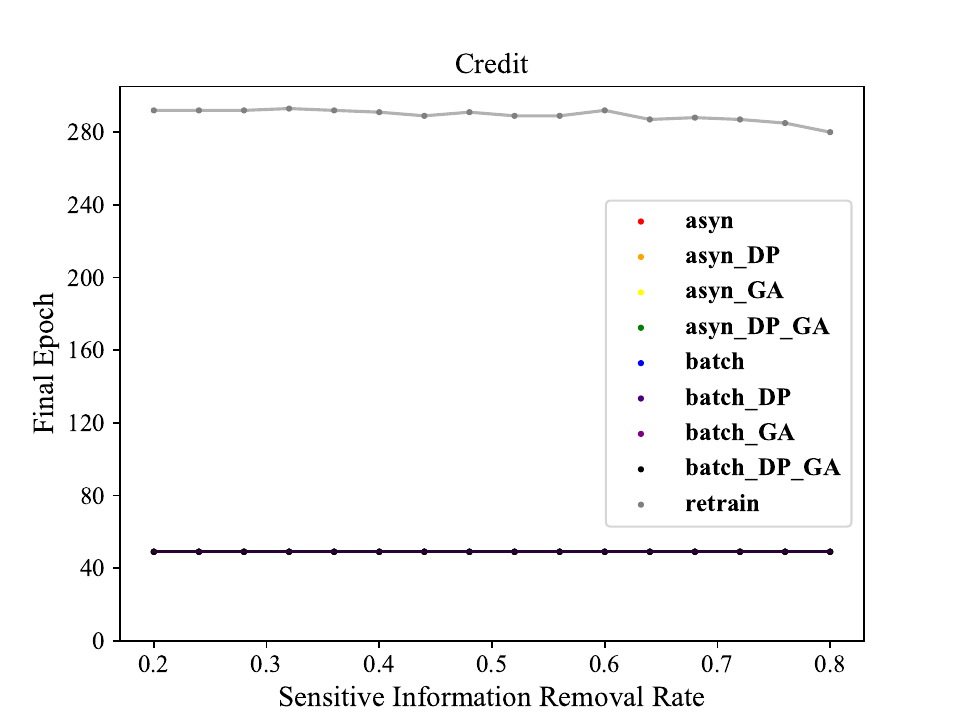}
        % \caption{a6}
    \end{minipage}
    \hfill
    \begin{minipage}[b]{0.23\textwidth}
        \centering
        \includegraphics[width=\textwidth]{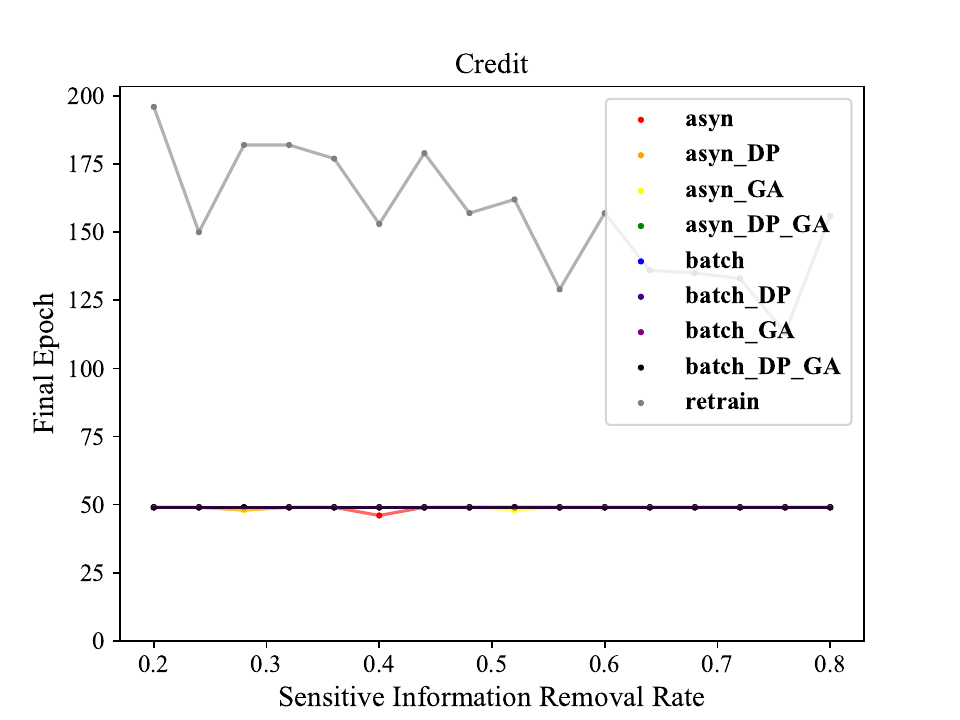}
        % \caption{子图6}
    \end{minipage}
    \hfill
    \begin{minipage}[b]{0.23\textwidth}
        \centering
        \includegraphics[width=\textwidth]{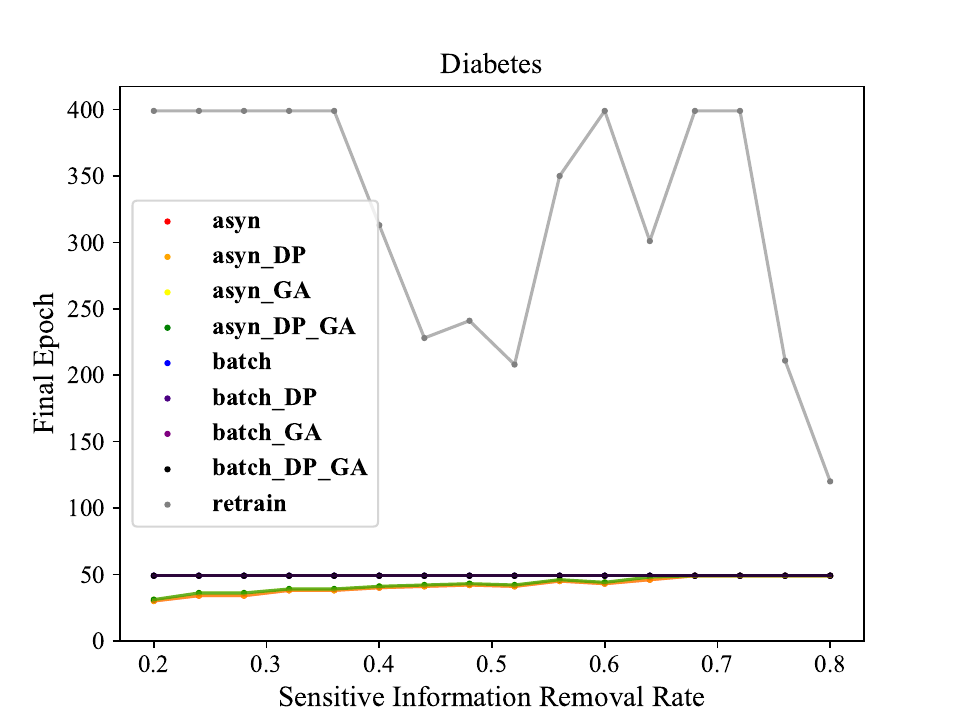}
        % \caption{子图6}
    \end{minipage}
    \hfill
    \begin{minipage}[b]{0.23\textwidth}
        \centering
        \includegraphics[width=\textwidth]{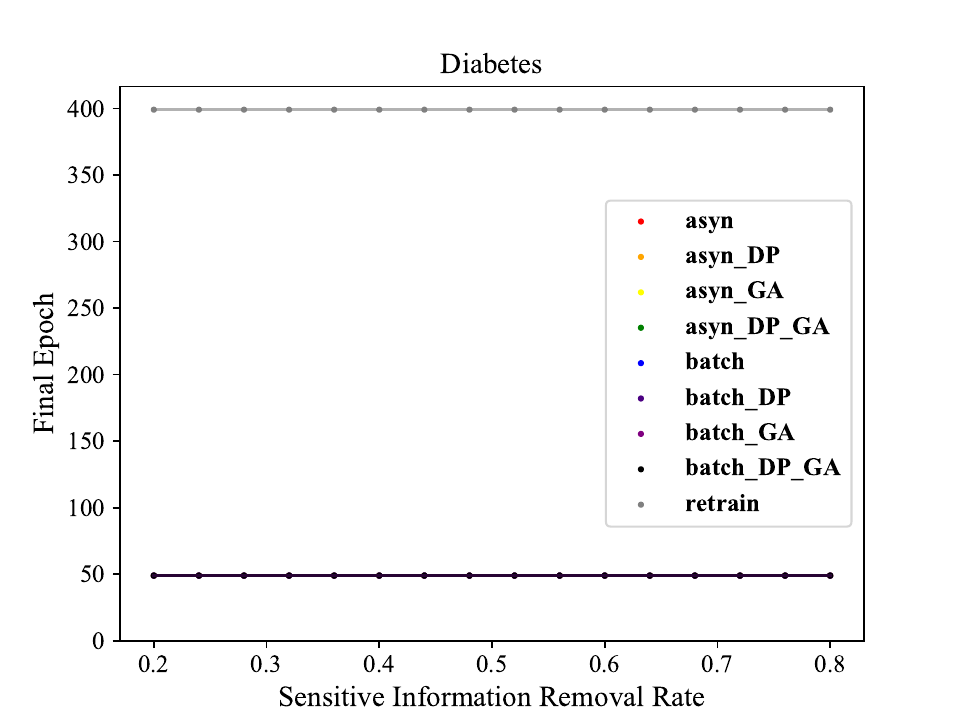}
        % \caption{子图6}
    \end{minipage}

    \vspace{0.5cm}  % 行间距

    % 第七列
    \begin{minipage}[b]{0.23\textwidth}
        \centering
        \includegraphics[width=\textwidth]{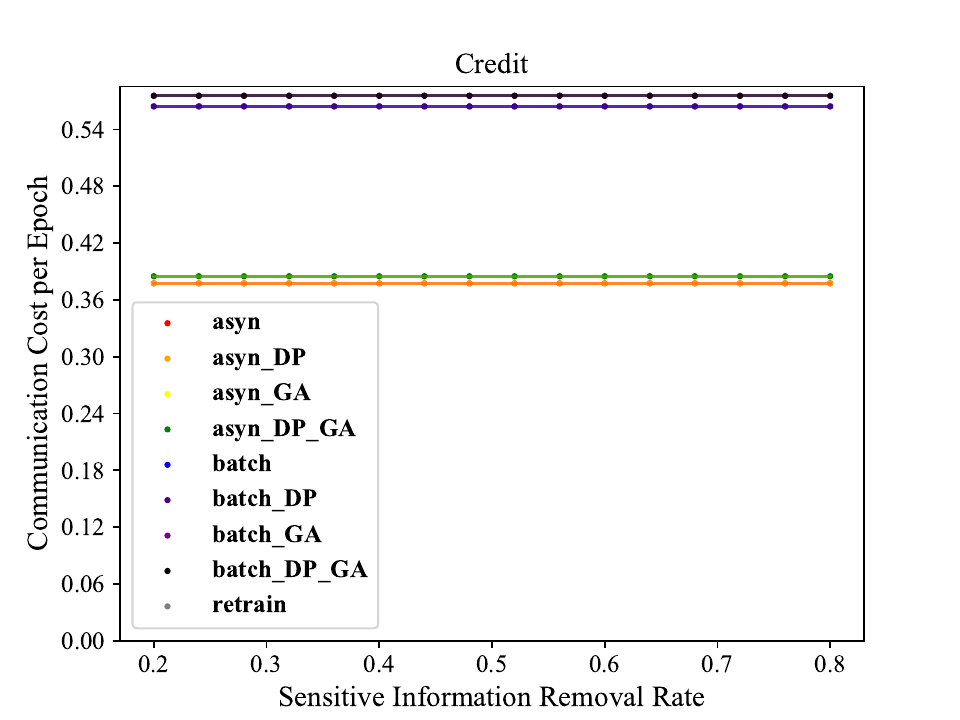}
        % \caption{a7}
    \end{minipage}
    \hfill
    \begin{minipage}[b]{0.23\textwidth}
        \centering
        \includegraphics[width=\textwidth]{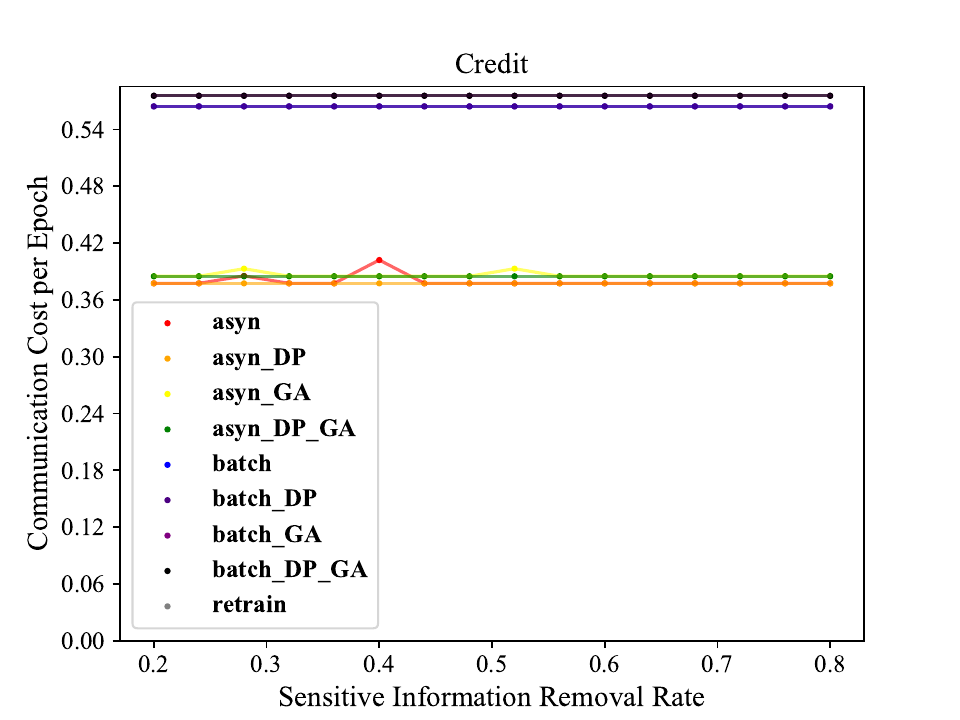}
        % \caption{子图7}
    \end{minipage}
    \hfill
    \begin{minipage}[b]{0.23\textwidth}
        \centering
        \includegraphics[width=\textwidth]{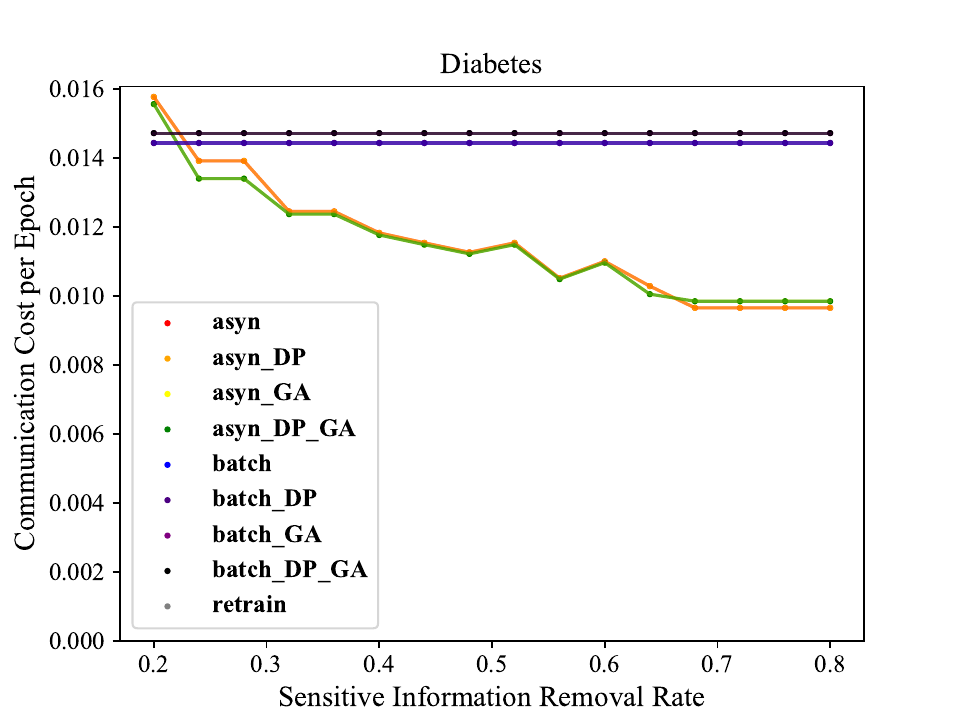}
        % \caption{子图7}
    \end{minipage}
    \hfill
    \begin{minipage}[b]{0.23\textwidth}
        \centering
        \includegraphics[width=\textwidth]{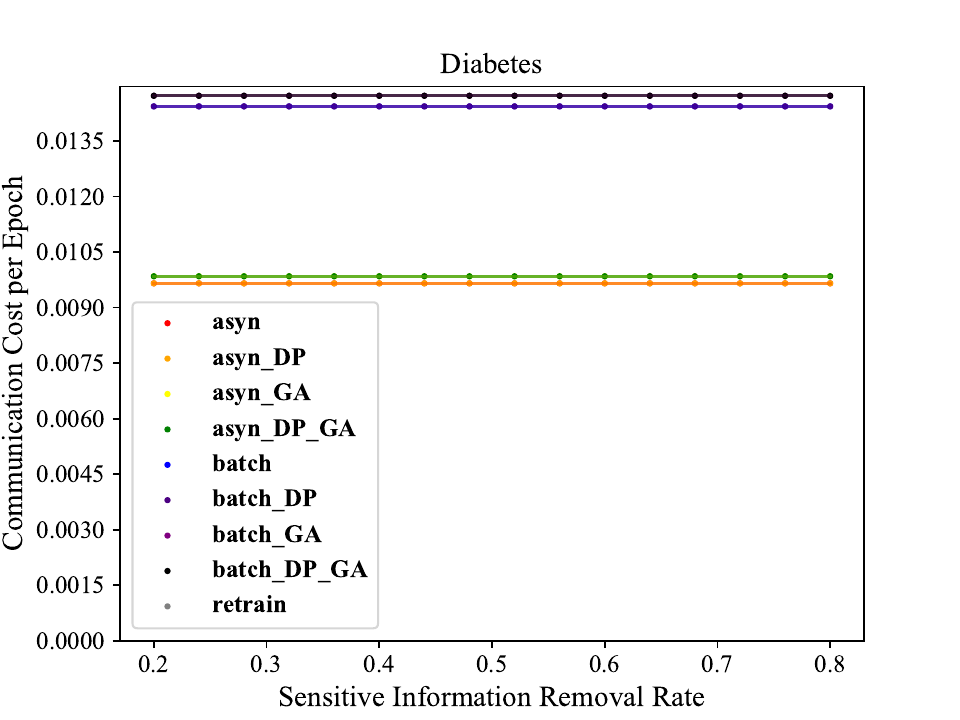}
        % \caption{子图7}
    \end{minipage}
    
    \caption{Sensitive Information Removal Results.}
    \label{Sensitive Information Removal Results.}
\end{figure*}

% \color{red}
% \subsection{Ablation Study}

% To achieve certified unlearning, our method requires two components: (1) adding noise during the training of the original model, and (2) performing gradient ascent on the original data during the first update. We conduct an ablation study to evaluate the impact of each component, leading to the following method variants:
% %The experiments are divided into four groups based on whether these two conditions are met:
% \begin{itemize}
%     \item \textbf{Ours}: both (1) and (2) are satisfied.
%     \item \textbf{Ours without noise (Ours w.o. noise)}: (1) is not satisfied, but (2) is.
%     \item \textbf{Ours without gradient ascending (Ours w.o. ascend)}: (1) is satisfied, but (2) is not.
%     \item \textbf{Ours without gradient ascending without noise (Ours w.o. ascend\&noise)}: neither (1) nor (2) is satisfied.
% \end{itemize}

% Here, we compare these variants of our method to show .... 

% \color{black}

\subsection{Unlearning Epoch Number Choice}
\label{sec:choice}

Our method performs unlearning through multiple training rounds. This raises the question: how many rounds of unlearning are required? Typically, in model training, an early stopping mechanism is used to prevent overfitting and determine the number of training epochs. In a VFL system, the test set data is generally not accessible during unlearning. Therefore, we use the training loss as the criterion for early stopping. While the training loss typically decreases during regular training, our asynchronous method approximates updates for clients that are offline. Thus, using the training loss for early stopping is still reasonable.

On the other hand, if too many rounds of unlearning are performed, the efficiency advantage over retraining is lost. Therefore, we cannot allow an excessive number of rounds. We need to determine a maximum number of rounds that is fewer than the retraining epochs, but still ensures sufficient training. We demonstrate this using the MLP model on the Malware dataset in a client removal experiment. When all clients are online, unlearning could potentially require hundreds of epochs. By recording the training loss, training accuracy, and gradient residual during the training process, we observe that the loss, accuracy, and gradient residual decrease quickly in the early stages of training and stabilize in the later stages, as shown in Figure ~\ref{epoch_choice}. Therefore, we select 50 epochs as the maximum number of rounds for unlearning. In the figure, we also plot the retraining curve for comparison, and it is evident that unlearning is more efficient than retraining.

\begin{figure*}[h]  
    
    \begin{minipage}[b]{0.3\textwidth}
        \centering
        \includegraphics[width=\textwidth]{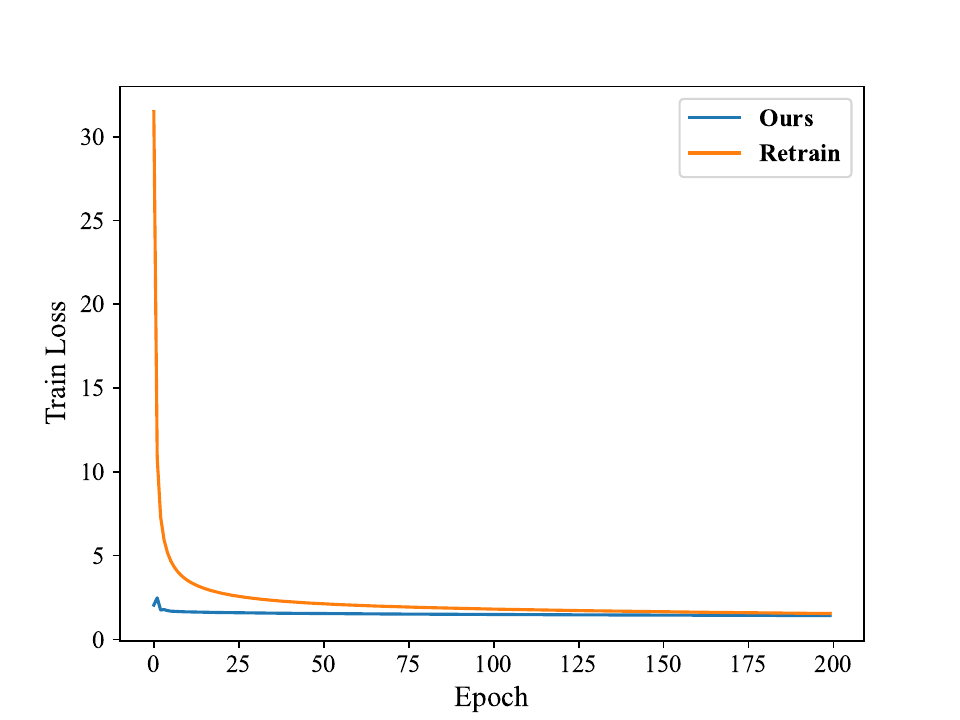} % 替换为您的图片路径
        % \caption{a1}
    \end{minipage}
    \hfill
    \begin{minipage}[b]{0.3\textwidth}
        \centering
        \includegraphics[width=\textwidth]{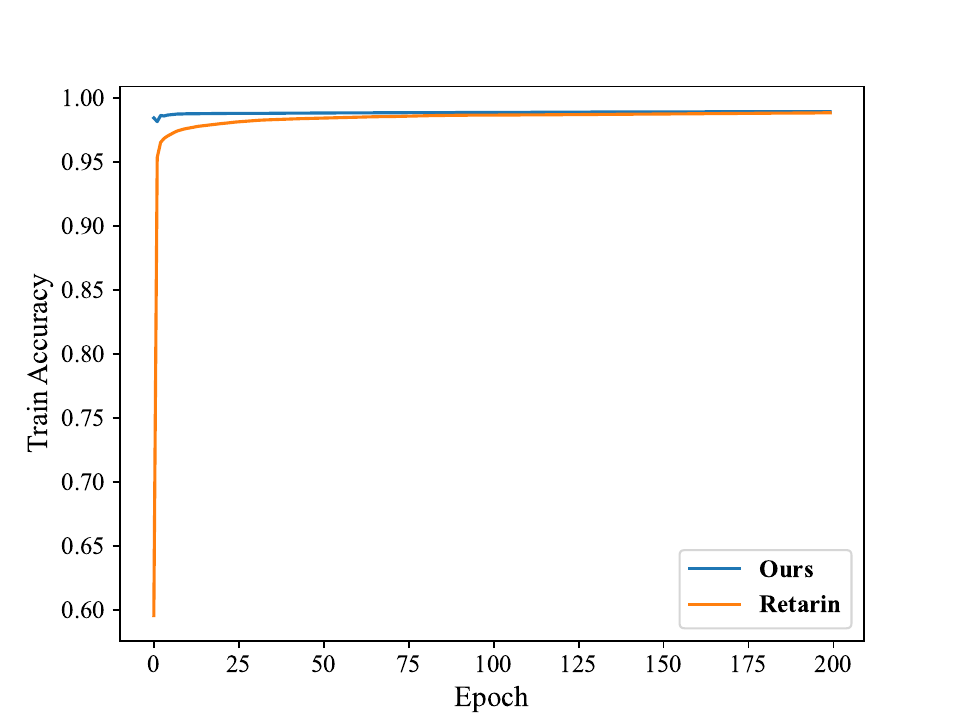} % 替换为您的图片路径
        % \caption{1}
    \end{minipage}
    \hfill
     \begin{minipage}[b]{0.3\textwidth}
        \centering
        \includegraphics[width=\textwidth]{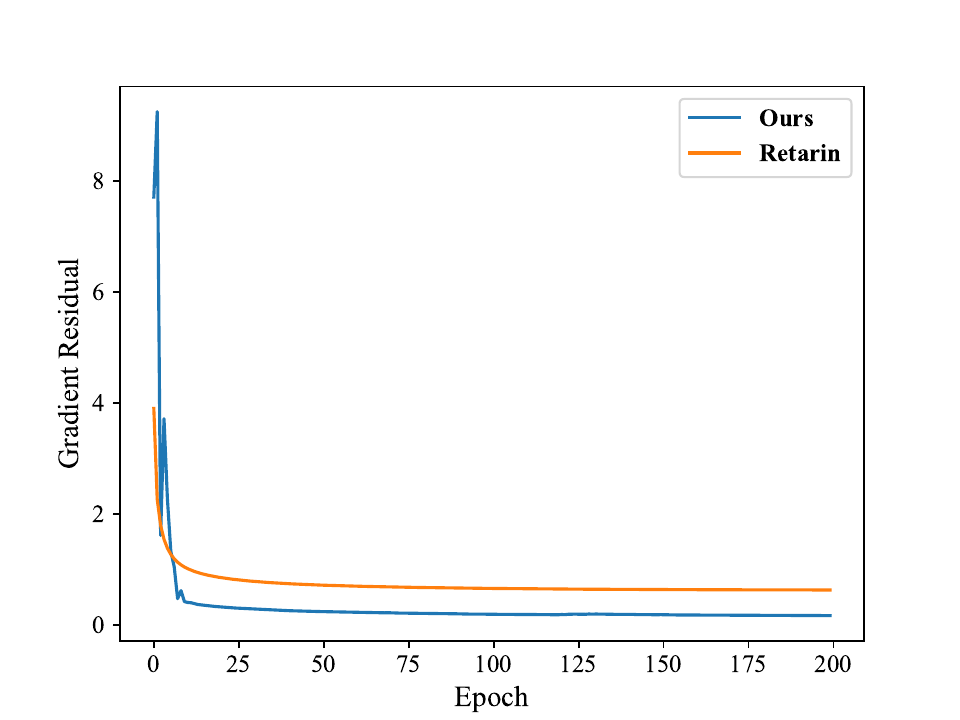} % 替换为您的图片路径
        % \caption{1}
    \end{minipage}

    \caption{The variation curves of training loss, training accuracy, and gradient residuals as updates progress.}
    \label{epoch_choice}
\end{figure*}

\subsection{Asynchronous Unlearning Online Rate}

% Our asynchronous method supports unlearning updates even when some clients are offline. One question that arises is whether the proportion of online clients affects the performance of our asynchronous method. To investigate this, we conducted experiments on client removal using the LR model on the Adult Income dataset. The Adult Income dataset contains 108 features, and we set up 16 clients, with 27 features assigned to the client selected for removal (Client 0), and the remaining 15 clients each possessing 5 or 6 features. In our system, there are always at least two clients online: the active party (Client 15) and the client requesting unlearning (Client 0), with a maximum of 16 clients online simultaneously.

Our asynchronous unlearning method enables updates even when a subset of clients is offline. A critical question is how the proportion of online clients affects performance. To address this, we evaluate client removal using the LR model on the Adult Income dataset. The dataset contains 108 features distributed among 16 clients: the target client (Client 0) holds 27 features, while the remaining 15 clients each possess 5 or 6 features. In our setup, the active party (Client 15) and the requesting client (Client 0) are always online, with the system supporting up to 16 concurrent clients.

% Since we cannot obtain the gradients of offline clients, gradient residuals cannot be computed. Additionally, when only a subset of clients are online, we are unable to compute metrics on the test set. Therefore, we use training set loss and accuracy as the evaluation metrics. In Table ~\ref{tab:online rate}, we tested the impact of different numbers of online clients after performing 50 rounds of unlearning. As observed, the number of online clients has minimal impact on the effectiveness of unlearning. Notably, when only two clients are online, the batch method achieves nearly the same unlearning performance as the fully online system, with only one-eighth of the system's total online time and communication cost for unlearning. This is a significant advantage.

We measure the fidelity of the unlearned model using test accuracy and AUC. In Table ~\ref{tab:online rate}, we tested the impact of different numbers of online clients after performing 50 rounds of unlearning. The results indicate that the number of online clients has minimal effect on fidelity. While model fidelity slightly improves with higher online rates, the asynchronous method achieves nearly identical performance to the fully online system when approximately 3/4 of clients are online.  This reduces the total online time and communication costs by 25\%, offering a substantial efficiency gain.

\begin{table}[t]
    \small
    \begin{center}
        \caption{The number of online clients has little impact on the effectiveness of unlearning.}
        \label{tab:online rate}
        \resizebox{\columnwidth}{!}{%
        \begin{tabular}{lcc}
            \toprule
            \textbf{Online Client Number} & \textbf{Test Accuracy}& \textbf{Test AUC}\\
            \midrule
            3& 0.813& 0.892\\
 4& 0.817&0.893\\
            6& 0.828& 0.894\\
            9& 0.841& 0.895\\
            12& 0.847& 0.896\\
            16 & 0.846& 0.896\\
            \bottomrule
        \end{tabular}
        }
    \end{center}
\end{table}

%-------------------------------------------------------------------------------
\section{Conclusion}
%-------------------------------------------------------------------------------

In VFU, multiple unlearning requests targeting different objectives may arise, which the current methods cannot handle in a compatible manner. Additionally, the requirement for all clients to participate in unlearning imposes a significant burden on the clients. To address the practical challenges in VFU tasks, we propose a new unlearning framework that can compatibly handle different unlearning objectives. For learning models with strongly convex loss functions, our method provides certified unlearning with theoretical guarantees. Furthermore, considering the characteristics of VFU, we introduce the first asynchronous VFU system capable of unlearning. This system balances the needs of different clients while reducing computation and communication costs. Extensive experiments validate the applicability of our method. 

However, our method still has limitations: the theoretical analysis of certified unlearning currently applies only to models with strongly convex loss functions. Additionally, our approach is currently focused on VFL systems using the AggVFL structure, and further research is needed to extend it to more complex split-NN-based VFL systems. These challenges require further exploration, and we hope to continue expanding the practical applications of VFU, achieving comprehensive privacy protection both during and after training.

% \clearpage

% \section*{\textbf{Ethics}}

% \begin{itemize}
%     \item "I attest that I read the ethics considerations discussions in the conference call for papers, the detailed submissions instructions, and the guidelines for ethics document."
%     \item "I attest that the research team considered the ethics of this research, that the authors believe the research was done ethically, and that the team's next-step plans (e.g., after publication) are ethical."
%     \item "I attest that the submission has a clearly-marked section on ethical considerations in the body of the paper and/or in the extra 'ethics considerations and compliance with the open science policy' page."
% \end{itemize}

% \section*{\textbf{Open Science Policy}}
% We use public datasets, we make all our implementations publicly available at https://anonymous.4open.science/r/vertical-federated-unlearning-DD85.

% \clearpage

%-------------------------------------------------------------------------------
% \bibliographystyle{plain}
\bibliography{usenix}

%-------------------------------------------------------------------------------
\section*{Appedix A}
%-------------------------------------------------------------------------------
Previous work has shown that in a VFL system using logistic regression, the update contribution coefficients of each client remain fixed during each update. We now extend this result to Multi-Layer Perceptron (MLP). While the contribution coefficients do not remain constant, we demonstrate that, under an appropriate setting for the unlearning tasks, the change in coefficients is minimal when the MLP's depth and width are constrained.

\subsection*{Mathematical Notation}\label{sec:converg}

Consider a classification problem with \( K \) labels and a cross-entropy loss function. 
The input vector is denoted as \( \mathbf{x} \in \mathbb{R}^n \), and the output 
\( \mathbf{y} \in \mathbb{R}^K \) is a one-hot encoded vector, where all components 
are 0 except for a single entry \( y_k = 1 \), which corresponds to the correct 
class label \( k \) for the input \( \mathbf{x} \). We use a fully connected neural network \( f_{\boldsymbol{\theta}}(\cdot) \) with 
\( L \) layers, applying the ReLU activation function after each linear transformation. 
The ReLU function, defined as 
\[
\sigma(\mathbf{x}) = [\mathbf{x} \geq \mathbf{0}] \mathbf{x},
\]
outputs the element-wise maximum of each input component and zero. The network's output is a vector of logits \( \mathbf{h} \in \mathbb{R}^K \), computed recursively as follows:

\begin{align*}\label{align:fc}
    &\mathbf{h}^{(p)} = \mathbf{W}^{(p)} \mathbf{x}^{(p)} + \mathbf{b}^{(p)}, \\
    &\mathbf{x}^{(p+1)} = \sigma(\mathbf{h}^{(p)}).
\end{align*}

Let the input and output of the \( p \)-th layer be denoted by \( \mathbf{x}^{(p)} \) and \( \mathbf{h}^{(p)} \), respectively.  We set \( \mathbf{x}^{(1)} = \mathbf{x} \) and \( \mathbf{h} = f_{\boldsymbol{\theta}}(\mathbf{x}) = \mathbf{h}^{(L)} \). The parameters of the network are represented as 
\[
\boldsymbol{\theta} = \mathrm{col}(\mathbf{w}^{(1)}, \mathbf{b}^{(1)}, \ldots, \mathbf{w}^{(L)}, \mathbf{b}^{(L)}) \in \mathbb{R}^{P},
\]
where \( \mathbf{w}^{(p)} \) is the flattened version of the weight matrix \( \mathbf{W}^{(p)} \) for the \( p \)-th layer, and \( \mathbf{b}^{(p)} \) is the corresponding bias vector. The output confidence vector \( \mathbf{p} \in \mathbb{R}^{K} \) is defined as the softmax of the logits \( \mathbf{h} \), i.e.,

\[ p_i = \mathrm{softmax}(\mathbf{h})_i = \dfrac{\exp{(h_i)}}{\sum_{j=1}^{K} \exp{(h_j)}} \in (0; 1). \] 
The loss function is cross-entropy loss: 

\[ \ell(\mathbf{h}, \mathbf{y}) = \mathrm{CE}(\mathbf{p}, \mathbf{y}) = - \sum_{k=1}^{K} y_k \log p_k \in \mathbb{R}^{+}. \]

Here the following is denoted:
\begin{itemize}
    \item Matrix representation of the ReLU activation function:
    \[ \mathbf{D}^{(p)} = \mathrm{diag}([\mathbf{h}^{(p)} \geqslant \mathbf{0}]), \]
    \item The partial derivative of logits w.r.t. logits at $p$-th layer:
    \[ \mathbf{G}^{(p)} = \dfrac{\partial \mathbf{h}}{\partial \mathbf{h}^{(p)}} = \mathbf{W}^{(L)} \mathbf{D}^{(L-1)} \mathbf{W}^{(L-1)} \mathbf{D}^{(L-2)} \cdot \ldots \cdot \mathbf{D}^{(p)}, \]
    \item Its stacked version:
    \[ \mathbf{F}\T = \begin{pmatrix}
        (\mathbf{G}^{(1)})\T \otimes \mathbf{x}^{(1)} \\
        (\mathbf{G}^{(1)})\T \\ 
        \vdots \\
        (\mathbf{G}^{(L)})\T \otimes \mathbf{x}^{(L)} \\
        (\mathbf{G}^{(L)})\T \\ 
    \end{pmatrix}, \]

\end{itemize}

Consider the next epoch, in which all the parameters in the MLP are updated by
\begin{align*}
    \bar{\mathbf{W}}^{(p)} &= \mathbf{W}^{(p)} + \tau \Delta \mathbf{W}^{(p)},\\
    \bar{\mathbf{b}}^{(p)} &= \mathbf{b}^{(p)} + \tau \Delta \mathbf{b}^{(p)},\quad p=1,\ldots, L.
\end{align*}
Then conduct back propagation this time, we would have
\begin{align*}
    &\bar{\mathbf{h}}^{(p)} = \bar{\mathbf{W}}^{(p)} \bar{\mathbf{x}}^{(p)} + \bar{\mathbf{b}}^{(p)}, \\
    &\bar{\mathbf{x}}^{(p+1)} = \sigma(\bar{\mathbf{h}}^{(p)}),\\
    &\bar{\mathbf{D}}^{(p)} = \mathrm{diag}([\bar{\mathbf{h}}^{(p)} \geqslant \mathbf{0}]),\\
    &\bar{\mathbf{G}}^{(p)} = \bar{\mathbf{W}}^{(L)} \bar{\mathbf{D}}^{(L-1)} \bar{\mathbf{W}}^{(L-1)} \bar{\mathbf{D}}^{(L-2)} \cdot \ldots \cdot \bar{\mathbf{D}}^{(p)}.
\end{align*}
Equipped with these notations,
\begin{equation*}
    \bar{\mathbf{F}}\T = \begin{pmatrix}
        (\bar{\mathbf{G}}^{(1)})\T \otimes \bar{\mathbf{x}}^{(1)} \\
        (\bar{\mathbf{G}}^{(1)})\T \\ 
        \vdots \\
        (\bar{\mathbf{G}}^{(L)})\T \otimes \bar{\mathbf{x}}^{(L)} \\
        (\bar{\mathbf{G}}^{(L)})\T \\ 
    \end{pmatrix}.
\end{equation*}
Our goal is to bound
\begin{equation*}
    \frac{\Vert\bar{\mathbf{F}}\Vert_F^2}{\Vert \mathbf{F}\Vert_F^2}.
\end{equation*}

\subsection*{Derivation}

Firstly, we need to add an assumption:

For any epoch t, the distance between parament $\theta_t$ and optimum parament $\theta^*$ is bounded by a small value $g$, an assumption widely used in NTK,

\begin{equation*}
    {\Vert \mathbf{\theta_t - \theta^*}\Vert_2} < g.
\end{equation*}

In machine unlearning, the data to be unlearned constitutes only a small portion of the original training set. As a result, the difference between model parament before and after unlearning is minimal, validating the assumption.

P.S. Our experiments show that the gradient of the original model on the training set after unlearning is negligible, further confirming the validity of this assumption.

Since we assume that the distance between \(\mathbf{F}_t\) and \(\mathbf{F}^*\) is small, meaning that \(\mathbf{F}_t\) is gradually approaching \(\mathbf{F}^*\) during the iteration process, then the two quantities can be related as follows: 
   \[
   \Vert \mathbf{F}_t \Vert_F^2 = \Vert \mathbf{F}^* \Vert_F^2 + 2 \langle \mathbf{F}_t - \mathbf{F}^*, \mathbf{F}^* \rangle + \Vert \mathbf{F}_t - \mathbf{F}^* \Vert_F^2.
   \]

Thus,
   \[
   \frac{\Vert \mathbf{F}_t \Vert_F^2}{\Vert \mathbf{F}^* \Vert_F^2} = 1 + \frac{2 \langle \mathbf{F}_t - \mathbf{F}^*, \mathbf{F}^* \rangle}{\Vert \mathbf{F}^* \Vert_F^2} + \frac{\Vert \mathbf{F}_t - \mathbf{F}^* \Vert_F^2}{\Vert \mathbf{F}^* \Vert_F^2}.
   \]

When \(\mathbf{F}_t\) is close to \(\mathbf{F}^*\), \(\langle \mathbf{F}_t - \mathbf{F}^*, \mathbf{F}^* \rangle\) is relatively small, Thus \(\frac{\Vert \mathbf{F}_t \Vert_F^2}{\Vert \mathbf{F}^* \Vert_F^2}\) equals to \(1 + \frac{\Vert \mathbf{F}_t - \mathbf{F}^* \Vert_F^2}{\Vert \mathbf{F}^* \Vert_F^2}\).

Our goal is to bound

\begin{equation*}
    \Vert \mathbf{F_t - F^*}\Vert_F^2,
\end{equation*}

note that since loss function is not considered
    
\begin{equation*}
    \Vert \mathbf{F^*}\Vert_F^2 \neq 0,
\end{equation*}

\[( \mathbf{F_t - F^*})\T = \begin{pmatrix}
        (\mathbf{G}^{(1)})\T \otimes \mathbf{x}^{(1)} - (\mathbf{G}^{(1)*})\T \otimes \mathbf{x}^{(1)*} \\
        (\mathbf{G}^{(1)} - \mathbf{G}^{(1)*})\T \\ 
        \vdots \\
        (\mathbf{G}^{(L)})\T \otimes \mathbf{x}^{(L)} - (\mathbf{G}^{(L)*})\T \otimes \mathbf{x}^{(L)*} \\
        (\mathbf{G}^{(L)} - \mathbf{G}^{(L)*})\T \\ 
    \end{pmatrix}. \]

Then, using the property that squared spectral norm of vertically-stacked matrices is less or equal to the sum of their squared spectral norms (it is easy to observe), we get:
\begin{align*} 
&\quad\| \mathbf{F_t - F^*} \|_2^2 \\
&\leqslant \sum\limits_{p=1}^{L} \Big( 
\| (\mathbf{G}^{(p)})^\top \otimes \mathbf{x}^{(p)} 
  - (\mathbf{G}^{(p)*})^\top \otimes \mathbf{x}^{(p)*} \|_2^2 \\
&\quad + \| (\mathbf{G}^{(p)} - \mathbf{G}^{(p)*})^\top \|_2^2 
\Big).
\end{align*}

The Kronecker product satisfies the distributive property over matrix addition,

\begin{align*}
     \| \mathbf{F_t - F^*} \|_2^2 
     &\leqslant \sum\limits_{p=1}^{L} \Big( 
     \| (\mathbf{G}^{(p)})^\top \otimes \mathbf{x}^{(p)} 
     - (\mathbf{G}^{(p)*})^\top \otimes \mathbf{x}^{(p)} \\
     &\quad + (\mathbf{G}^{(p)*})^\top \otimes \mathbf{x}^{(p)} 
     - (\mathbf{G}^{(p)*})^\top \otimes \mathbf{x}^{(p)*} \|_2^2 \\
     &\quad + \| (\mathbf{G}^{(p)} - \mathbf{G}^{(p)*})^\top \|_2^2 
     \Big) \\
     &\leqslant \sum\limits_{p=1}^{L} \Big( 
     \| (\mathbf{G}^{(p)} - \mathbf{G}^{(p)*})^\top \otimes \mathbf{x}^{(p)} \\
     &\quad + (\mathbf{G}^{(p)*})^\top \otimes (\mathbf{x}^{(p)} - \mathbf{x}^{(p)*}) \|_2^2 \\
     &\quad + \| (\mathbf{G}^{(p)} - \mathbf{G}^{(p)*})^\top \|_2^2 
     \Big).
\end{align*}

By the triangle inequality and the sub-additivity of the spectral norm,

\begin{align*}
     \| \mathbf{F_t - F^*} \|_2^2 
     &\leqslant \sum\limits_{p=1}^{L} \Big( 
     \| (\mathbf{G}^{(p)} - \mathbf{G}^{(p)*})^\top \otimes \mathbf{x}^{(p)} \|_2^2 \\
     &\quad + \| (\mathbf{G}^{(p)*})^\top \otimes (\mathbf{x}^{(p)} - \mathbf{x}^{(p)*}) \|_2^2 \\
     &\quad + \| (\mathbf{G}^{(p)} - \mathbf{G}^{(p)*})^\top \|_2^2 
     \Big).
\end{align*}

Spectral norm of the Kronecker matrix product is equal to their ordinary product norm, 

\begin{align*}
\| \mathbf{F_t - F^*} \|_2^2 
&\leqslant \sum\limits_{p=1}^{L} \Bigg( 
\| \mathbf{G}^{(p)} - \mathbf{G}^{(p)*} \|_2^2 \cdot \left( \| \mathbf{x}^{(p)} \|_2^2 + 1 \right) \\
&\quad + \| (\mathbf{G}^{(p)*})^\top \|_2^2 \cdot \| (\mathbf{x}^{(p)} - \mathbf{x}^{(p)*}) \|_2^2
\Bigg).
\end{align*}

In this problem, we assume that the distance between the network parameters \(\mathbf{\theta_t}\) and the optimal parameters \(\mathbf{\theta^*}\) is bounded by a small value \(g\), and the goal now is to derive an upper bound for the difference in the intermediate values \(\| \mathbf{x}^{(p)} - \mathbf{x}^{(p)*} \|_2^2\)  and $\| \mathbf{G}^{(p)} - \mathbf{G}^{(p)*} \|_2^2 $ for each layer of the MLP.

Bound on \(\| \mathbf{x}^{(p)} - \mathbf{x}^{(p)*} \|_2^2\):

The correct expression for the difference in activations at layer \( p \) is:

\[
\mathbf{x}^{(p)} - \mathbf{x}^{(p)*} = \sigma\left( \mathbf{W}^{(p)}  \mathbf{x}^{(p-1)} -  \mathbf{W}^{(p)*}\mathbf{x}^{(p-1)*}  + \mathbf{b}^{(p)} - \mathbf{b}^{(p)*} \right),
\]

We define the activation function as being Lipschitz continuous with a constant $\gamma_z$. For ReLU function we used, $\gamma_z = 1.$ The difference in activations can be bounded as:

\begin{align*}
    \| \mathbf{x}^{(p)} - \mathbf{x}^{(p)*} \|_2^2 
    &\le 2\gamma_z^2\left[\|\mathbf{W}^{(p)}  \mathbf{x}^{(p-1)} -  \mathbf{W}^{(p)*}\mathbf{x}^{(p-1)*}\|^2_2 \right. \\
    &\quad \left. + \|\mathbf{b}^{(p)} - \mathbf{b}^{(p)*}\|_2^2\right] \\
    &\le 4\gamma_z^2\left[\|\mathbf{W}^{(p)}   -  \mathbf{W}^{(p)*}\|^2_2\|\mathbf{x}^{(p-1)*}\|_2^2 \right. \\
    &\quad + \|\mathbf{W}^{(p)}\|_2^2\|\mathbf{x}^{(p-1)}-\mathbf{x}^{(p-1)*}\|^2_2 \\
    &\quad \left. + \|\mathbf{b}^{(p)} - \mathbf{b}^{(p)*}\|_2^2\right].
\end{align*}

Assuming that the $\|\mathbf{x}^{(t)}\|_2^2$ is bounded by $C_x^2$, and $\|\mathbf{W}^{(p)}\|_2^2$ is bounded by $C_W^2$, we can use the fact that the perturbation in the parameters is bounded by \( g \), and propagate this bound through the layers.  Since$\|\mathbf{x}^{(1)} - \mathbf{x}^{(1)*} \|_2^2 = 0$, the bound on the difference in activations at layer \( p \) becomes:

\begin{align*}
     \|\mathbf{x}^{(p)} - \mathbf{x}^{(p)*} \|_2^2 
     &\le 4\gamma_z^2\left[g^2 C_x^2+g^2+C_W^2\|\mathbf{x}^{(p-1)} - \mathbf{x}^{(p-1)*} \|_2^2\right] \\
     &\le 4\gamma_z^2g^2\left[ C_x^2+1\right] \left[\frac{(4\gamma_z^2C_W^2)^{p-1} - 1}{4\gamma_z^2C_W^2 - 1}\right] \\
     &\quad + (4\gamma_z^2C_W^2)^{p-1}\|\mathbf{x}^{(1)} - \mathbf{x}^{(1)*} \|_2^2 \\
     &= 4\gamma_z^2g^2\left[ C_x^2+1\right] \left[\frac{(4\gamma_z^2C_W^2)^{p-1} - 1}{4\gamma_z^2C_W^2 - 1}\right].
\end{align*}

Bound on $\| \mathbf{G}^{(p)} - \mathbf{G}^{(p)*} \|_2^2 $:

For the gradients, we consider the difference in gradients at layer \( p \):

\[
\mathbf{G}^{(p)} - \mathbf{G}^{(p)*} = \frac{\partial \mathbf{h}}{\partial \mathbf{h}^{(p)}}  -  \frac{\partial \mathbf{h^*}}{\partial \mathbf{h}^{(p)*}},
\]
where \( \frac{\partial \mathbf{h}}{\partial \mathbf{h}^{(p)}} = \mathbf{W}^{(L)} \mathbf{D}^{(L-1)} \mathbf{W}^{(L-1)} \mathbf{D}^{(L-2)} \cdots \mathbf{D}^{(p)} \) is the chain of derivatives through the layers.

\begin{equation*}
    \mathbf{G}^{(p)}=\mathbf{G}^{(p+1)}\mathbf{W}^{(p+1)}\mathbf{D}^{(p)},\quad \mathbf{G}^{(p)\ast}=\mathbf{G}^{(p+1)\ast}\mathbf{W}^{(p+1)\ast}\mathbf{D}^{(p)\ast},
\end{equation*}

\begin{align*}
&\quad \|\mathbf{G}^{(p)} - \mathbf{G}^{(p)*}\|^2 \\
&= \|\mathbf{G}^{(p+1)}\mathbf{W}^{(p+1)}\mathbf{D}^{(p)} - \mathbf{G}^{(p+1)\ast}\mathbf{W}^{(p+1)\ast}\mathbf{D}^{(p)\ast}\|^2 \\
&\le 2\|\mathbf{G}^{(p+1)}\|^2 \cdot \|\mathbf{W}^{(p+1)}\mathbf{D}^{(p)} - \mathbf{W}^{(p+1)\ast}\mathbf{D}^{(p)\ast}\|^2 \\
&\quad + 2\|\mathbf{W}^{(p+1)\ast}\mathbf{D}^{(p)\ast}\|^2 \cdot \|\mathbf{G}^{(p+1)} - \mathbf{G}^{(p+1)\ast}\|^2.
\end{align*}

Assuming that the $\|\mathbf{G}^{(p+1)}\|_2^2$ is bounded by $C_G^2$, and $\|\mathbf{W}^{(p+1)\ast}\mathbf{D}^{(p)\ast}\|_2^2$ is bounded by $C_W^2$, similarly:

\begin{align*}
&\quad\|\mathbf{G}^{(p)} - \mathbf{G}^{(p)*} \|^2\\
    &\le 2C_G^2\|\mathbf{W}^{(p+1)}\mathbf{D}^{(p)}-\mathbf{W}^{(p+1)\ast}\mathbf{D}^{(p)\ast}\|^2\\
    &\quad+2C_W^2\|\mathbf{G}^{(p+1)}-\mathbf{G}^{(p+1)\ast}\|^2\\
    &\le 2C_G^2g^2+2C_W^2\|\mathbf{G}^{(p+1)}-\mathbf{G}^{(p+1)\ast}\|^2\\
    &\le 2C_G^2g^2\frac{(2C_W^2)^{L-p} - 1}{2C_W^2 - 1}+(2C_W^2)^{L-p}\|\mathbf{G}^{(L)}-\mathbf{G}^{(L)\ast}\|^2\\
    &\le 2C_G^2g^2\frac{(2C_W^2)^{L-p} - 1}{2C_W^2 - 1}+(2C_W^2)^{L-p}g^2.
\end{align*}

The original inequality is given by:

\begin{align*}
\| \mathbf{F_t - F^*} \|_2^2 
    &\leq \sum_{p=1}^{L} \left( \| \mathbf{G}^{(p)} - \mathbf{G}^{(p)*} \|_2^2 \cdot (\| \mathbf{x}^{(p)} \|_2^2 + 1) \right. \\
    &\quad \left. + \|(\mathbf{G}^{(p)*})^\top \|_2^2 \cdot \|(\mathbf{x}^{(p)} - \mathbf{x}^{(p)*}) \|_2^2 \right).
\end{align*}

Substitute \( \| \mathbf{G}^{(p)} - \mathbf{G}^{(p)*} \|_2^2 \):
\[
\| \mathbf{G}^{(p)} - \mathbf{G}^{(p)*} \|^2 \leq 2C_G^2g^2 \frac{(2C_W^2)^{L-p} - 1}{2C_W^2 - 1} + (2C_W^2)^{L-p}g^2.
\]
\begin{align*}
    &\quad\| \mathbf{G}^{(p)} - \mathbf{G}^{(p)*} \|_2^2 \cdot (\| \mathbf{x}^{(p)} \|_2^2 + 1) \\
    &\leq [2C_G^2g^2 \frac{(2C_W^2)^{L-p} - 1}{2C_W^2 - 1} + (2C_W^2)^{L-p}g^2 ](\| \mathbf{x}^{(p)} \|_2^2 + 1).
\end{align*}

Substitute \( \| \mathbf{x}^{(p)} - \mathbf{x}^{(p)*} \|_2^2 \):
\[
\|\mathbf{x}^{(p)} - \mathbf{x}^{(p)*} \|_2^2 \leq 4\gamma_z^2g^2 \left[C_x^2 + 1\right] \frac{(4\gamma_z^2C_W^2)^{p-1} - 1}{4\gamma_z^2C_W^2 - 1}.
\]

\begin{align*}
    &\quad\|(\mathbf{G}^{(p)*})^\top\|_2^2 \cdot \|(\mathbf{x}^{(p)} - \mathbf{x}^{(p)*}) \|_2^2 \\
    &\leq \|(\mathbf{G}^{(p)*})^\top\|_2^2 \cdot 4\gamma_z^2g^2 \left[C_x^2 + 1\right] \cdot \frac{(4\gamma_z^2C_W^2)^{p-1} - 1}{4\gamma_z^2C_W^2 - 1}.
\end{align*}

We assume the following bounds:
\[
\|\mathbf{x}^{(p)}\|_2^2 \leq C_x^2, \quad \|(\mathbf{G}^{(p)*})^\top\|_2^2 \leq C_G^2.
\]

Substitute these bounds into the inequality:
\begin{align*}
    &\quad\| \mathbf{F_t - F^*} \|_2^2 \\
    &\leq  \sum_{p=1}^L \Bigg\{
    \left[2C_G^2g^2 \frac{(2C_W^2)^{L-p} - 1}{2C_W^2 - 1} + (2C_W^2)^{L-p}g^2\right](C_x^2 + 1) \\
    &\quad + 4C_G^2\gamma_z^2g^2 (C_x^2 + 1) \frac{(4\gamma_z^2C_W^2)^{p-1} - 1}{4\gamma_z^2C_W^2 - 1}
    \Bigg\} \\
    &\leq  \sum_{p=1}^L \Bigg\{
    \left[2C_G^2 \frac{(2C_W^2)^{L-p} - 1}{2C_W^2 - 1} + (2C_W^2)^{L-p} \right. \\
    &\quad \left. + 4C_G^2\gamma_z^2\frac{(4\gamma_z^2C_W^2)^{p-1} - 1}{4\gamma_z^2C_W^2 - 1}\right]g^2(C_x^2 + 1)
    \Bigg\}.
\end{align*}

since
\[
\| \mathbf{A} \|_F \leq \sqrt{\text{rank}(\mathbf{A})} \| \mathbf{A} \|_2,
\]

we get

\[
\| \mathbf{F_t - F^*} \|_F^2 \leqslant \sqrt{\text{rank}(\mathbf{F})} \| \mathbf{F_t - F^*} \|_2^2.
\]

Assuming that the layers of the MLP are constrained to have a width of less than $w$, 
\[
\text{rank}(\mathbf{F}) = max(w^2, 2L).
\]
$\text{rank}(\mathbf{F})$ is determined by the largest width or depth of the MLP. So,

\begin{align*}
    &\quad\| \mathbf{F_t - F^*} \|_F^2 \\
    &\leq  \sqrt{max(w^2, 2L)}\sum_{p=1}^L \Bigg\{
    \left[2C_G^2 \frac{(2C_W^2)^{L-p} - 1}{2C_W^2 - 1} + (2C_W^2)^{L-p} \right. \\
    &\quad \left. + 4C_G^2\gamma_z^2\frac{(4\gamma_z^2C_W^2)^{p-1} - 1}{4\gamma_z^2C_W^2 - 1}\right]g^2(C_x^2 + 1)
    \Bigg\}.
\end{align*}

When $g$ is close to 0, $\| \mathbf{F_t - F^*} \|_F^2$ is close to 0.

The analysis of the bound expression reveals that the stability of the update contribution coefficient is influenced by several factors, including the MLP model’s depth and width, as well as the absolute magnitudes of the parameters, activations, and gradients at each layer. Of these, the model's depth has the most significant impact.

\section*{Appedix B}
% %-------------------------------------------------------------------------------

We can also prove that in VFL, performing unlearning only for the client requesting unlearning is insufficient to achieve the overall unlearning objective.

\subsection*{Objective}

We aim to prove that when only a subset of clients participates in the unlearning process, the unlearning operation cannot achieve approximate unlearning. To achieve approximate unlearning, we need to ensure that the parameters after unlearning are very close to the parameters of a retrained model. This can be done by controlling the gradient residual.

\subsection*{Composition of Prediction $P$}

In vertical federated learning, the prediction $P$ is calculated based on the intermediate values of all client output.

\[
P = Global(\sum_{i=1}^{k}h_i(\theta_i, \points_i)).
\]

where:
- $h_i(\theta_i, \points_i)$ is the intermediate output of the $i$-th client.
- $k$ is the number of clients participating in federated learning.
- $Global()$ is the global model located in the activate party.

\subsection*{Gradient Representation}

To derive the requirements for approximate unlearning, we first examine the gradient of the logistic regression model's loss function with respect to the model parameters. Using the chain rule, the gradient can be written as the product of the following three terms:

\[
\nabla_{\theta} \mathcal{L} = \frac{\partial \mathcal{L}}{\partial P} \cdot \frac{\partial P}{\partial h} \cdot \frac{\partial h}{\partial \theta}.
\]

where:
\begin{itemize}
    \item $\frac{\partial \mathcal{L}}{\partial P}$: The derivative of the loss function (e.g., cross-entropy loss) with respect to the prediction $P$.
    \item $\frac{\partial P}{\partial h}$: The derivative of the prediction $P$ with respect to each client's intermediate value $h$.
    \item $\frac{\partial h}{\partial \theta}$: The derivative of each client's intermediate value $h$ with respect to the client's parameters $\theta$.
\end{itemize}

Our goal is to make the gradient $\nabla_{\theta} \mathcal{L}$ close to zero, which would achieve approximate unlearning. The main factor affecting the gradient is the first term, $\frac{\partial \mathcal{L}}{\partial P}$, i.e., the derivative of the loss function with respect to the prediction $P$. This term determines the size of the final gradient. If the prediction after unlearning is close to the retrained prediction, this term will be close to zero, which would bring the gradient close to zero. So, to achieve the approximate unlearning, the difference between the unlearned predictions and the retrained predictions should be small

\[
|\tilde{P} - \tilde{P}^*| = |Global(\sum_{i=1}^{k}h_i(\tilde{\theta}_i, \perts_i)) - Global(\sum_{i=1}^{k}h_i(\tilde{\theta}^*_i, \perts_i))|,
\]

where $\perts$ is the training set after the unlearning request,  $\tilde{\theta}_i$ is the unlearned parament of i-th client and $\tilde{\theta}^*_i$ is the retrained parament.

In AggVFL, the global model is not trainable, so the summed predictions from the clients should be close to that of the retrained model's output. To achieve the approximate unlearning, we want that there exists a small value $\epsilon > 0$ that

\[
|\sum_{i=1}^{k}h_i(\tilde{\theta}_i, \perts_i) - \sum_{i=1}^{k}h_i(\tilde{\theta}^*_i, \perts_i)| < \epsilon.
\]

Suppose overall there are k features, among them the first j features take part in the unlearning process 

\[
\tilde{\theta} = [\tilde{\theta}_1^T,...,\tilde{\theta}_j^T,\theta_{j+1}^{*T}, ..., \theta_{k}^{*T}]^T.
\]

The unchanged intermediate client output with features denoted as $\theta_{i}^{*T}$, are different from the retrained value with a distance $|h_U(\theta_i^\ast, \perts_i)|$ caused by unlearned feature.

\begin{align*}
&\quad|\sum_{i=1}^{k}h_i(\theta_i^\ast, \perts_i) - \sum_{i=1}^{k}h_i(\tilde{\theta}^*_i, \perts_i)| \\
&=|\sum_{i=1}^{k}h_i(\theta_i^\ast, \perts_i)-\sum_{i=1}^{k}h_i(\theta^*_i, \points_i)+\sum_{i=1}^{k}h_i(\theta^*_i, \points_i) - \sum_{i=1}^{k}h_i(\tilde{\theta}^*_i, \perts_i)|\\
&=|h_U(\theta^\ast, \perts)+\epsilon|.
\end{align*}

The approximate unlearning target below is hard to achieve when $j$ is small because $h$ has a much larger dimension, the unlearned features' gap can hardly be covered by $j$ features. 

\begin{align*}
&\quad|\sum_{i=1}^{k}h_i(\tilde{\theta}_i, \perts_i) - \sum_{i=1}^{k}h_i(\tilde{\theta}^*_i, \perts_i)| \\
&= |\sum_{i=1}^{j}h_i(\tilde{\theta}_i, \perts_i) - \sum_{i=1}^{j}h_i(\theta^*_i, \perts_i) +\sum_{i=1}^{k}h_i(\theta_i^\ast, \perts_i) - \sum_{i=1}^{k}h_i(\tilde{\theta}^*_i, \perts_i)|\\
& =|\sum_{i=1}^{j}h_i(\tilde{\theta}_i, \perts_i) - \sum_{i=1}^{j}h_i(\theta^*_i, \perts_i) +h_U(\theta^\ast, \perts)+\epsilon|.
\end{align*}

Besides, we can interpret the prediction difference in terms of vector spaces. Let \( A = \sum_{i=1}^{j} h_i(\tilde{\theta}_i, \perts_i) \) represent the modified outputs, and \( B = \sum_{i=j+1}^{k} h_i(\theta_i^*, \perts_i) \) represent the unchanged outputs. The difference between the modified and retrained outputs is:

\[
|A - B| = |h_U(\theta^*, \perts) + \epsilon|.
\]

When \( j \) is small, the unlearning operation can only change a small part of the output space. Since \( h_U(\theta^*, \perts) \) involves contributions from the unchanged clients' outputs, and their rank may be high, the modification of \( j \) clients' outputs is insufficient to reduce the overall error. Hence, the difference cannot be small enough to achieve approximate unlearning.

The approximate unlearning target is difficult to achieve when \( j \) is small because the change induced by unlearning only a few clients cannot cover the large-dimensional gap caused by the unlearned features. The rank of the output matrix will not be sufficiently reduced, and the residual error will remain significant, preventing the global prediction from matching the retrained model's output. Therefore, unlearning by a subset of clients is insufficient for achieving approximate unlearning.

% %-------------------------------------------------------------------------------
\section*{Appedix C}
% %-------------------------------------------------------------------------------

The results are shown in Figure ~\ref{feature extra} and ~\ref{info extra}.

\begin{figure*}[h]  % 使用 figure* 环境，使图片跨越两栏
    \centering

    % 第二列
    \begin{minipage}[b]{0.23\textwidth}
        \centering
        \includegraphics[width=\textwidth]{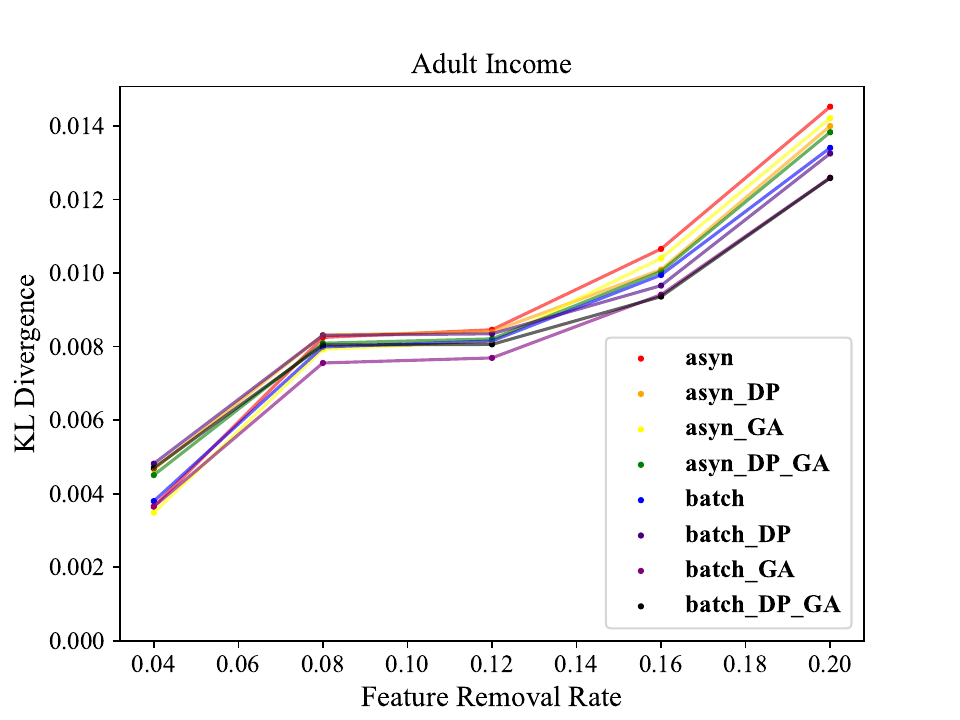}
        % \caption{a2}
    \end{minipage}
    \hfill
    \begin{minipage}[b]{0.23\textwidth}
        \centering
        \includegraphics[width=\textwidth]{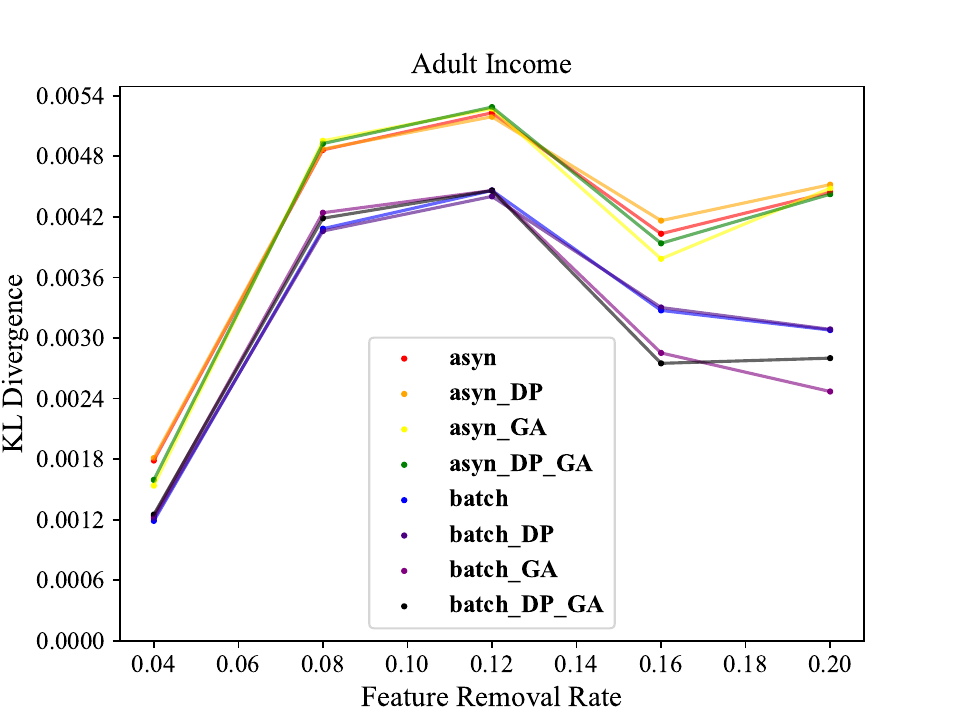}
        % \caption{子图2}
    \end{minipage}
    \hfill
    \begin{minipage}[b]{0.23\textwidth}
        \centering
        \includegraphics[width=\textwidth]{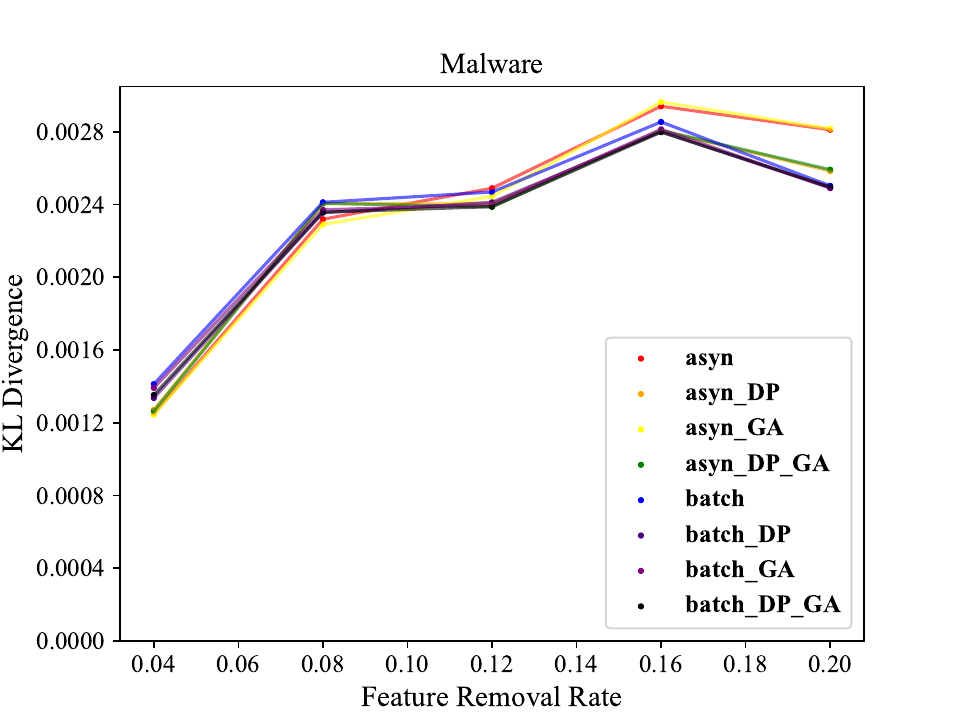}
        % \caption{子图2}
    \end{minipage}
    \hfill
    \begin{minipage}[b]{0.23\textwidth}
        \centering
        \includegraphics[width=\textwidth]{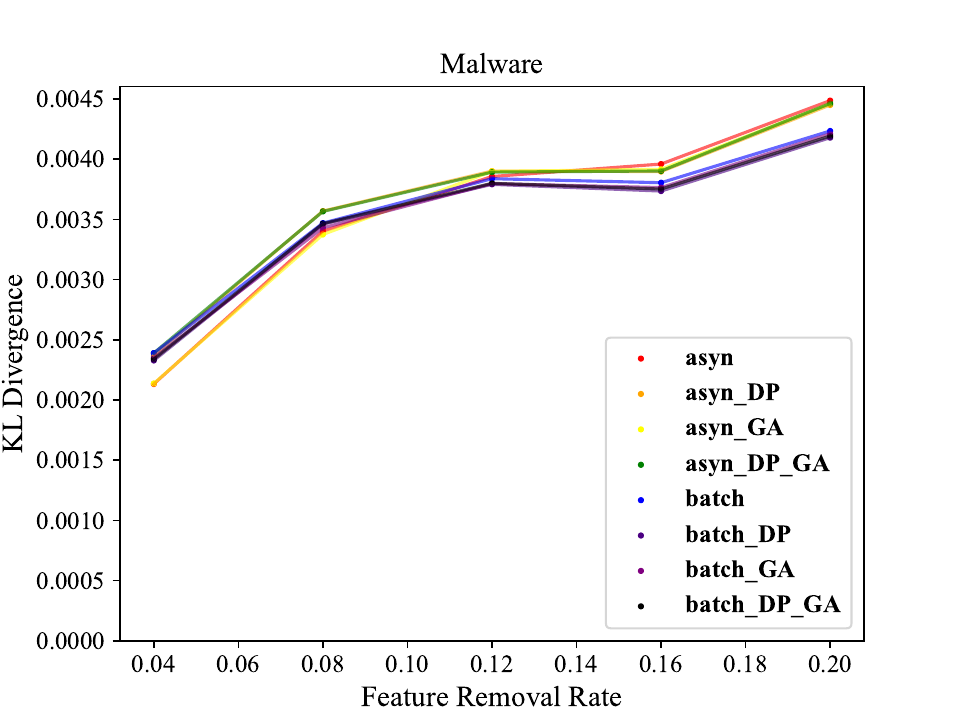}
        % \caption{子图2}
    \end{minipage}

    \vspace{0.5cm}  % 行间距

    % 第四列
    \begin{minipage}[b]{0.23\textwidth}
        \centering
        \includegraphics[width=\textwidth]{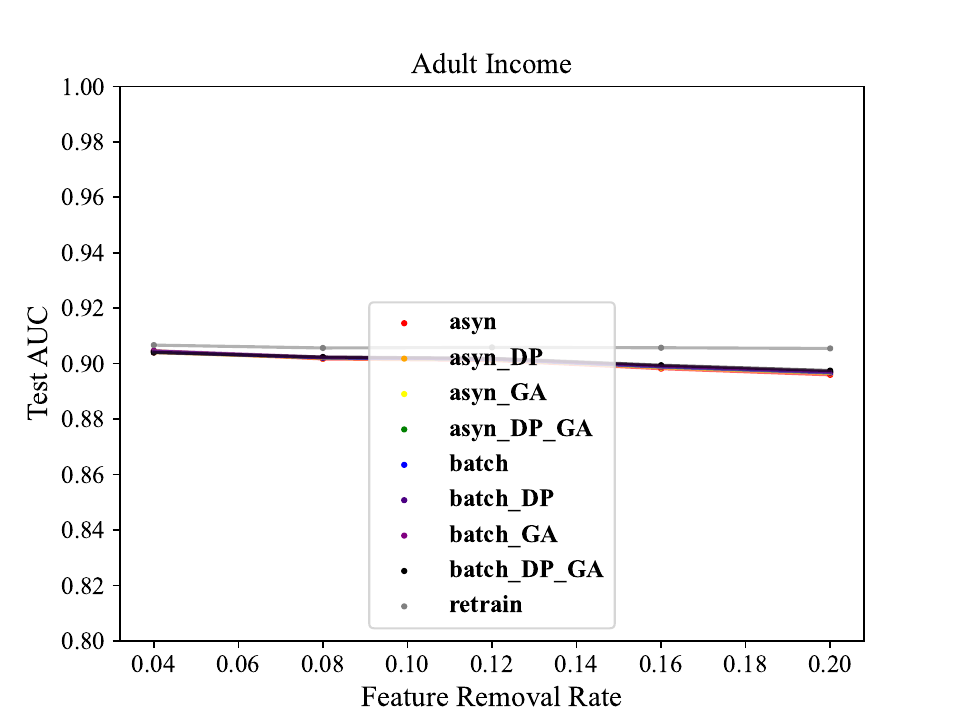}
        % \caption{a4}
    \end{minipage}
    \hfill
    \begin{minipage}[b]{0.23\textwidth}
        \centering
        \includegraphics[width=\textwidth]{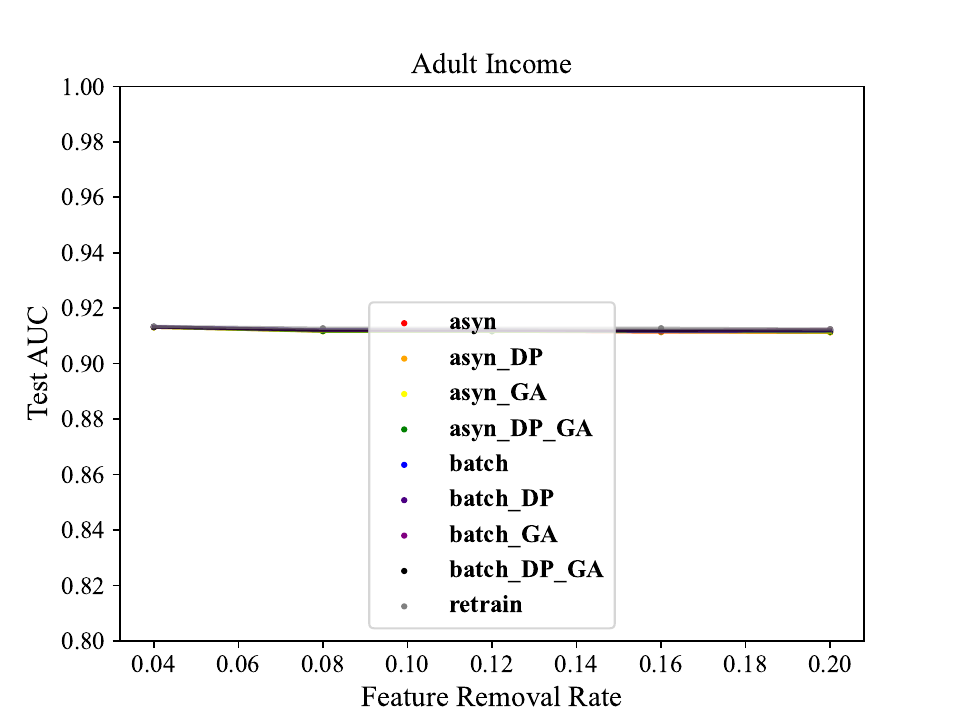}
        % \caption{子图4}
    \end{minipage}
    \hfill
    \begin{minipage}[b]{0.23\textwidth}
        \centering
        \includegraphics[width=\textwidth]{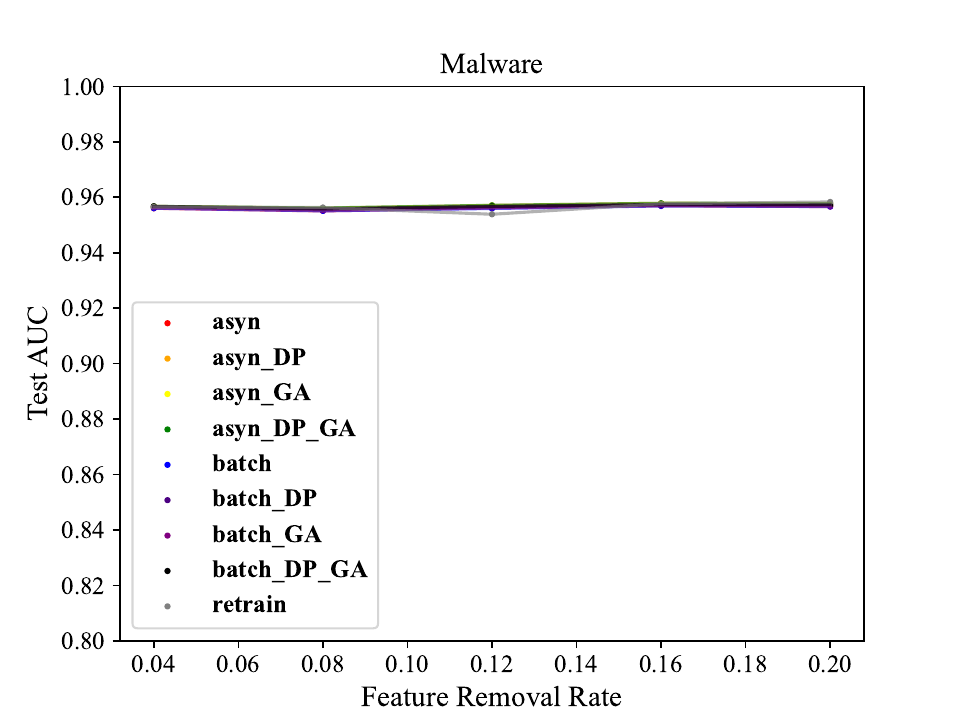}
        % \caption{子图4}
    \end{minipage}
    \hfill
    \begin{minipage}[b]{0.23\textwidth}
        \centering
        \includegraphics[width=\textwidth]{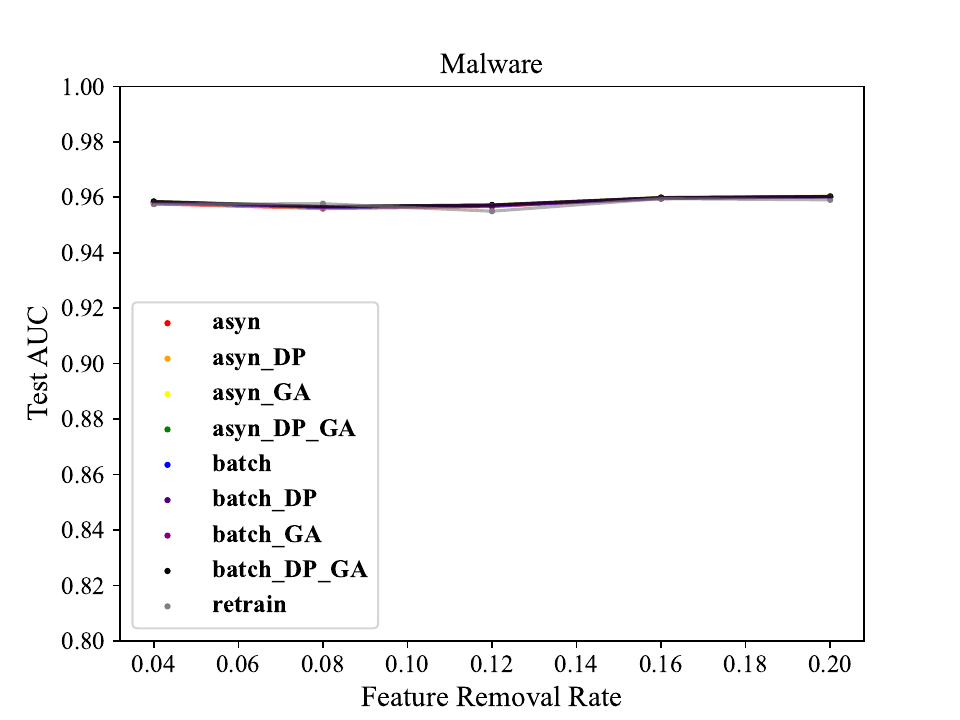}
        % \caption{子图4}
    \end{minipage}

    \vspace{0.5cm}  % 行间距

    % 第五列
    \begin{minipage}[b]{0.23\textwidth}
        \centering
        \includegraphics[width=\textwidth]{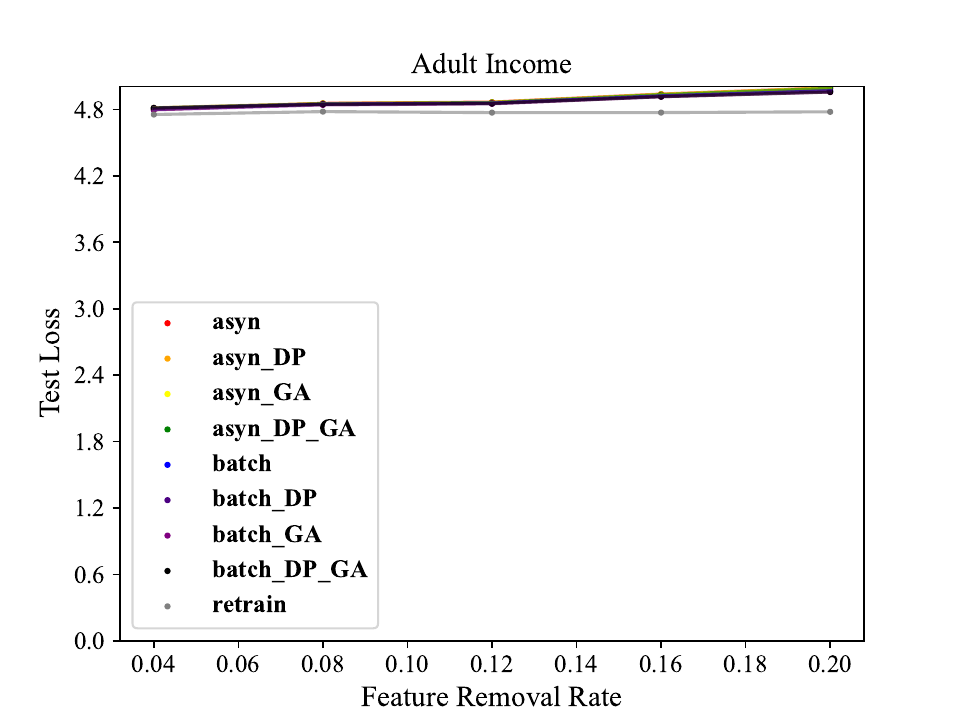}
        % \caption{a5}
    \end{minipage}
    \hfill
    \begin{minipage}[b]{0.23\textwidth}
        \centering
        \includegraphics[width=\textwidth]{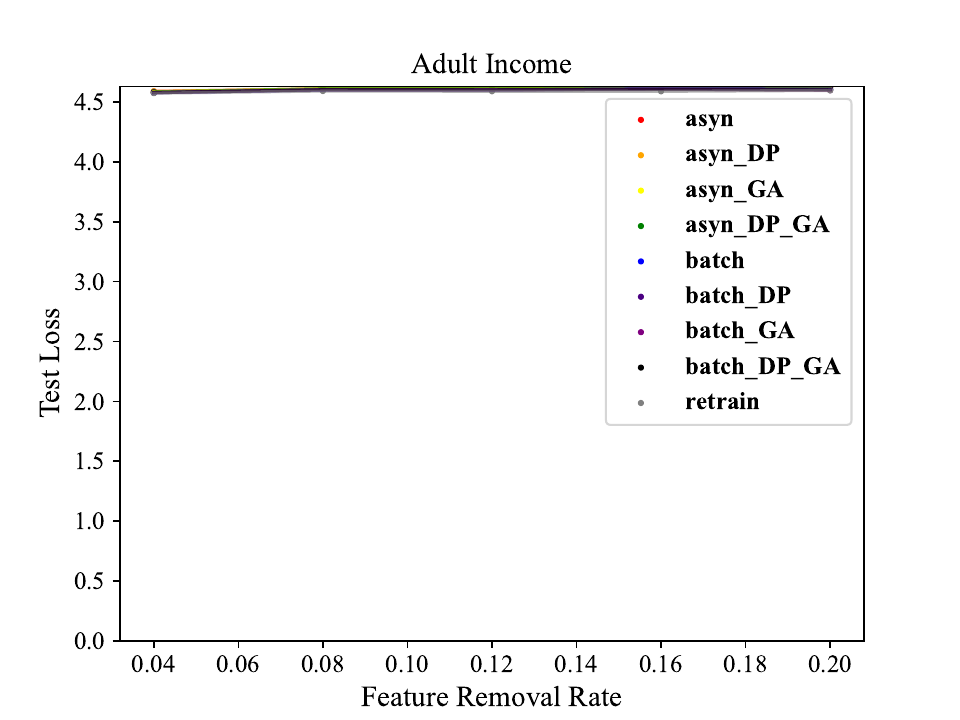}
        % \caption{子图5}
    \end{minipage}
    \hfill
    \begin{minipage}[b]{0.23\textwidth}
        \centering
        \includegraphics[width=\textwidth]{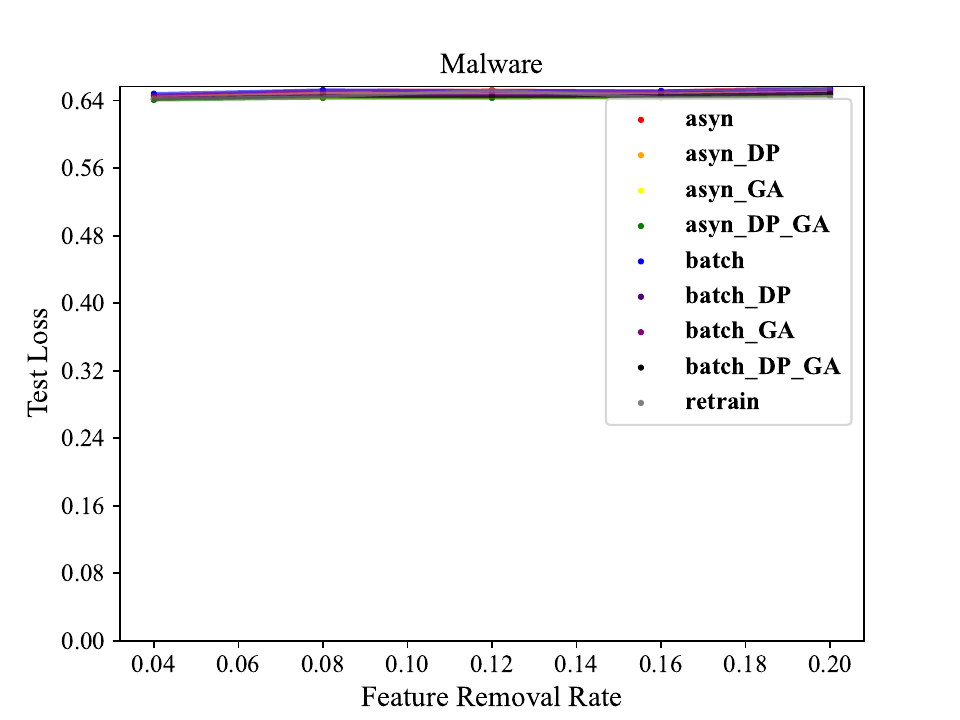}
        % \caption{子图5}
    \end{minipage}
    \hfill
    \begin{minipage}[b]{0.23\textwidth}
        \centering
        \includegraphics[width=\textwidth]{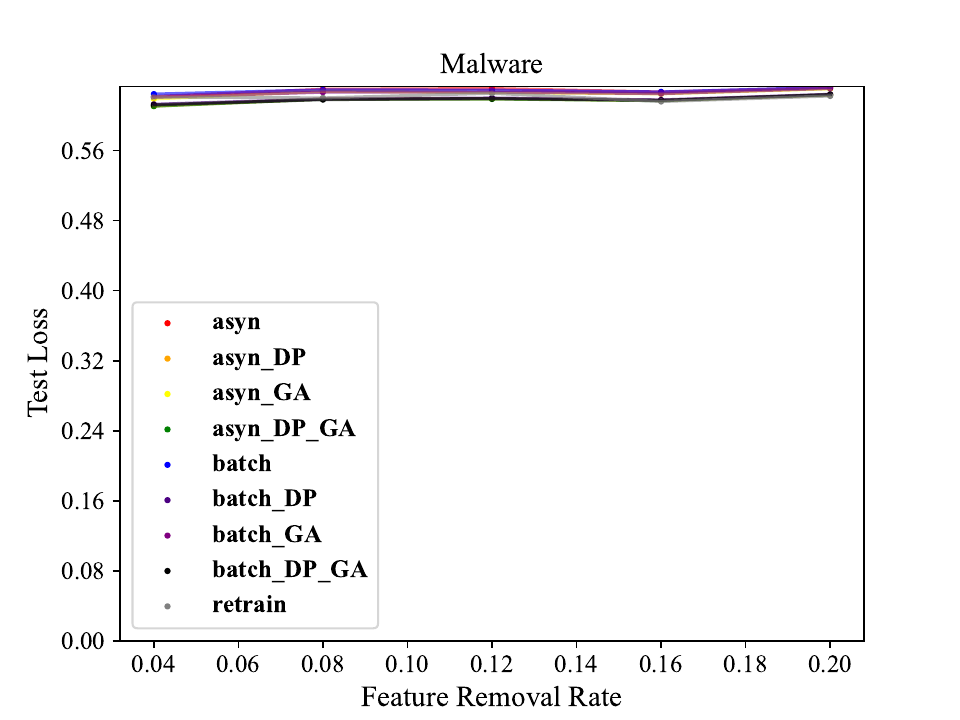}
        % \caption{子图5}
    \end{minipage}

    \caption{Feature Removal Results.}
    \label{feature extra}
\end{figure*}

\begin{figure*}[h]  % 使用 figure* 环境，使图片跨越两栏
    \centering

    % 第二列
    \begin{minipage}[b]{0.23\textwidth}
        \centering
        \includegraphics[width=\textwidth]{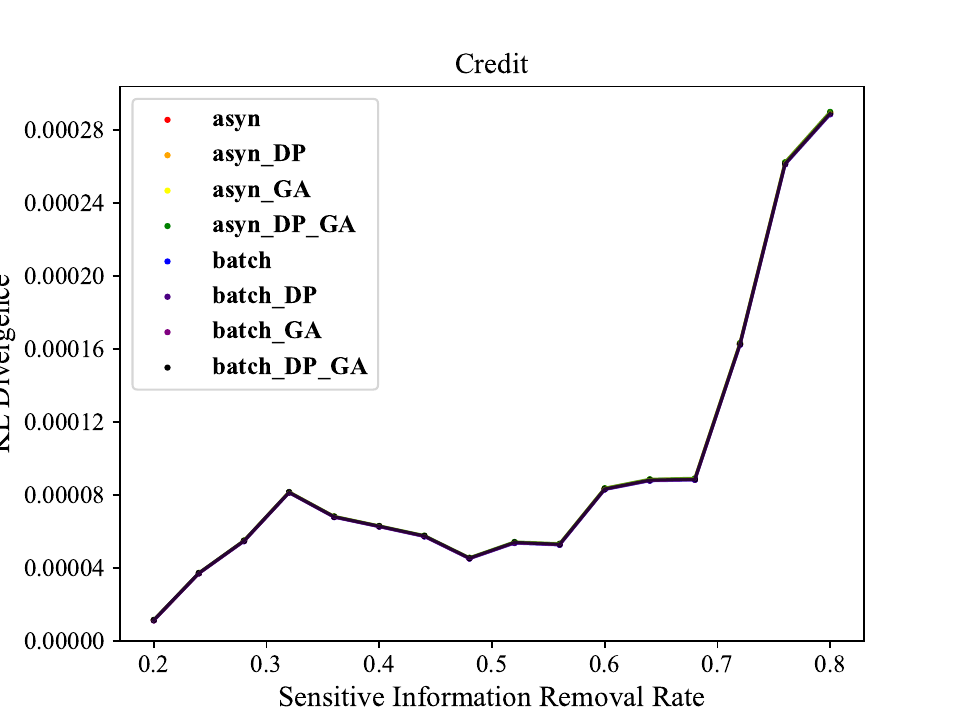}
        % \caption{a2}
    \end{minipage}
    \hfill
    \begin{minipage}[b]{0.23\textwidth}
        \centering
        \includegraphics[width=\textwidth]{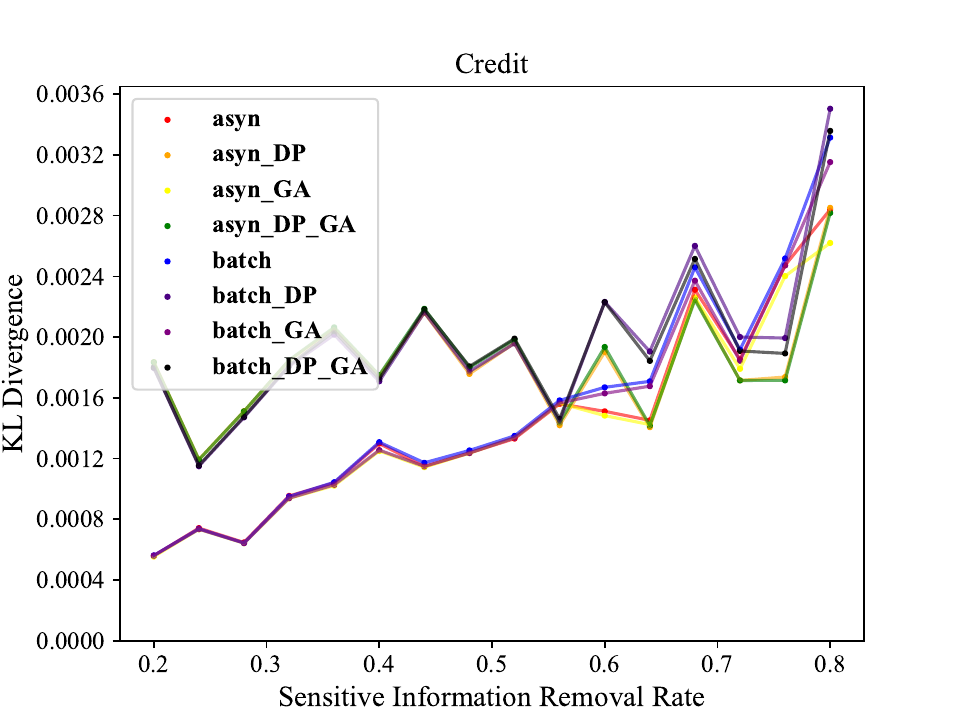}
        % \caption{子图2}
    \end{minipage}
    \hfill
    \begin{minipage}[b]{0.23\textwidth}
        \centering
        \includegraphics[width=\textwidth]{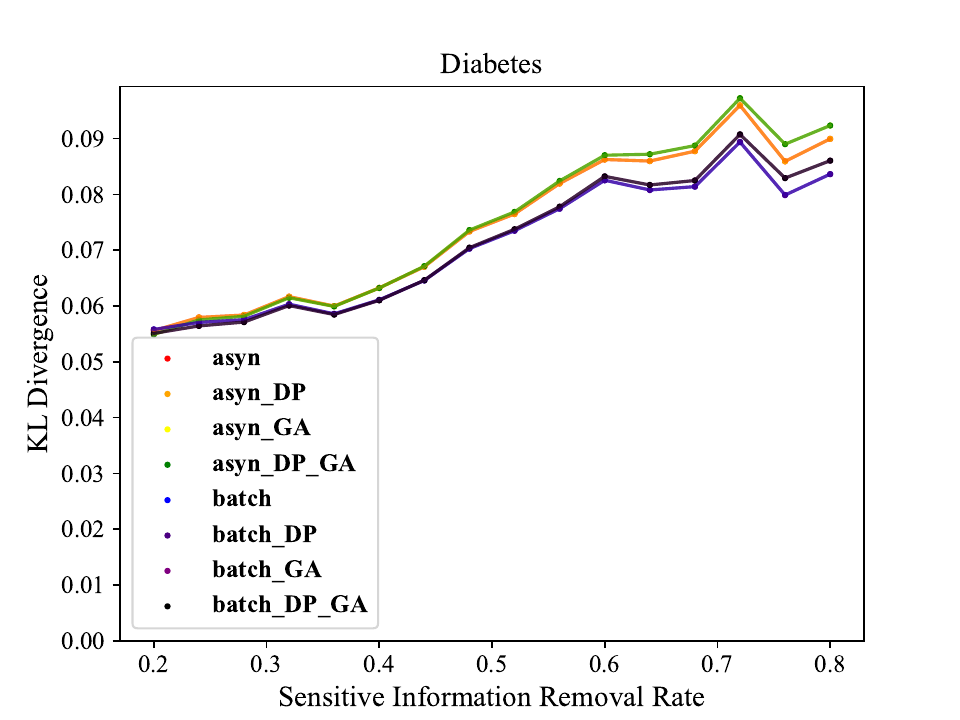}
        % \caption{子图2}
    \end{minipage}
    \hfill
    \begin{minipage}[b]{0.23\textwidth}
        \centering
        \includegraphics[width=\textwidth]{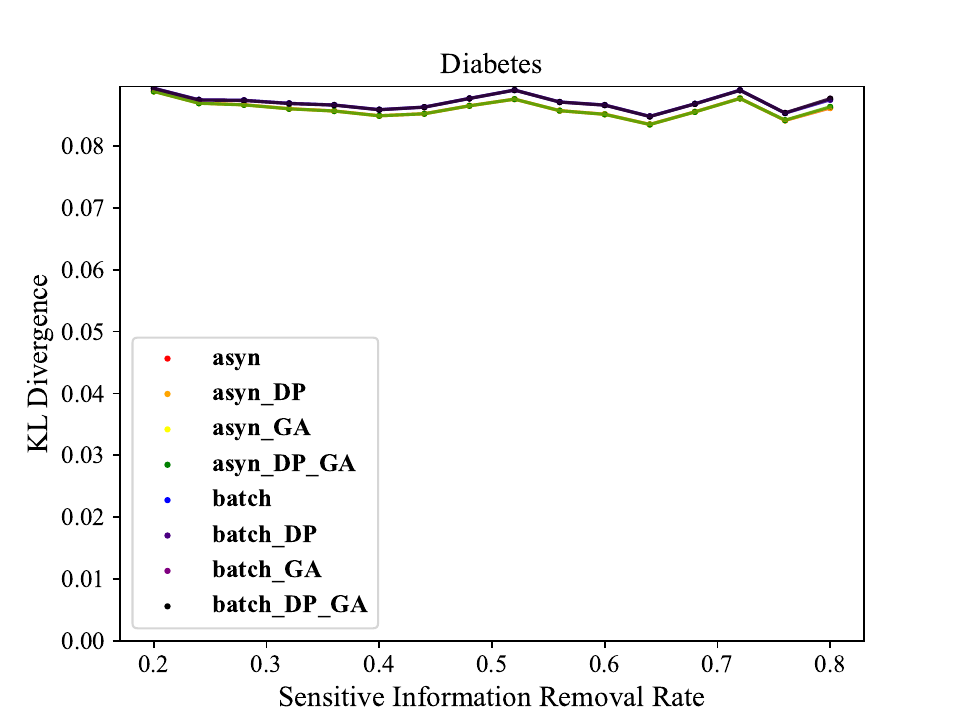}
        % \caption{子图2}
    \end{minipage}

    \vspace{0.5cm}  % 行间距

    % 第四列
    \begin{minipage}[b]{0.23\textwidth}
        \centering
        \includegraphics[width=\textwidth]{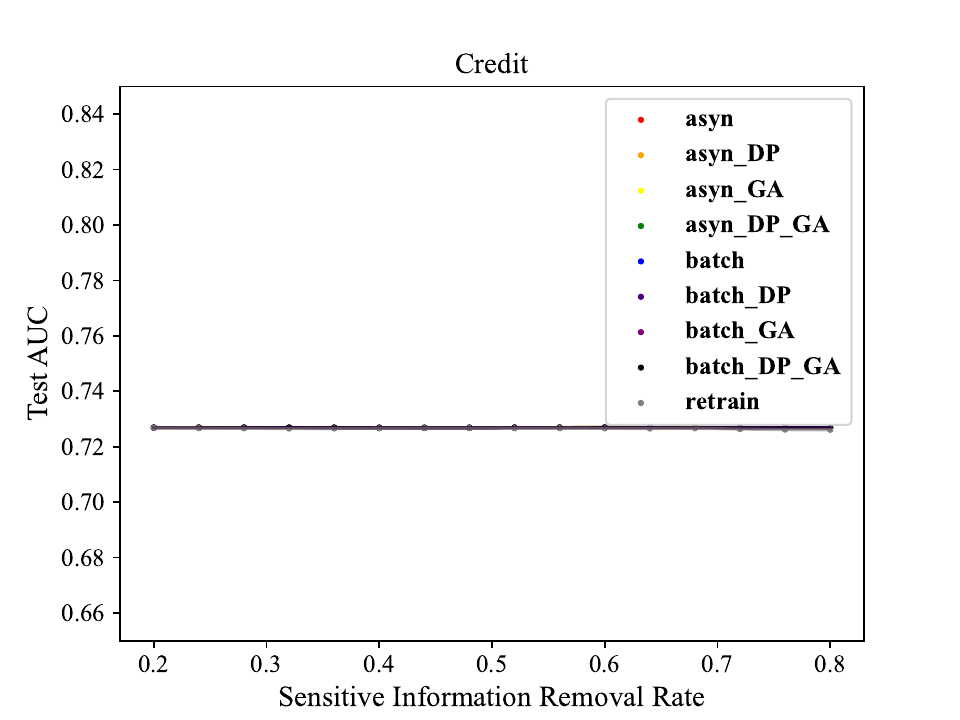}
        % \caption{a4}
    \end{minipage}
    \hfill
    \begin{minipage}[b]{0.23\textwidth}
        \centering
        \includegraphics[width=\textwidth]{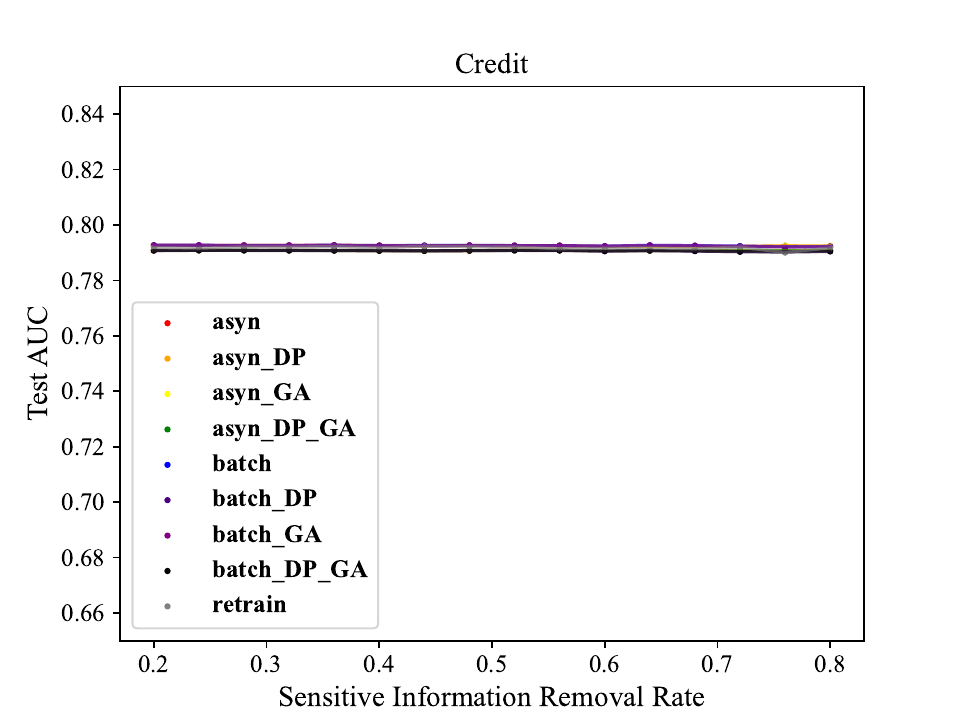}
        % \caption{子图4}
    \end{minipage}
    \hfill
    \begin{minipage}[b]{0.23\textwidth}
        \centering
        \includegraphics[width=\textwidth]{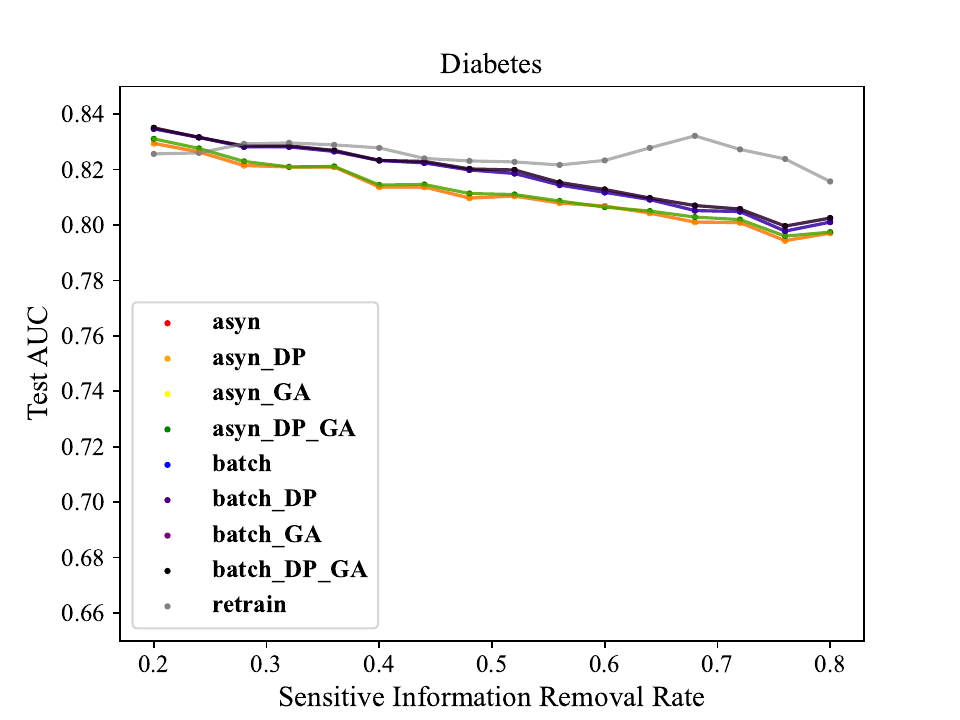}
        % \caption{子图4}
    \end{minipage}
    \hfill
    \begin{minipage}[b]{0.23\textwidth}
        \centering
        \includegraphics[width=\textwidth]{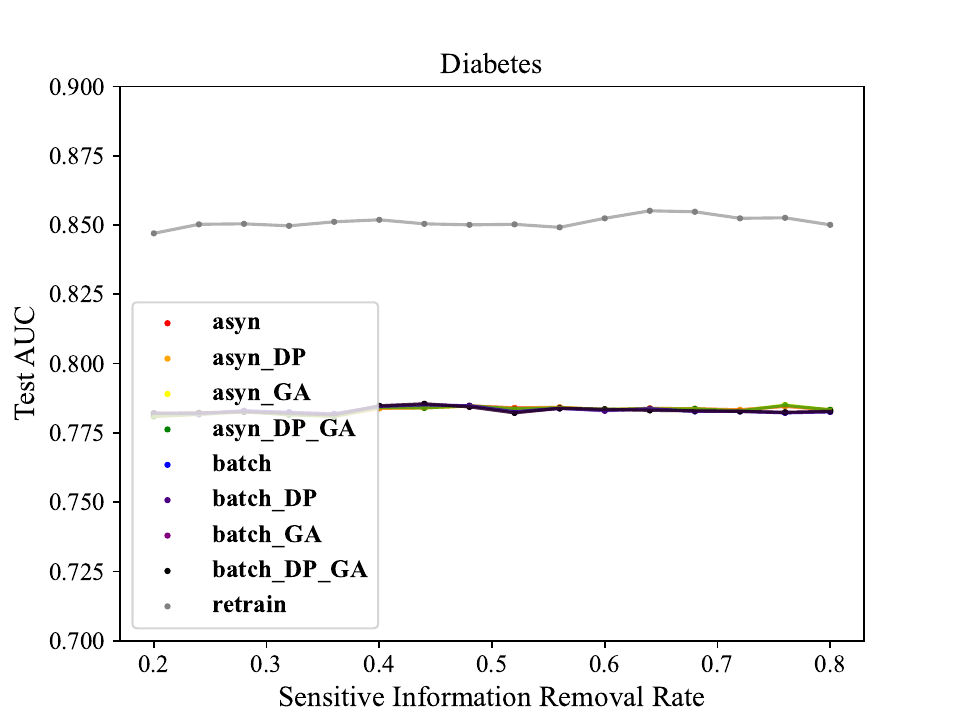}
        % \caption{子图4}
    \end{minipage}

    \vspace{0.5cm}  % 行间距

    % 第五列
    \begin{minipage}[b]{0.23\textwidth}
        \centering
        \includegraphics[width=\textwidth]{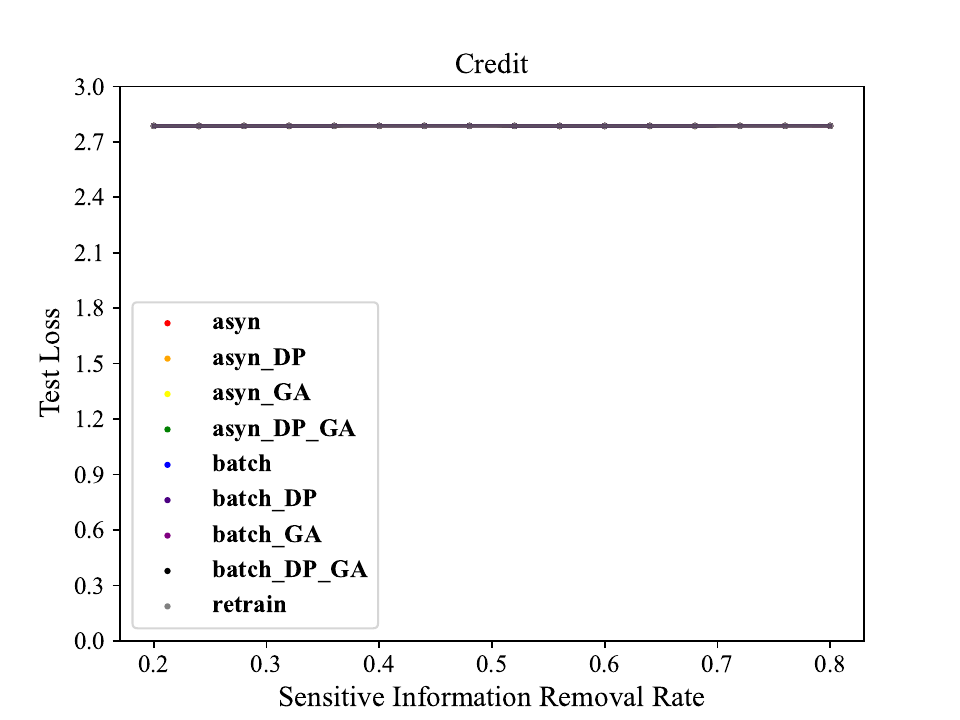}
        % \caption{a5}
    \end{minipage}
    \hfill
    \begin{minipage}[b]{0.23\textwidth}
        \centering
        \includegraphics[width=\textwidth]{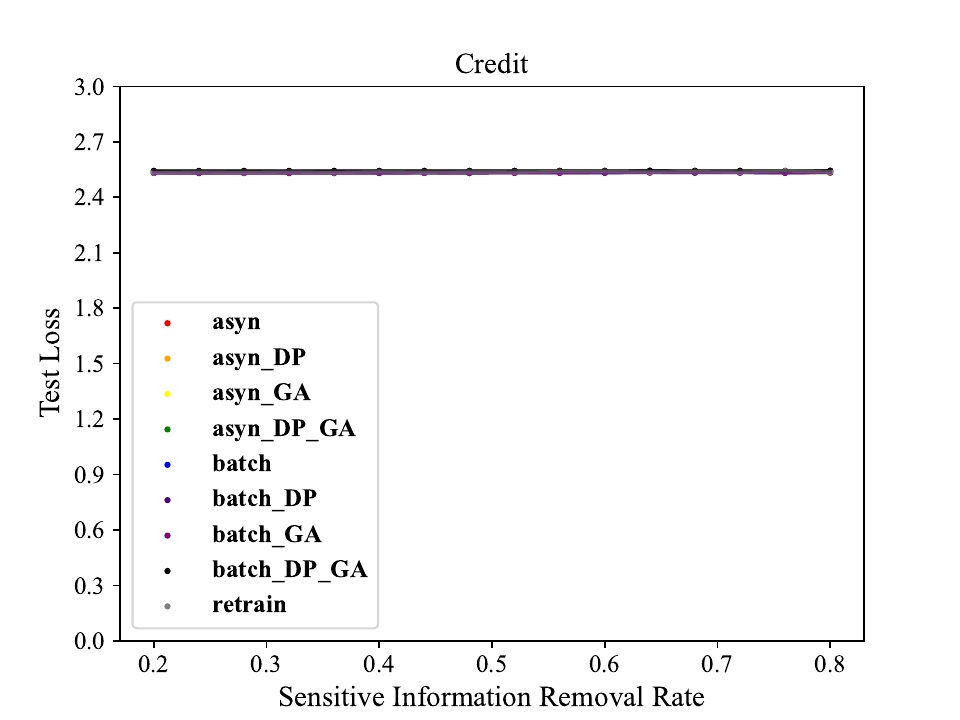}
        % \caption{子图5}
    \end{minipage}
    \hfill
    \begin{minipage}[b]{0.23\textwidth}
        \centering
        \includegraphics[width=\textwidth]{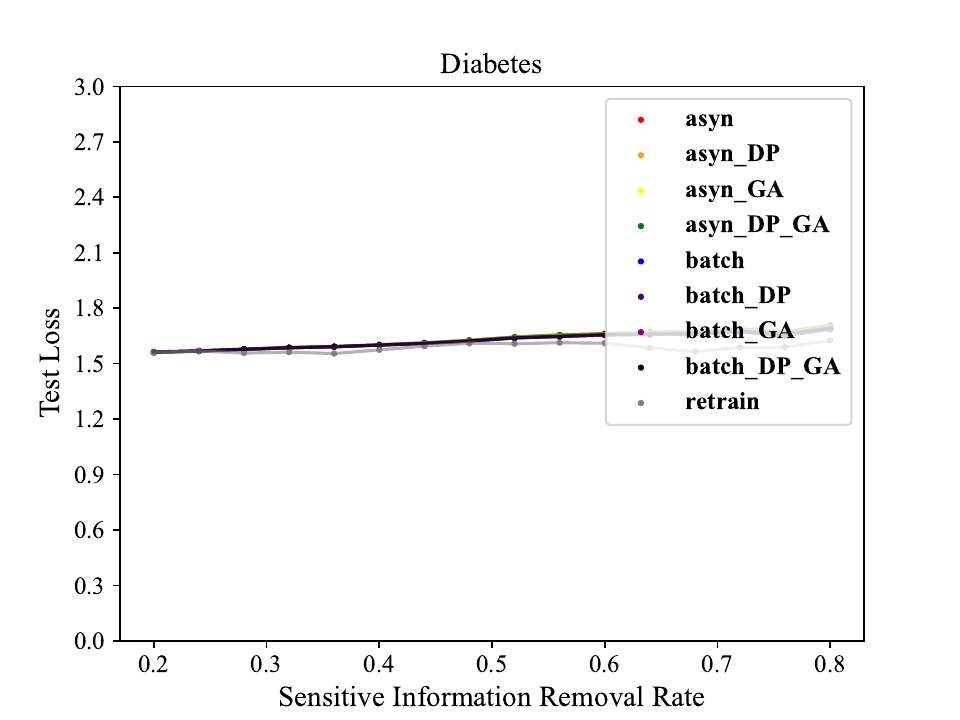}
        % \caption{子图5}
    \end{minipage}
    \hfill
    \begin{minipage}[b]{0.23\textwidth}
        \centering
        \includegraphics[width=\textwidth]{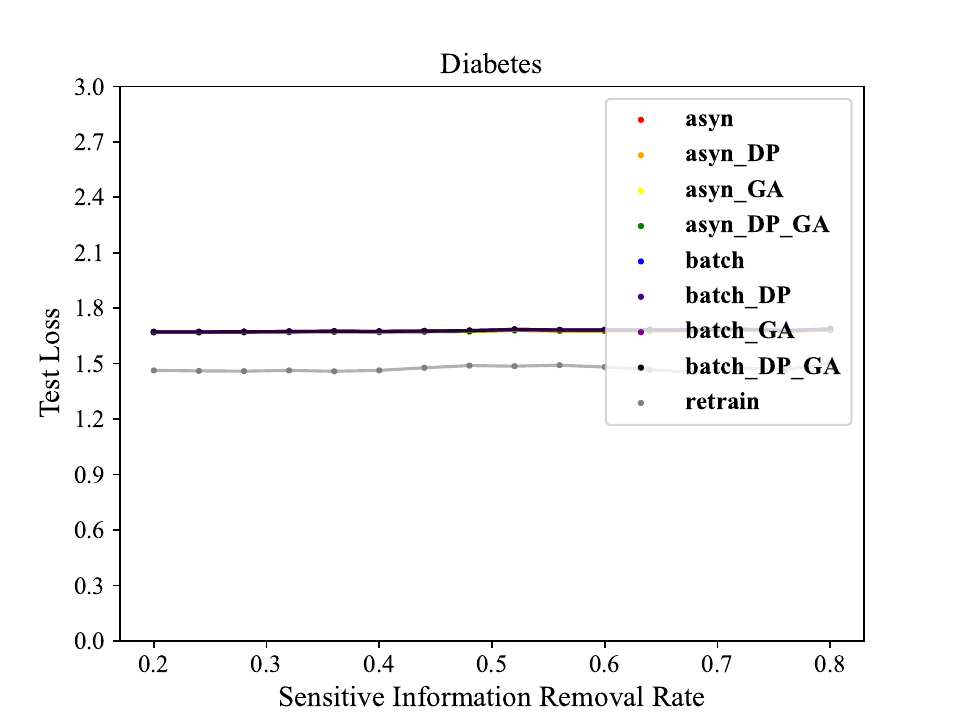}
        % \caption{子图5}
    \end{minipage}
    
    \caption{Sensitive Information Removal Results.}
    \label{info extra}
\end{figure*}

% %-------------------------------------------------------------------------------
\section*{Appedix D}
% %-------------------------------------------------------------------------------
\label{app:proof}

Our method involves multiple rounds of retraining, with the parameters of the pre-trained original model serving as the initial state. To achieve the rigorously defined certified unlearning, additional adjustments are required. It is worth noting that we make two key assumptions: First, the loss function is convex. Second, we include an \( L_2 \) regularization \( \frac{1}{2} \| \model \|_2^2 \). 

In the first round of updates, we simultaneously perform gradient ascent on the original data and gradient descent on the modified data. It can be shown that this update is equivalent to the expression in ~\cref{eqn:update1} from the well-established influence function-based certified forgetting method~\cite{warnecke2023machineunlearningfeatureslabels}. According to Lemma~\ref{thm:thm1} in that paper, the gradient bound of the updated loss function can be derived from \cref{eqn:update1}.

\begin{align*}
	\Delta(\points, \perts) &= -\tau \Big(\nabla\bigloss(\theta^* ; D^\prime) - \nabla\bigloss(\theta^* ; D)\Big)\\
    &= -\tau \Big(\sum_{\pert \in \perts}
	\gradlossPert + \nabla\bigloss(\theta^* ; D^\prime\setminus\perts)\\
    &\quad- \sum_{\point \in \points} \gradlossPoint - \nabla\bigloss(\theta^* ; D\setminus\points)\Big)\\
    &=-\tau \Big(\sum_{\pert \in \perts}
	\gradlossPert - \sum_{\point \in \points} \gradlossPoint\Big)
\end{align*}

\begin{lemma}
	\label{thm:thm1}
	~\cite{warnecke2023machineunlearningfeatureslabels}Assume that $\Vert x_i\Vert_2 \leq 1$ for all data points and the
%\DIFdelbegin \DIFdel{loss }\DIFdelend
%\DIFaddbegin \DIFadd{gradient }\DIFaddend 
gradient
$\nabla\loss(z,\model)$ is $\gamma_z$-Lipschitz with respect to $z$ at
$\optmodel$ and \mbox{$\gamma$-Lipschitz} with respect to
$\model$. Further let $\perts$ change the features $j,\dots,j+F$ by
magnitudes at most $m_j,\dots,m_{j+F}$. If $M=\sum_{j=1}^{F} m_j$ the
following upper bounds hold:
For the following update form
\begin{equation}
	\label{eqn:update1}
	\Delta(\points, \perts) = -\tau \Big(\sum_{\pert \in \perts}
	\gradlossPert - \sum_{\point \in \points} \gradlossPoint\Big)
\end{equation}
 We have%If the unlearning rate $\tau\leq \frac{1}{\gamma n}$, we have %$\optmodel_
	%\unlearn{\points}{\perts}$
	    $$\big\Vert\nabla\bigloss\big(\optmodel_
		{\points\rightarrow\perts}, D^\prime\big)\big\Vert_2 \leq
		(1+\tau\gamma n)\gamma_zM\vert\points\vert$$

\end{lemma}

The \emph{gradient residual} $\nabla\bigloss(\theta; D^\prime)$ of a model $\model$ with respect to the corrected dataset $D^\prime$ is zero only when $\model = \mathcal{A}(D^\prime)$. For strongly convex loss functions, the magnitude of this gradient residual, $\Vert \nabla\bigloss(\theta; D^\prime) \Vert_2$, reflects the discrepancy between the model $\model$ and the one obtained by retraining on $D^\prime$.

Next, in subsequent updates, we employ an early stopping mechanism to ensure that the training loss continues to decrease on the updated training set. Since the loss function is strongly convex, the gradient of the loss function after the update will also be smaller than that of the first round. 

\begin{align*}
\big\Vert\nabla\bigloss(\theta_{t+1})\big\Vert^2 \leq \big\Vert\nabla\bigloss(\theta_{t})\big\Vert^2
\end{align*}

Finally, we prove that after unlearning, the gradient of the loss function has an upper bound. Based on Lemma~\ref{thm:thm3} from the~\cite{warnecke2023machineunlearningfeatureslabels}, we are able to demonstrate that certified unlearning holds.

When a vector $b$ is added, the gradient residual $r$ for the loss function $\bigloss_b$ becomes:

\[
r = \nabla\bigloss_b(\theta; D^\prime) = \sum_{\point \in D^\prime} \nabla \loss(z, \theta) + \lambda \model + b
\]

By manipulating the distribution of $b$, certified unlearning can be achieved, akin to sensitivity-based techniques~\cite{dp2011}.

\begin{lemma}
	\label{thm:thm3}
	~\cite{warnecke2023machineunlearningfeatureslabels}Let $\mathcal{A}$ be the learning algorithm that returns the
unique minimum of $\bigloss_b(\theta;D^\prime)$ and let
$\mathcal{U}$ be an unlearning method that produces a model
$\model_{\mathcal{U}}$. If
$\Vert \nabla\bigloss(\theta_\mathcal{U};D^\prime)\Vert_2 \leq
\epsilon'$
for some $\epsilon' >0$ we have the following guarantees.
\begin{enumerate}%[wide, labelwidth=!, labelindent=0pt[wide, labelwidth=!, labelindent=0pt]
	\item If $b$ is drawn from a distribution with density
	$p(b)=e^{-\frac{\epsilon}{\epsilon'}\Vert b\Vert_2}$ then
	$\mathcal{U}$ performs $\epsilon$-certified unlearning for
	$\mathcal{A}$.
	
	\item If $p\sim \mathcal{N}(0, c\epsilon'/\epsilon)^d$ for
	some $c>0$ then $\mathcal{U}$ performs $(\epsilon,
	\delta)$-certified unlearning for $\mathcal{A}$ with
	$\delta=1.5e^{-c^2/2}$.
\end{enumerate}
\end{lemma}

%%%%%%%%%%%%%%%%%%%%%%%%%%%%%%%%%%%%%%%%%%%%%%%%%%%%%%%%%%%%%%%%%%%%%%%%%%%%%%%%
\end{document}